\documentclass{article}

\PassOptionsToPackage{numbers, compress}{natbib}
\usepackage[final]{nips_2017}

\usepackage[T1]{fontenc}    % use 8-bit T1 fonts
\usepackage{hyperref}       % hyperlinks
\usepackage{url}            % simple URL typesetting
\usepackage{booktabs}       % professional-quality tables
\usepackage{amsfonts}       % blackboard math symbols
\usepackage{nicefrac}       % compact symbols for 1/2, etc.
\usepackage{microtype}      % microtypography
\usepackage{natbib}
\usepackage{amsthm}
\usepackage{algorithm}
\usepackage{algorithmic}
\usepackage{bbm}
\usepackage{amsfonts}
\usepackage{amsmath}           
\usepackage{mathtools}    
\usepackage{amsopn}

\usepackage{scrwfile}
\TOCclone[\contentsname~(\appendixname)]{toc}{atoc}
\newcommand\StartAppendixEntries{}
\AfterTOCHead[toc]{%
  \renewcommand\StartAppendixEntries{\value{tocdepth}=-10000\relax}%
}
\AfterTOCHead[atoc]{%
  \edef\maintocdepth{\the\value{tocdepth}}%
  \value{tocdepth}=-10000\relax%
  \renewcommand\StartAppendixEntries{\value{tocdepth}=\maintocdepth\relax}%
}
\newcommand*\appendixwithtoc{%
  \cleardoublepage
  \appendix
  \addtocontents{toc}{\protect\StartAppendixEntries}
  \listofatoc
}
\usepackage{blindtext}
\newtheorem{assumption}{Assumption}
\newtheorem{theorem}{Theorem}
\newtheorem{lemma}[theorem]{Lemma}

\newtheorem{definition}[theorem]{Definition}
\newtheorem{remark}[theorem]{Remark}
\numberwithin{theorem}{section}
\DeclareMathOperator{\val}{val}
\DeclareMathOperator{\spa}{sp}
\DeclareMathOperator{\solve}{solve}
\allowdisplaybreaks

\newlength\leftindent
\newcommand{\norm}[1]{\left\lVert#1\right\rVert}

\newlength\rightindent
\usepackage{amsmath}
\DeclareMathOperator*{\argmin}{\arg\!\min}
\DeclareMathOperator*{\argmax}{\arg\!\max}
\title{Online Reinforcement Learning in Stochastic Games}

\author{
  Chen-Yu Wei\\
Institute of Information Science\\
  Academia Sinica, Taiwan\\
  \texttt{bahh723@iis.sinica.edu.tw} \\
 \And Yi-Te Hong \\
Institute of Information Science\\
  Academia Sinica, Taiwan\\
  \texttt{ted0504@iis.sinica.edu.tw} \\
\And Chi-Jen Lu\\
  Institute of Information Science\\
  Academia Sinica, Taiwan\\
  \texttt{cjlu@iis.sinica.edu.tw} \\
}

\begin{document}
% \nipsfinalcopy is no longer used

\maketitle

%#############################################################################
%   Abstract
%#############################################################################

\begin{abstract} 
We study online reinforcement learning in average-reward stochastic games (SGs). An SG models a two-player zero-sum game in a Markov environment, where state transitions and one-step payoffs are determined simultaneously by a learner and an adversary. We propose the \textsc{UCSG} algorithm that achieves a sublinear regret compared to the game value when competing with an arbitrary opponent. This result improves previous ones under the same setting. The regret bound has a dependency on the \textit{diameter}, which is an intrinsic value related to the mixing property of SGs. If we let the opponent play an optimistic best response to the learner, \textsc{UCSG} finds an $\varepsilon$-maximin stationary policy with a sample complexity of $\tilde{\mathcal{O}}\left(\text{poly}(1/\varepsilon)\right)$, where $\varepsilon$ is the gap to the best policy. 
\end{abstract} 

%#############################################################################
%    Introduction
%#############################################################################

\section{Introduction}
Many real-world scenarios (e.g., markets, computer networks, board games) can be cast as multi-agent systems. The framework of Multi-Agent Reinforcement Learning (MARL) targets at learning to act in such systems. While in traditional reinforcement learning (RL) problems, Markov decision processes (MDPs) are widely used to model a single agent's interaction with the environment, stochastic games (SGs, \cite{shapley1953stochastic}), as an extension of MDPs, are able to describe multiple agents' simultaneous interaction with the environment. In this view, SGs are most well-suited to model MARL problems \cite{littman1994markov}. 

In this paper, two-player zero-sum SGs are considered. These games proceed like MDPs, with the exception that in each state, both players select their own actions \textit{simultaneously} \footnote{Turn-based SGs, like Go, are special cases: in each state, one player's action set contains only a null action.}, which jointly determine the transition probabilities and their rewards .  The \textit{zero-sum} property restricts that the two players' payoffs sum to zero. Thus, while one player (Player 1) wants to maximize his/her total reward, the other (Player 2) would like to minimize that amount. Similar to the case of MDPs, the reward can be discounted or undiscounted, and the game can be episodic or non-episodic.   

In the literature, SGs are typically learned under two different settings, and we will call them \textit{online} and \textit{offline} settings, respectively. In the offline setting, the learner controls both players in a centralized manner, and the goal is to find the equilibrium of the game \cite{szepesvari1996generalized, lagoudakis2002value, perolat2015approximate}. This is also known as finding the worst-case optimality for each player (a.k.a. maximin or minimax policy). In this case, we care about the \textit{sample complexity}, i.e., how many samples are required to estimate the worst-case optimality such that the error is below some threshold. In the online setting, the learner controls only one of the players, and plays against an arbitrary opponent \cite{littman1994markov, bowling2001rational, brafman2002r, conitzer2007awesome, prasad2015two}. In this case, we care about the learner's \textit{regret}, i.e., the difference between some benchmark measure and the learner's total reward earned in the learning process. This benchmark can be defined as the total reward when both players play optimal policies \cite{brafman2002r}, or when Player 1 plays the best stationary response to Player 2 \cite{bowling2001rational}. Some of the above online-setting algorithms can find the equilibrium simply through self-playing. 

Most previous results on offline sample complexity consider discounted SGs. Their bounds depend heavily on the chosen discount factor \cite{szepesvari1996generalized, lagoudakis2002value, perolat2015approximate, prasad2015two}. However, as noted in \cite{brafman2002r, jaksch2010near}, the discounted setting might not be suitable for SGs that require long-term planning, because only finite steps are relevant in the reward function it defines. This paper, to the best of our knowledge, is the first to give an offline sample complexity bound of order $\tilde{\mathcal{O}}\left(\text{poly}(1/\varepsilon)\right)$ in the average-reward (undiscounted and non-episodic) setting, where $\varepsilon$ is the error parameter.  A major difference between our algorithm and previous ones is that the two players play asymmetric roles in our algorithm: by focusing on finding only one player's worst-case optimal policy at a time, the sampling can be rather efficient. This resembles but strictly extends \cite{garivier2016maximin}'s methods in finding the maximin action in a two-stage game. 

In the online setting, we are only aware of \cite{brafman2002r}'s \textsc{R-max} algorithm that deals with average-reward SGs and provides a regret bound. Considering a similar scenario and adopting the same regret definition, we significantly improve their bounds (see Appendix \ref{appendix:previous_bound} for details). Another difference between our algorithm and theirs is that ours is able to output a currently best stationary policy at any stage in the learning process, while theirs only produces a $T_\varepsilon$-step fixed-horizon policy for some input parameter $T_\varepsilon$. The former could be more natural since the worst-case optimal policy is itself a stationary policy. 

The techniques used in this paper are most related to RL for MDPs based on the optimism principle \cite{auer2007logarithmic, jaksch2010near, dann2015sample} (see Appendix \ref{appendix:previous_bound}). The optimism principle built on concentration inequalities automatically strikes a balance between exploitation and exploration, eliminating the need to manually adjust the learning rate or the exploration ratio.  However, when importing analysis from MDPs to SGs, we face the challenge caused by the opponent's uncontrollability and non-stationarity. This prevents the learner from freely exploring the state space and makes previous analysis that relies on stationary distribution's perturbation analysis \cite{auer2007logarithmic} useless. In this paper, we develop a novel way to replace the opponent's non-stationary policy with a stationary one in the analysis (introduced in Section \ref{Regret decomposition and the introduction of pi2}), which facilitates the use of techniques based on perturbation analysis. We hope that this technique can benefit future analysis concerning non-stationary agents in MARL. 

One related topic is the robust MDP problem \cite{nilim2005robustcontrol, Iyengar:2005robustdynamic, lim2016robustMDP}. It is an MDP where some state-action pairs have adversarial rewards and transitions. It is often assumed in robust MDP that the adversarial choices by the environment are not directly observable by the Player, but in our SG setting, we assume that the actions of Player 2 can be observed. However, there are still difficulties in SG that are not addressed by previous works on robust MDP. 

Here we compare our work to \cite{lim2016robustMDP}, a recent work on learning robust MDP. In their setting, there are adversarial and stochastic state-action pairs, and their proposed \textsc{OLRM2} algorithm tries to distinguish them. Under the scenario where the environment is fully adversarial, which is the counterpart to our setting, the worst-case transitions and rewards are all revealed to the learner, and what the learner needs to do is to perform a maximin planning. In our case, however, the worst-case transitions and rewards are still to be learned, and the opponent's arbitrary actions may hinder the learner to learn this information. We would say that the contribution of \cite{lim2016robustMDP} is orthogonal to ours. 

Other lines of research that are related to SGs are on MDPs with adversarially changing reward functions \cite{even2009online, neu2010online, neu2012adversarial, dick2014online} and with adversarially changing transition probabilities \cite{yu2009arbitrarily, abbasi2013online}. The assumptions in these works have several differences with ours, and therefore their results are not comparable to our results. However, they indeed provide other viewpoints about learning in stochastic games. %There is the potential that the frameworks of adversarial MDPs and SGs could be unified.  

%#############################################################################
%   Preliminaries and Notations
%#############################################################################
\section{Preliminaries}
\label{section:preliminary}
\textbf{Game Models and Policies.}\ \ \ \ \ A SG is a 4-tuple $M=(\mathcal{S},\mathcal{A},r,p)$. $\mathcal{S}$ denotes the state space and $\mathcal{A}=\mathcal{A}^1\times\mathcal{A}^2$ the players' joint action space. We denote $S=|\mathcal{S}|$ and $A=|\mathcal{A}|$. The game starts from an initial state $s_1$. Suppose at time $t$ the players are at state $s_t$. After the players play the joint actions $(a^1_t,a^2_t)$, Player 1 receives the reward $r_t = r(s_t,a^1_t,a^2_t)\in [0,1]$ from Player 2, and both players visit state $s_{t+1}$ following the transition probability $p(\cdot|s_t,a^1_t,a^2_t)$. For simplicity, we consider deterministic rewards as in \cite{bartlett2009regal}. The extension to stochastic case is straightforward. We shorten our notation by $a\coloneqq (a^1,a^2)$ or $a_t\coloneqq (a_t^1,a_t^2)$, and use abbreviations such as $r(s_t,a_t)$ and $p(\cdot\vert s_t,a_t)$. 

Without loss of generality, players are assumed to determine their actions based on the history. A policy $\pi$ at time $t$ maps the history up to time $t$, $H_t=(s_1, a_1, r_1, ..., s_t)\in \mathcal{H}_t $, to a probability distribution over actions. Such policies are called \textit{history-dependent} policies, whose class is denoted by $\Pi^{\text{HR}}$. On the other hand, a \textit{stationary} policy, whose class is denoted by $\Pi^{\text{SR}}$, selects actions as a function of the current state. For either class, joint policies $(\pi^1, \pi^2)$ are often written as $\pi$.

\textbf{Average Return and the Game Value.}\ \ \ Let the players play joint policy $\pi$. %\footnote{Note that when we do not specify from which class $\pi$ is chosen from, it is chosen from $\Pi^{\text{HR}}$.} 
Define the $T$-step total reward as $R_{T}(M, \pi, s)\coloneqq \sum_{t=1}^T r(s_t,a_t),  \text{\ where } s_1=s$, and the average reward as $\rho(M,\pi, s)\coloneqq \lim_{T\rightarrow \infty}\frac{1}{T}\mathbbm{E}\left[R_T(M,\pi, s)\right]$, whenever the limit exists. In fact, the game value exists\footnote{Unlike in one-player MDPs, the $\sup$ and $\inf$ in the definition of $\rho^*(M,s)$ are not necessarily attainable. Moreover,  
players may not have stationary optimal policies.} \cite{mertens1981stochastic}: 
\begin{align*}
\rho^*(M,s)\coloneqq\sup_{\pi^1} \inf_{\pi^2} \lim_{T\rightarrow\infty} \frac{1}{T}\mathbbm{E}\left[R_T(M,\pi^1,\pi^2,s)\right].
\end{align*}
If $\rho(M,\pi, s)$ or $\rho^*(M,s)$ does not depend on the initial state $s$, we simply write $\rho(M,\pi)$ or $\rho^*(M)$. 

\textbf{The Bias Vector.}\ \ \ \ \ For a stationary policy $\pi$, the bias vector $h(M,\pi,\cdot)$ is defined, for each coordinate $s$, as 
\begin{align}
h(M,\pi,s)\coloneqq\mathbbm{E}\left[\sum_{t=1}^{\infty}r(s_t,a_t)-\rho(M, \pi,s) \Big\vert s_1=s, a_t\sim\pi(\cdot\vert s_t)\right]. \label{bias_vec}
\end{align}
%As our convention, if we are refering to SGs, $h(M,\pi,s)$ and $h(M,\pi^1, \pi^2, s)$ are used interchangeably. 
The bias vector satisfies the Bellman equation:  $\forall s\in \mathcal{S}$,
\begin{align*}
\rho(M,\pi,s)+h(M,\pi,s)=r(s,\pi)+\sum_{s'}p(s'|s,\pi)h(M,\pi,s'),
\end{align*}
where $r(s,\pi)\coloneqq\mathbbm{E}_{a\sim \pi(\cdot|s)}[r(s,a)]$ and $p(s'|s,\pi)\coloneqq$$\mathbbm{E}_{a\sim\pi(\cdot|s)}[p(s'\vert s,a)]$.

The vector $h(M,\pi, \cdot)$ describes the relative \textit{advantage} among states under model $M$ and (joint) policy $\pi$. The advantage (or disadvantage) of state $s$ compared to state $s'$ under policy $\pi$ is defined as the difference between the accumulated rewards with initial states $s$ and $s^\prime$, which, from \eqref{bias_vec}, converges to the difference $h(M,\pi, s)-h(M,\pi, s')$ asymptotically. For the ease of notation, the \textit{span} of a vector $v$ is defined as $\spa(v)\coloneqq \max_iv_i-\min_iv_i$. Therefore if a model, together with any policy, induces large $\spa{(h)}$, then this model will be difficult to learn because visiting a bad state costs a lot in the learning process. As shown in \cite{bartlett2009regal} for the MDP case, the regret has an inevitable dependency on $\spa (h(M, \pi^{*}, \cdot))$, where $\pi^*$ is the optimal policy.

On the other hand, $\spa(h(M,\pi, \cdot))$ is closely related to the mean first passage time under the Markov chain induced by $M$ and $\pi$. Actually we have $\spa(h(M,\pi, \cdot)) \leq T^\pi(M)\coloneqq \max_{s,s^\prime}T^\pi_{s\rightarrow  s^\prime}(M)$, where $T^{\pi}_{s\rightarrow s^\prime}(M)$ denotes the expected time to reach state $s^\prime$ starting from $s$ when the model is $M$ and the player(s) follow the (joint) policy $\pi$. This fact is intuitive, and the proof can be seen at Remark \ref{remark:mean_first_passage_time}.

\textbf{Notations.}\ \ \ \ \ In order to save space, we often write equations in vector or matrix form. 
We use vectors inequalities: if $u, v\in \mathbb{R}^n$, then $u\leq v \Leftrightarrow u_i\leq v_i \ \forall i=1, ..., n$. 
For a general matrix game with matrix $G$ of size $n\times m$, we denote the value of the game as $\displaystyle \val G\coloneqq \max_{p\in \Delta_n} \min_{q\in \Delta_m}p^\top Gq= \min_{q\in \Delta_m}\max_{p\in \Delta_n}p^\top Gq$, where $\Delta_k$ is the probability simplex of dimension $k$. In SGs, given the estimated value function $u(s^\prime) \ \forall s^\prime$, we often need to solve the following matrix game equation: 
\begin{align*}
\scalebox{0.9}{$\displaystyle v(s)=\max_{a^1\sim\pi^1(\cdot\vert s)}\min_{a^2\sim\pi^2(\cdot\vert s)} \{ r(s,a^1,a^2) + \sum_{s^\prime} p(s^\prime\vert s,a^1,a^2)u(s^\prime) \}$}, 
\end{align*}
and this is abbreviated with the vector form $v=\val\{ r+Pu \}$. We also use $\solve_1 G$ and $\solve_2 G$ to denote the optimal solutions of $p$ and $q$. In addition, the indicator function is denoted by $\mathbbm{1}\{\cdot\}$ or $\mathbbm{1}_{\{\cdot\}}$.

%#############################################################################
%   Problem Setting and Results Overview
%#############################################################################
\section{Problem Settings and Results Overview}
We assume that the game proceeds for $T$ steps. In order to have meaningful regret bounds (i.e., sublinear to $T$), we must make some assumptions to the SG model itself. Our two different assumptions are 
\begin{assumption}
\label{assumption:irreducible}
$\scalebox{1.0}{$\displaystyle \max_{s,s^\prime}\max_{\pi^1\in \Pi^{\text{SR}}} \max_{\pi^2\in\Pi^{\text{SR}}} T_{s\rightarrow s'}^{\pi^1, \pi^2}(M) \leq D.$}$
\end{assumption}
\begin{assumption}
\label{assumption:ergodic}
$\scalebox{1.0}{$\displaystyle \max_{s,s^\prime}\max_{\pi^2\in\Pi^{\text{SR}}} \min_{\pi^1\in\Pi^{\text{SR}}} T_{s\rightarrow s'}^{\pi^1, \pi^2}(M) \leq D.$}$
\end{assumption}

Why we make these assumptions is as follows. Consider an SG model where the opponent (Player 2) has some way to lock the learner (Player 1) to some bad state. The best strategy for the learner might be to totally avoid, if possible, entering that state. However, in the early stage of the learning process, the learner won't know this, and he/she will have a certain probability to visit that state and get locked. This will cause linear regret to the learner. Therefore, we assume the following: whatever policy the opponent executes, the learner always has some way to reach any state within some bounded time. This is essentially our Assumption \ref{assumption:ergodic}. 

Assumption \ref{assumption:irreducible} is the stronger one that actually implies that under any policies executed by the players (not necessarily stationary, see Remark \ref{remark:average_time_<D}), every state is visited within an average of $D$ steps. We find that under this assumption, the asymptotic regret can be improved. This assumption also has a sense similar to those required for Q-learning-type algorithms' convergence: they require that every state be visited infinitely often. See \cite{jaakkola1994convergence} for example.  

%Similar to previous works on average return reinforcement learning, we make mixing assumptions on the Markov model. We provide two different assumptions. %\footnote{When $s=s^\prime$, $T_{s\rightarrow s}^{\pi^1,\pi^2}(M)$ evaluates the \textit{returning time} to $s$, and hence it is always larger than zero.} 

%Each assumption defines some notion of \textit{diameter} $D$ of the stochastic game. %Assumption \ref{assumption:irreducible} is strong because it implies that every pair of policies induces an irreducible Markov chain with bounded travelling time. On the contrary, Assumption \ref{assumption:ergodic} is a weakest sufficient condition for a stochastic game to have state-independent game value (i.e., the game value does not depend on the initial states). This is in the sense that if this condition does not hold, there exists a reward assignment such that the game value depends on the initial state \cite{boros2010pumping}. 
These assumptions define some notion of \textit{diameters} that are specific to the SG model. It is known that under Assumption \ref{assumption:irreducible} or Assumption \ref{assumption:ergodic}, both players have optimal stationary policies, and the game value is independent of the initial state. Thus we can simply write $\rho^*(M,s)$ as $\rho^*(M)$. For a proof of these facts, please refer to Theorem \ref{theorem:stationary_optimal} in the appendix. 

%Analogous to \cite{auer2007logarithmic}'s irreducibility assumption, Assumption \ref{assumption:irreducible} ensures that every pair of stationary policies induce an irreducible Markov chain and implies that under \textit{any} policy pair (not necessarily stationary), every state is visited within an average period of no more than $D$ steps. See Remark \ref{remark:average_time_<D}.

%Analogous to \cite{jaksch2010near}'s bounded diameter assumption, Assumption \ref{assumption:ergodic} says that whatever stationary policy Player 2 executes, Player 1 can always find a way to reach any state at most $D$ expected time steps. Clearly, if this does not hold, Player 2 may prevent Player 1 from (or make him spend a very long time before) entering some specific states, which makes the online learning problem impossible or very difficult. 
%The game value or the maximin value of the Markov game at state $s$ is defined as $\max_{\pi^1}\min_{\pi^2}\rho(M,\pi^1,\pi^2,s)$ provided that this value exists. Under Assumption \ref{assumption:irreducible} or Assumption \ref{assumption:ergodic}, the maximin value always exists, and both optimal policies can be found within $\Pi^{SR}$. 

%Throughout this paper we assume that rewards are deterministic and bounded in $[0,1]$. We use the notation $\tilde{\mathcal{O}}(\cdot)$ to hide polynomials of logarithms. For example, we write, ``with high probability,
%$g=\tilde{\mathcal{O}}(f)$" to indicate ``with probability $\geq 1-\delta,\   g=f_1\mathcal{O}(f)+f_2$'',
%where $f_1,f_2$ are some polynomials of $\log S,\log A,\log D,\log T,\log(1/\delta)$.
\subsection{Two Settings and Results Overview}
\label{section:Two Settings and Results Overview}
We focus on training Player 1 and discuss two settings. In the online setting, Player 1 competes with an arbitrary Player 2. The regret is defined as
\begin{align*}
\text{Reg}_T^{\text{(on)}}=\sum_{t=1}^{T}\rho^*(M)-r(s_t, a_t). 
\end{align*}
In the offline setting, we control both Player 1 and Player 2's actions, and find Player 1's maximin policy. The sample complexity is defined as 
\begin{align*}
L_\varepsilon=\sum_{t=1}^{T}\mathbbm{1}\{\rho^*(M)-\min_{\pi^2}\rho(M,\pi_t^1, \pi^2)>\varepsilon\},
\end{align*}
where $\pi_t^1$ is a stationary policy being executed by Player 1 at time $t$. This definition is similar to those in \cite{kakade2003sample, jaksch2010near} for one-player MDPs. By the definition of $L_\varepsilon$, if we have an upper bound for $L_\varepsilon$ and run the algorithm for $T> L_\varepsilon$ steps, there is some $t$ such that $\pi_t^1$ is $\varepsilon$-optimal. We will explain how to pick this $t$ in Section \ref{section:Sample Complexity of Offline Training} and Appendix \ref{appendix:Proofs for Offline Training Complexity}. 
%Clearly, if we can bound this sample complexity, then in the long run we have $\min_{\pi^2}\rho(M,\pi_t^1, \pi^2) \rightarrow \max_{\pi^1}\min_{\pi^2} \rho(M,\pi^1, \pi^2)$, which assures $\pi_t^1$'s convergence to optimality. 
%Here we introduce the definition of regret in the viewpoint of Player 1. %In the first setting, the two players are both controllable by the learner, and we care about how fast the learner can learn the maximin policy. The regret in this setting is the difference between the game value and the reward Player 1 would receive against an omniscient Player 2 who always knows how to best respond to $\pi^1_t$. In the second setting, the learner can only control Player 1, and 
%We target at maximizing Player 1's total reward. The regret compares her reward at time $t$ with the game value.
%\begin{definition} 
%Regret in view of Player 1 is defined as
%$$R_T\coloneqq \sum_{t=1}^T \max_{\pi^1}\min_{\pi^2}\rho(M,\pi^1,\pi^2)-\min_{\pi^2}\rho(M,\pi^1_t,\pi^2)$$
%or
%$R_T\coloneqq \sum_{t=1}^T \max_{\pi^1}\min_{\pi^2}\rho(M,\pi^1,\pi^2)-r(s_t,a_t).$
%\end{definition}

It turns out that we can use almost the same algorithm to handle these two settings. Since learning in the online setting is more challenging, from now on we will mainly focus on the online setting, and leave the discussion about the offline setting at the end of the paper.  Our results can be summarized by the following two theorems. 
\begin{theorem}
\label{theorem:irreducible}
Under Assumption \ref{assumption:irreducible}, \textsc{UCSG} achieves $\text{Reg}_T^{\text{(on)}}=\tilde{\mathcal{O}}(D^3S^5A+DS\sqrt{AT})$ \textit{w.h.p.} 
\footnote{We write, ``with high probability,
$g=\tilde{\mathcal{O}}(f)$'' or ``\textit{w.h.p.},
$g=\tilde{\mathcal{O}}(f)$'' to indicate ``with probability $\geq 1-\delta,\   g=f_1\mathcal{O}(f)+f_2$'',
where $f_1,f_2$ are some polynomials of $\log D,\log S,\log A,\log T,\log(1/\delta)$.}
\end{theorem}
\begin{theorem}
\label{regret_finite_horizon_variant}
Under Assumption \ref{assumption:ergodic}, \textsc{UCSG} achieves $\text{Reg}_T^{\text{(on)}}=\tilde{\mathcal{O}}(\sqrt[3]{DS^2AT^2})$ \textit{w.h.p}.
\end{theorem}

%#######################################################################
%   The UCSG algorithm
%#######################################################################
%#######################################################################
%   The UCSG algorithm
%#######################################################################
\section{Upper Confidence Stochastic Game Algorithm (\textsc{UCSG})}
\label{UCSG}
\begin{algorithm}[tb!]
   \caption{\textsc{UCSG}}
   \label{alg:maxmax}
\begin{algorithmic}
   \STATE {\bfseries Input:} $\mathcal{S}$, $\mathcal{A}=\mathcal{A}^1\times\mathcal{A}^2, T$.
   \STATE {\bfseries Initialization:} $t=1$.
   \FOR{phase $k=1, 2, ...$}  
   \STATE $t_k=t.$
   \STATE {\bfseries {1. Initialize phase $k$:} } $v_k(s,a)=0, \ \ \  n_k(s,a)=\max\Big\{1,\sum_{\tau=1}^{t_k-1}\mathbbm{1}_{(s_\tau, a_\tau)=(s,a)}\Big\},\ \ $
   \STATE $ \ \ \ \ \ n_k(s,a,s^\prime)=\sum_{\tau=1}^{t_k-1}\mathbbm{1}_{(s_\tau, a_\tau,s_{\tau+1})=(s,a,s^\prime)}, \ \ \hat{p}_k(s^{\prime}\vert s,a)=\frac{n_k(s,a,s^{\prime})}{n_k(s,a)}, \ \ \forall s, a, s^\prime$.
   %\erightindent
   \STATE {\bfseries 2. Update the confidence set: } $\mathcal{M}_k= \{\tilde{M}: \forall s,a, \ \ \tilde{p}(\cdot\vert s,a)\in \mathcal{P}_k(s,a)\}$, where
   \STATE \scalebox{1.0}{\ \ \ \ \ $\mathcal{P}_k(s,a)\coloneqq \textsc{conf}_1(\hat{p}_k(\cdot\vert s,a), n_k(s,a))  \cap\ \textsc{conf}_2(\hat{p}_k(\cdot\vert s,a), n_k(s,a)). $}
   %\erightindent
   \STATE {\bfseries 3. Optimistic planning: } \scalebox{1.0}{$\displaystyle\left(M_k^1, \pi^1_k\right)=\textsc{Maximin-EVI}\left(\mathcal{M}_k, \gamma_k\right), \  \text{where} \ \  \gamma_k \coloneqq 1/\sqrt{t_k}$}.
 % \STATE \scalebox{0.9}{$\displaystyle\Big[\left(M_k^2, \pi^2_k\right)=\textsc{Min-EVI}\left(\mathcal{M}_k, \pi_k^1, \gamma_k\right)\Big]$}.
   %\erightindent
   \STATE {\bfseries 4. Execute policies: }
   %\brightindent
   \STATE {\bfseries \ \ \ \ \ repeat}
   %\brightindent
   \STATE \ \ \ \ \ \ \ \ \ \ Draw $a_t^1\sim \pi^1_k(\cdot\vert s_t) $; observe the reward $r_t$ and the next state $s_{t+1}$.
   \STATE \ \ \ \ \ \ \ \ \ \  Set $v_k(s_t,a_t)=v_k(s_t,a_t)+1$ and $t=t+1$.
   %\erightindent
   \STATE {\bfseries \ \ \ \ \ until} $\exists (s,a)$ such that $v_k(s,a)=n_k(s,a)$
   %\erightindent
   \ENDFOR
   %\eleftindent
   %\bleftindent
   \STATE {\bfseries Definitions of confidence regions:}
   \STATE $\textsc{conf}_1(\hat{p}, n)\coloneqq \bigg\{\tilde{p} \in \left[0,1\right]^{S}:  \norm{\tilde{p}-\hat{p}}_1 \leq \sqrt{\frac{2S\ln(1/\delta_1)}{n}} \bigg\},\ \  \delta_1=\frac{\delta}{2S^2A\log_2 T}.$
   \STATE $\textsc{conf}_2(\hat{p}, n)\coloneqq\bigg\{ \tilde{p}\in [0,1]^S: \forall i, \big\vert \sqrt{\tilde{p}_i(1-\tilde{p}_i)}-\sqrt{\hat{p}_i(1-\hat{p}_i)}\big\vert\leq \sqrt{\frac{2\ln(6/\delta_1)}{n-1}},$
   \STATE \scalebox{1.0}{$ \ \ \ \ \ \ \ \ \ \ \ \ \ \ \ \ \ \ \ \ \ \ \ \ \ \ \ \ \  \ \ \ \ \ \ \ \ \ \ \ \ \ \ \  |\tilde{p}_i-\hat{p_i}|\leq \min\Big(\sqrt{\frac{\ln({6}/{\delta_1})}{2n}},\sqrt{\frac{2\hat{p}_i(1-\hat{p}_i)}{n}\ln\frac{6}{\delta_1}}+\frac{7}{3(n-1)}\ln\frac{6}{\delta_1}\Big)\bigg\}$}.

   %\eleftindent
\end{algorithmic}
\end{algorithm}

The Upper Confidence Stochastic Game algorithm (\textsc{UCSG}) (Algorithm \ref{alg:maxmax}) extends \textsc{UCRL2} \cite{jaksch2010near}, using the optimism principle to balance exploitation and exploration. It proceeds in phases (indexed by $k$), and only changes the learner's policy $\pi_k^1$ at the beginning of each phase. The length of each phase is not fixed a priori, but depends on the statistics of past observations.  
%maintains a confidence set of plausible models given the observations so far (Step 2), and finds an optimistic model $M_k^1$ among the plausible models (Step3). Then it chooses the optimal policy $\pi_k^1$ for this optimistic model $M_k^1$, and execute it for a while (Step 4).  

In the beginning of each phase $k$, the algorithm estimates the transition probabilities using empirical frequencies $\hat{p}_k(\cdot\vert s,a)$ observed in previous phases (Step 1). With these empirical frequencies, it can then create a confidence region $\mathcal{P}_k(s,a)$ for each transition probability. The transition probabilities lying in the confidence regions constitute a set of plausible stochastic game models $\mathcal{M}_k$, where the true model $M$ belongs to with high probability (Step 2). Then, Player 1 optimistically picks one model $M_k^1$ from $\mathcal{M}_k$, and finds the optimal (stationary) policy $\pi_k^1$ under this model (Step 3). Finally, Player 1 executes the policy $\pi_k^1$ for a while until some $(s,a)$-pair's number of occurrences is doubled during this phase (Step 4). The count $v_k(s,a)$ records the number of steps the $(s,a)$-pair is observed in phase $k$; it is reset to zero in the beginning of every phase. 

In Step 3, to pick an optimistic model and a policy is to pick $M_k^1\in\mathcal{M}_k$ and $\pi_k^1\in \Pi^{\text{SR}}$ such that $\forall s$, 
\begin{align}
\scalebox{1.0}{$\displaystyle\min_{\pi^2}\rho(M_k^1, \pi_k^1, \pi^2, s) \geq \max_{\tilde{M}\in\mathcal{M}_k}\rho^*(\tilde{M},s)-\gamma_k. \label{eqn:gamma_optimal}$}
\end{align} 
where $\gamma_k$ denotes the error parameter for \textsc{Maximin-EVI}. The LHS of \eqref{eqn:gamma_optimal} is well-defined because Player 2 has stationary optimal policy under the MDP induced by $M_k^1$ and $\pi_k^1$. Roughly speaking, \eqref{eqn:gamma_optimal} says that $\displaystyle\min_{\pi^2}\rho(M_k^1, \pi_k^1, \pi^2, s)$ should approximate $\displaystyle \max_{\tilde{M}\in\mathcal{M}_k, \pi^1} \min_{\pi^2}\rho(\tilde{M}, \pi^1, \pi^2, s)$ by an error no more than $\gamma_k$. %(in Section \ref{subsection:extended SG} we will see that $\rho^*(\tilde{M},s)$ is equal to $\max_{\pi^1}\min_{\pi^2}\rho(\tilde{M},\pi^1,\pi^2,s)$). 
That is, $(M_k^1, \pi_k^1)$ are picked optimistically in $\mathcal{M}_k\times \Pi^{\text{SR}}$ considering the most adversarial opponent.  

\subsection{Extended SG and Maximin-EVI}
\label{subsection:extended SG}

The calculation of $M_k^1$ and $\pi_k^1$ involves the technique of Extended Value Iteration (EVI), which also appears in \cite{jaksch2010near} as a one-player version.  

Consider the following SG, named $M^+$. Let the state space $\mathcal{S}$ and Player 2's action space $\mathcal{A}^2$ remain the same as in $M$. Let $\mathcal{A}^{1+}$, $p^+(\cdot\vert \cdot,\cdot,\cdot)$, $r^+(\cdot, \cdot,\cdot)$ be Player 1's action set, the transition kernel, and the reward function of $M^+$, such that for any $a^1\in \mathcal{A}^1$ and $a^2\in \mathcal{A}^2$ and an admissible transition probability $\tilde{p}(\cdot\vert s,a^1,a^2)\in \mathcal{P}_k(s,a^1,a^2)$, there is an action $a^{1+}\in \mathcal{A}^{1+}$ such that $p^+(\cdot\vert s,a^{1+}, a^2)=\tilde{p}(\cdot\vert s,a^1,a^2)$ and $r^+(s,a^{1+},a^2)=r(s,a^1,a^2)$. In other words, Player 1 selecting an action in $\mathcal{A}^{1+}$ is equivalent to selecting an action in $\mathcal{A}^1$ and simultaneously selecting an admissible transition probability in the confidence region $\mathcal{P}_k(\cdot,\cdot)$. 

Suppose that $M\in \mathcal{M}_k$, then the extended SG $M^+$ satisfies Assumption \ref{assumption:ergodic} because the true model $M$ is embedded in $M^+$. 
By Theorem \ref{theorem:stationary_optimal} in Appendix \ref{stationary_optimal_policy}, it has a constant game value $\rho^*(M^+)$ independent of the initial state, and satisfies Bellman equation of the form $\val\{r+Pf\}=\rho\cdot\mathbf{e}+f$, for some bounded function $f(\cdot)$, where $\mathbf{e}$ stands for the all-one constant vector. With the above conditions, we can use value iteration with Schweitzer transform (a.k.a. aperiodic transform)\cite{van1980successive} to solve the optimal policy in the extended EG $M^+$. We call it \textsc{Maximin-EVI}. For the details of \textsc{Maximin-EVI}, please refer to Appendix \ref{appendix:maximinEVI}. We only summarize the result with the following Lemma. 

\begin{lemma}
\label{lemma:EVI_short}
Suppose the true model $M\in \mathcal{M}_k$, then the estimated model $M_k^1$ and stationary policy $\pi_k^1$ output by \textsc{Maximin-EVI} in Step 3 satisfy 
\begin{align*}
\forall s, \ \ \ \min_{\pi^2}\rho(M_k^1, \pi_k^1, \pi^2, s) \geq \max_{\pi^1}\min_{\pi^2} \rho(M, \pi^1, \pi^2, s) - \gamma_k. 
\end{align*} 
\end{lemma}

Before diving into the analysis under the two assumptions, we first establish the following fact. 
\begin{lemma}
\label{lemma:bound_fail}
With high probability, the true model $M\in\mathcal{M}_k$ for all phases $k$. 
\end{lemma}
It is proved in Appendix \ref{section:Lemmas for Failing Events}. With Lemma \ref{lemma:bound_fail}, we can fairly assume $M\in \mathcal{M}_k$ in most of our analysis. 

\section{Analysis under Assumption \ref{assumption:irreducible}}
\label{section:Analysis under Assumption 1}
In this section, we import analysis techniques from one-player MDPs \cite{auer2007logarithmic, jaksch2010near, lattimore2012pac, dann2015sample}. We also develop some techniques that deal with non-stationary opponents. 

We model Player 2's behavior in the most general way, i.e., assuming it using a history-dependent randomized policy. Let $H_{t}=(s_1, a_1, r_1, ..., s_{t-1}, a_{t-1},r_{t-1}, s_t)\in \mathcal{H}_{t}$ be the history up to $s_t$, then we assume $\pi_t^2$ to be a mapping from $\mathcal{H}_{t}$ to a distribution over $\mathcal{A}^2$. We will simply write $\pi_t^2(\cdot)$ and hide its dependency on $H_{t}$ inside the subscript $t$. A similar definition applies to $\pi_t^1(\cdot)$. With abuse of notations, we denote  by $k(t)$ the phase where step $t$ lies in, and thus our algorithm uses policy $\pi_t^1(\cdot)=\pi_{k(t)}^1(\cdot\vert s_t)$. The notations $\pi_t^1$ and $\pi_k^1$ are used interchangeably.
Let $T_k\coloneqq t_{k+1}-t_k$ be the length of phase $k$. We decompose the regret in phase $k$ in the following way: 
\begin{align}
\Lambda_k \coloneqq T_k \rho^*(M)-\sum_{t=t_k}^{t_{k+1}-1}r(s_t,a_t)=\sum_{n=1}^{4} \Lambda_k^{(n)},  \label{eqn:decompose_interval_regret}
\end{align}
in which we define
\begin{align*}
&\Lambda_k^{(1)}=T_k\left( \rho^*(M) - \min_{\pi^2}\rho(M_k^1, \pi_k^1, \pi^2, s_{t_k})\right), \nonumber \\
& \Lambda_k^{(2)}=T_k\left( \min_{\pi^2}\rho(M_k^1, \pi_k^1, \pi^2,  s_{t_k}) - \rho(M_k^1, \pi_k^1, \bar{\pi}_k^{2},  s_{t_k})\right), \nonumber\\
&\Lambda_k^{(3)}=T_k\left(  \rho(M_k^1, \pi_k^1,  \bar{\pi}_k^{2},  s_{t_k}) - \rho(M, \pi_k^1,  \bar{\pi}_k^{2})\right), \nonumber \\
&\Lambda_k^{(4)}=T_k\rho(M, \pi_k^1, \bar{\pi}_k^{2}) - \sum_{t=t_k}^{t_{k+1}-1}r(s_t,a_t),  
\end{align*}
where $ \bar{\pi}_k^{2}$ is some stationary policy of Player 2 which will be defined later. Since the actions of Player 2 are arbitrary, $ \bar{\pi}_k^{2}$ is imaginary and only exists in analysis. Note that under Assumption \ref{assumption:irreducible}, any stationary policy pair over $M$ induces an irreducible Markov chain, so we do not need to specify the initial states for $\rho(M,\pi_k^1, \bar{\pi}_k^2)$ in \eqref{eqn:decompose_interval_regret}. Among the four terms, $\Lambda_k^{(2)}$ is clearly non-positive, and $\Lambda_k^{(1)}$, by optimism, can be bounded using Lemma \ref{lemma:EVI_short}. Now remains to bound $\Lambda_k^{(3)}$ and $\Lambda_k^{(4)}$. 

%=============================================
% UCSG: bounding time steps with inaccurate probabilities
%=============================================

\subsection{Bounding $\sum_k \Lambda_k^{(3)}$ and $\sum_k \Lambda_k^{(4)}$}
\label{Regret decomposition and the introduction of pi2}

\textbf{The Introduction of $\bar{\pi}_k^{2}$.}\ \ \ \ \ $\Lambda_k^{(3)}$ and $\Lambda_k^{(4)}$ involve the artificial policy $\bar{\pi}_k^{2}$, which is a stationary policy that replaces Player 2's non-stationary policy in the analysis. This replacement costs some constant regret but facilitates the use of perturbation analysis in regret bounding. The selection of $\bar{\pi}_k^{2}$ is based on the principle that the behavior (e.g., total number of visits to some $(s,a)$) of the Markov chain induced by $M, \pi_k^1, \bar{\pi}_k^2$ should be close to the empirical statistics. Intuitively, $\bar{\pi}_k^{2}$ can be defined as
\begin{align}
\bar{\pi}_k^{2}(a^2\vert s) \coloneqq \frac{\sum_{t=t_k}^{t_{k+1}-1}\mathbbm{1}_{s_t=s}\pi_t^2(a^2)}{\sum_{t=t_k}^{t_{k+1}-1} \mathbbm{1}_{s_t=s}}. \label{eqn:pi2star} 
\end{align} 
Note two things, however. First, since we need the actual trajectory in defining this policy, it can only be defined \textit{after} phase $k$ has ended. Second, $\bar{\pi}_k^{2}$ can be undefined because the denominator of \eqref{eqn:pi2star} can be zero. 
However, this will not happen in too many steps. Actually, we have
\begin{lemma}
\label{lemma:undefined}
$\sum_{k} T_k \mathbbm{1}\{\bar{\pi}_k^2 \text{ not well-defined}\}$$\leq \tilde{\mathcal{O}}(DS^2A)$ with high probability. 
\end{lemma}

Before describing how we bound the regret with the help of $\bar{\pi}_k^{2}$ and the perturbation analysis, we establish the following lemma: 

\begin{lemma}
\label{lemma:epsilon_accurate}
We say the transition probability at time step $t$ is $\varepsilon$-accurate if $\vert p_k^1(s^\prime\vert s_t, \pi_t) - p(s^\prime\vert s_t, \pi_t) \vert \leq \varepsilon \ \forall s^\prime$ where $p_k^1$ denotes the transition kernel of $M_k^1$. We let $B_t(\varepsilon)=1$ if the transition probability at time $t$ is $\varepsilon$-accurate; otherwise $B_t(\varepsilon)=0$. Then for any state $s$, with high probability,  $\sum_{t=1}^{T} \mathbbm{1}_{s_t=s}\mathbbm{1}_{B_t(\varepsilon)=0} \leq \tilde{\mathcal{O}}\left(A/\varepsilon^2\right).$
\end{lemma}

%After bounding the number of inaccurately estimated transition probabilities, \cite{auer2007logarithmic} uses the fact ``accurate transition probabilities lead to accurate average return'' to bound regret. We instead use ``accurate transition probabilities lead to accurate first passage time.'' Since first passage time is related to the span of the bias vector, we can continue to use approaches similar to \cite{jaksch2010near} to get tighter regret bound.
%Bounding the number of steps with inaccurately estimated transition probability also appears in Section 3.3.1 of \cite{auer2007logarithmic}. They then use standard perturbation analysis for stationary distribution (Theorem \ref{theorem:perturbation_stationary_distribution}) to bound the error in average reward, which is further related to regret. Their approach leads to the regret bound of order $\tilde{\mathcal{O}}(\sqrt{D^3S^5AT})$. We improve their analysis by developing perturbation bounds for mean first passage times (Theorem \ref{theorem:perturbation_first_passage_time}). This, combined with the analysis from \cite{jaksch2010near}, improves the asymptotic regret to $\tilde{\mathcal{O}}(DS\sqrt{AT})$. More details are presented later.

%=========================================================
%  Bounding the regret in steps with accurately estimated transition probabilities
%=========================================================

Now we are able to sketch the logic behind our proofs. Let's assume that $\bar{\pi}_k^2$ models $\pi_k^2$ quite well, i.e., the expected frequency of every state-action pair induced by $M, \pi_k^1, \bar{\pi}_k^2$ is close to the empirical frequency induced by $M, \pi_k^1, \pi_k^2$. Then clearly, $\Lambda_k^{(4)}$ is close to zero in expectation. The term $\Lambda_k^{(3)}$ now becomes the difference of average reward between two Markov reward processes with slightly different transition probabilities. This term has a counterpart in \cite{jaksch2010near} as a single-player version. Using similar analysis, we can prove that the dominant term of $\Lambda_k^{(3)}$ is proportional to $\spa (h(M_k^1, \pi_k^1, \bar{\pi}_k^2, \cdot))$. In the single-player case, \cite{jaksch2010near} can directly claim that $\spa (h(M_k^1, \pi_k^1, \cdot)) \leq D$ (see their Remark 8), but unfortunately, this is not the case in the two-player version. \footnote{The argument in \cite{jaksch2010near} is simple: suppose that $h(M_k^1,\pi_k^1,s)-h(M_k^1,\pi_k^1,s^\prime) > D$, by the communicating assumption, there is a path from $s^\prime$ to $s$ with expected time no more than $D$. Thus a policy that first goes from $s^\prime$ to $s$ within $D$ steps and then executes $\pi_k^1$ will outperform $\pi_k^1$ at $s^\prime$. This leads to a contradiction. In two-player SGs, with a similar argument, we can also show that $\spa(h(M_k^1, \pi_k^1, \pi_k^{2*}, \cdot)) \leq D$, where $\pi_k^{2*}$ is the best response to $\pi_k^1$ under $M_k^1$. However, since Player 2 is uncontrollable, his/her policy $\pi_k^2$ (or $\bar{\pi}_k^2$) can be quite different from $\pi_k^{2*}$, and thus $\spa(h(M_k^1, \pi_k^1, \bar{\pi}_k^2, \cdot)) \leq D$ does not necessarily hold true.}

To continue, we resort to the perturbation analysis for the mean first passage times (developed in Appendix \ref{subsection:perturbation}). Lemma \ref{lemma:epsilon_accurate} shows that $M_k^1$ will not be far from $M$ for too many steps. Then Theorem \ref{theorem:perturbation_first_passage_time} in Appendix \ref{subsection:perturbation} tells that if $M_k^1$ are close enough to $M$, $T^{\pi_k^1, \bar{\pi}_k^2}(M_k^1)$ can be bounded by $2T^{\pi_k^1, \bar{\pi}_k^2}(M)$. As Remark \ref{remark:mean_first_passage_time} implies that $\spa(h(M_k^1, \pi_k^1, \bar{\pi}_k^2, \cdot)) \leq T^{\pi_k^1, \bar{\pi}_k^2}(M_k^1)$ and Assumption \ref{assumption:irreducible} guarantees that $ T^{\pi_k^1, \bar{\pi}_k^2}(M) \leq D$, we have  $\spa(h(M_k^1, \pi_k^1, \bar{\pi}_k^2, \cdot)) \leq T^{\pi_k^1, \bar{\pi}_k^2}(M_k^1)\leq 2T^{\pi_k^1, \bar{\pi}_k^2}(M) \leq 2D$. 

%The following lemma demonstrates the function of $\bar{\pi}_k^2$: except for a certain amount of steps, $\bar{\pi}_k^2$  approximates Player 2's non-stationary policy, which enables us to continue the analysis as if Player 2 executes a stationary policy.   

The above approach leads to Lemma \ref{lemma:bound_benign}, which is a key in our analysis. We first define some notations. Under Assumption \ref{assumption:irreducible}, any pair of stationary policies induces an irreducible Markov chain, which has a unique stationary distribution. If the policy pair $\pi=(\pi^1,\pi^2)$ is executed, we denote its stationary distribution by $\mu(M,\pi^1,\pi^2, \cdot)=\mu(M,\pi,\cdot)$. Besides, denote $v_k(s)\coloneqq \sum_{t=t_k}^{t_{k+1}-1} \mathbbm{1}_{s_t=s}$. 
 
We say a phase $k$ is benign if the following hold true: the true model $M$ lies in $\mathcal{M}_k$, $\bar{\pi}_k^{2}$ is well-defined, $\spa(h(M_k^1, \pi_k^1, \bar{\pi}_k^2, \cdot))\leq 2D$, and $\mu(M,\pi_k^1, \bar{\pi}_k^2, s) \leq \frac{2v_k(s)}{T_k} \ \forall s$. We can show the following: 
\begin{lemma}
\label{lemma:bound_benign}
$\sum_{k} T_k \mathbbm{1}\{\text{phase\ } k \text{\ is not benign}\}\leq $$\tilde{\mathcal{O}}(D^3S^5A)$ with high probability. 
%With high probability, all phases are benign except for at most $\tilde{\mathcal{O}}(D^3S^5A)$ time steps.  
\end{lemma}

Finally, for benign phases, we can have the following two lemmas.  
%\subsubsection{Bounding the fourth term of \eqref{eqn:decompose_interval_regret}}
\begin{lemma}
\label{lemma:bound_fourth_term}
$\sum_{k}  \Lambda_k^{(4)} \mathbbm{1}\{\bar{\pi}_k^{2} \text{\ is well-defined\ }\}$$\leq \tilde{\mathcal{O}}(D\sqrt{ST}+DSA) $ with high probability. 
%\begin{align*}
%\sum_{k}  \Lambda_k^{(4)} \mathbbm{1}\{\bar{\pi}_k^{2} \text{\ well-defined\ }\}\leq \tilde{O}(D\sqrt{ST}+DSA).
%\end{align*}
\end{lemma}

%\subsubsection{Bounding the third term of \eqref{eqn:decompose_interval_regret}}
\begin{lemma}
\label{lemma:typical} 
$\sum_{k} \Lambda_k^{(3)}\mathbbm{1}\{\text{phase\ } k \text{\ is benign}\}\leq $$\tilde{\mathcal{O}}(DS\sqrt{AT}+DS^2A)$ with high probability, 
%\begin{align*}
%\sum_{k: \text{benign}} \Lambda_k^{(3)} \leq \tilde{\mathcal{O}}(DS\sqrt{AT}+DS^2A).
%&\sum_{k: \text{benign}}T_k\left(\rho(M_k^1, \bar{\pi}_k,s_{t_k})-\rho(M, \bar{\pi}_k)\right)\\
%&\ \ \ \ \ \ \ \ \ \ \ \ \ \ \ \ \ \ \ \ \ \ \ \ \ \ \ \ \ \ \ \ \ \ \ \ \ \ \ \ \ \ \leq \tilde{\mathcal{O}}(DS\sqrt{AT}+DS^2A).
%\end{align*}
\end{lemma}

\begin{proof}[Proof of Theorem \ref{theorem:irreducible}]
The regret proof starts from the decomposition of \eqref{eqn:decompose_interval_regret}. $\Lambda_k^{(1)}$ is bounded with the help of Lemma \ref{lemma:EVI_short}: $\sum_k \Lambda_k^{(1)} \leq \sum_k T_k/\sqrt{t_k} = \mathcal{O}(\sqrt{T})$. $\sum_k\Lambda_k^{(2)}\leq 0$ by definition. Then with Lemma \ref{lemma:undefined},  \ref{lemma:bound_benign}, \ref{lemma:bound_fourth_term}, and \ref{lemma:typical}, we can bound $\Lambda_k^{(3)}$ and $\Lambda_k^{(4)}$ by $\tilde{\mathcal{O}}(D^3S^5A+DS\sqrt{AT})$.
\end{proof}

%#####################################################################################################################################################
\section{Analysis under Assumption \ref{assumption:ergodic}}
\label{section:Analysis under Assumption 2}
%#####################################################################################################################################################
In Section \ref{section:Analysis under Assumption 1}, the main ingredient of regret analysis lies in bounding the span of the bias vector, $\text{sp}(h(M_k^1, \pi_k^1, \bar{\pi}_k^2, \cdot))$. However, the same approach does not work because under the weaker Assumption \ref{assumption:ergodic}, we do not have a bound on the mean first passage time under arbitrary policy pairs. Hence we adopt the approach of approximating the average reward SG problem by a sequence of finite-horizon SGs: on a high level, first, with the help of Assumption \ref{assumption:ergodic}, we approximate the $T$ multiple of the original average-reward SG game value (i.e. the total reward in hindsight) with the sum of those of $H$-step episodic SGs; second, we resort to \cite{dann2015sample}'s results to bound the $H$-step SGs' sample complexity and translates it to regret. 

\textbf{Approximation by repeated episodic SGs.}\ \ \ For the approximation, the quantity $H$ does not appear in \textsc{UCSG} but only in the analysis. The horizon $T$ is divided into episodes each with length $H$. Index episodes with $i=1, ..., T/H$, and denote episode $i$'s first time step by $\tau_i$. We say $i\in \text{ph}(k)$ if all $H$ steps of episode $i$ lie in phase $k$. Define the $H$-step expected reward under joint policy $\pi$ with initial state $s$ as $V_H(M,\pi,s)\coloneqq \mathbb{E}\left[\sum_{t=1}^H r_t| a_t\sim \pi ,s_1=s\right]$. %\footnote{Note that $\pi$ can be any policy in $\Pi^{\text{HR}}$.}
Now we decompose the regret in phase $k$ as 
\begin{align}
&\Delta_k \coloneqq T_k\rho^* - \sum_{t=t_k}^{t_{k+1}-1} r(s_t,a_t)\leq\sum_{n=1}^{6} \Delta_k^{(n)},\label{eqn:fh_decompose}
\end{align}
where
\begin{align}
&\scalebox{1}{$\Delta_k^{(1)}= \sum_{i\in\text{ph}(k)}H \Big(\rho^*-\min_{\pi^2}\rho(M_k^1,\pi^1_k,\pi^2,s_{\tau_i})\Big),$}\nonumber\\
&\scalebox{1}{$\Delta_k^{(2)}=\sum_{i\in\text{ph}(k)} \Big(H\min_{\pi^2}\rho(M_k^1,\pi^1_k,\pi^2,s_{\tau_i})-\min_{\pi^2}V_H(M_k^1, \pi_k^1, \pi^2, s_{\tau_i})\Big),$} \nonumber \\
&\scalebox{1}{$\Delta_k^{(3)}=\sum_{i\in\text{ph}(k)} \Big( \min_{\pi^2}V_H(M_k^1, \pi_k^1, \pi^2, s_{\tau_i}) - V_H(M_k^1, \pi_k^1, \pi^2_i, s_{\tau_i})\Big),$}\nonumber\\
&\scalebox{1}{$\Delta_k^{(4)}=\sum_{i\in\text{ph}(k)}\Big(V_H(M_k^1, \pi_k^1, \pi_i^2, s_{\tau_i})-V_H(M, \pi_k^1, \pi_i^2, s_{\tau_i})\Big),$}\nonumber\\
&\scalebox{1}{$\Delta_k^{(5)}=\sum_{i\in\text{ph}(k)}\Big(V_H(M,\pi_k^1, \pi_i^2, s_{\tau_i})-\sum_{t=\tau_i}^{\tau_{i+1}-1}r(s_t,a_t)\Big), \ \ \  \Delta_k^{(6)}=2H$}. \nonumber
\end{align}
Here, $\pi_i^2$ denotes Player 2's policy in episode $i$, which may be non-stationary. $\Delta_k^{(6)}$ comes from the possible two incomplete episodes in phase $k$. $\Delta_k^{(1)}$ is related to the tolerance level we set for the \textsc{Maximin-EVI} algorithm: $\Delta_k^{(1)}\leq T_k \gamma_k=T_k/\sqrt{t_k}$. $\Delta_k^{(2)}$ is an error caused by approximating an infinite-horizon SG by a repeated episodic $H$-step SG (with possibly different initial states). $\Delta_k^{(3)}$ is clearly non-positive. It remains to bound $\Delta_k^{(2)},\Delta_k^{(4)}$ and $\Delta_k^{(5)}$. 

\begin{lemma}
\label{theorem:FH_fifth_term}
By Azuma-Hoeffding's inequality, $\sum_k \Delta_k^{(5)}\leq \tilde{\mathcal{O}}(\sqrt{HT})$ with high probability. 
%\begin{align*}
%\sum_k \Delta_k^{(5)}\leq \tilde{\mathcal{O}}(\sqrt{T}).
%&\sum_k\sum_{i\in\text{ph}(k)}\left(V_H(M,\pi_k^1, \pi_i^2, s_{\tau_i})-\sum_{t=\tau_i}^{\tau_{i+1}-1}r(s_t,a_t)\right) \\
%&\ \ \ \ \ \ \ \ \ \ \ \ \ \ \ \ \ \ \ \ \ \ \ \ \ \ \ \ \ \ \ \ \ \ \ \ \ \ \ \ \ \ \ \ \ \ \ \ \ \ \ \ \ \ \ \ \ \ \ \ \ \ \ \ \ \leq \tilde{\mathcal{O}}(\sqrt{T}).
%\end{align*}
\end{lemma}

%\subsection{Bounding the second term of \eqref{eqn:fh_decompose}}
\begin{lemma}
\label{lemma:FH_second_term}
Under Assumption \ref{assumption:ergodic}, $\sum_k \Delta_k^{(2)}\leq TD/H+\sum_{k}T_k\gamma_k.$
%\begin{align*}
%&\scalebox{0.85}{$\displaystyle\sum_k\sum_{i\in\text{ph}(k)}\left(H\min_{\pi^2}\rho(M^1_k,\pi^1_k,\pi^2,s_{\tau_i})-\min_{\pi^2}V_H(M_k^1, \pi_k^1, \pi^2, s_{\tau_i} )\right)$}\nonumber \\
%&\ \ \ \ \ \ \ \ \ \ \ \ \ \ \ \ \ \ \ \ \ \ \ \ \ \ \ \ \ \ \ \ \ \ \ \ \ \ \ \ \ \ \ \ \leq TD/H+\sum_{k}T_k\gamma_k.
%\end{align*}
%where $\gamma_k$ specifies the tolerance of the stopping criterion in phase $k$.
\end{lemma}

\textbf{From sample complexity to regret bound.}\ \ \ \ \  As the main contributor of regret, $\Delta_k^{(4)}$ corresponds to the inaccuracy in the transition probability estimation. Here we largely reuse \cite{dann2015sample}'s results where they consider one-player episodic MDP with a fixed initial state distribution. Their main lemma states that the number of episodes in phases such that $\lvert V_H(M_k^1,\pi_k, s_0)-V_H(M,\pi_k, s_0)\rvert>\varepsilon$ will not exceed $\tilde{\mathcal{O}}\left(H^2S^2A/\varepsilon^2\right)$, where $s_0$ is their initial state in each episode. In other words, $\sum_k \frac{T_k}{H}\mathbbm{1}\{\lvert V_H(M_k^1,\pi_k, s_0)-V_H(M,\pi_k, s_0)\rvert>\varepsilon\}=\tilde{\mathcal{O}}(H^2S^2A/\varepsilon^2).$ Note that their proof allows $\pi_k$ to be an arbitrarily selected non-stationary policy for phase $k$.  

We can directly utilize their analysis and we summarize it as Theorem \ref{theorem:PAC} in the appendix. While their algorithm has an input $\varepsilon$, this input can be removed without affecting bounds. This means that the PAC bounds holds for arbitrarily selected $\varepsilon$. With the help of Theorem \ref{theorem:PAC}, we have
\begin{lemma}
\label{theorem:FH_fourth_term}
$\sum_k \Delta_k^{(4)} \leq \tilde{\mathcal{O}}(S\sqrt{HAT}+HS^2A)$ with high probability. 
%\begin{align*}
%&\sum_k \sum_{i\in \text{ph}(k)} V_H(M_k^1, \pi_k^1, \pi_i^2, s_{\tau_i}) - V_H(M, \pi_k^1, \pi_i^2, s_{\tau_i}) \nonumber \\
%&\ \ \ \ \ \ \ \ \ \ \ \ \ \ \ \ \ \ \ \ \ \ \ \ \ \ \ \ \ \ \ \ \ \ \ \ \ \ \ \ \ \ \ = \tilde{\mathcal{O}}(S\sqrt{HAT}+HS^2A).
%\end{align*}
\end{lemma}

\begin{proof}[Proof of Theorem \ref{regret_finite_horizon_variant}]
With the decomposition \eqref{eqn:fh_decompose} and the help of Lemma \ref{theorem:FH_fifth_term}, \ref{lemma:FH_second_term}, and \ref{theorem:FH_fourth_term}, the regret is bounded by $\tilde{O}(\frac{TD}{H}+S\sqrt{HAT}+S^2AH)=\tilde{\mathcal{O}}(\sqrt[3]{DS^2AT^2})$ by selecting $H=\max\{D, \sqrt[3]{D^2T/(S^2A)}\}$.
\end{proof}
%#############################################################
%Offline Training
%#############################################################

\section{Sample Complexity of Offline Training}
\label{section:Sample Complexity of Offline Training}
In Section \ref{section:Two Settings and Results Overview}, we defined $L_\varepsilon$ to be the sample complexity of Player 1's maximin policy. In our offline version of \textsc{UCSG}, in each phase $k$ we let both players each select their own optimistic policy. After Player 1 has optimistically selected $\pi^1_k$, Player 2 then optimistically selects his policy $\pi^2_k$ based on the known $\pi_k^1$. Specifically, the model-policy pair $\left(M_k^2, \pi_k^2\right)$ is obtained by another extended value iteration on the extended MDP under fixed $\pi_k^1$, where Player 2's action set is extended. By setting the stopping threshold also as $\gamma_k$, we have 
\begin{align}
\rho(M_k^2, \pi_k^1, \pi_k^2, s) \leq \min_{\tilde{M}\in  \mathcal{M}_k} \min_{\pi^2} \rho(\tilde{M}, \pi_k^1, \pi^2, s)+\gamma_k \label{eqn:pessimistic_selection}
\end{align}
when value iteration halts. 
%To bound $\text{L}_\varepsilon$, we define
%\begin{align}
%&\text{Reg}_\varepsilon^{\text{(off)}}\coloneqq \sum_{k\in K_\varepsilon} T_k(\rho^*(M)-\min_{\pi^2}\rho(M,\pi_t^1, \pi^2)) \nonumber \\
%&= \text{Reg}_\varepsilon^{\text{(on)}} + \sum_{k\in K_\varepsilon} \sum_{t=t_k}^{t_{k+1}-1} \left(r_t - \min_{\pi^2}\rho(M ,\pi_k^1, \pi^2)\right). \nonumber \\
%& + \sum_{k\in K_\varepsilon}\sum_{t=t_k}^{t_{k+1}-1}  \left(\rho(M_k^2, \pi_k^1, \pi_k^2, s_{t}) - \min_{\pi^2}\rho(M, \pi_k^1, \pi^2)\right), \nonumber
%\end{align}
%where $K_\varepsilon\coloneqq \{k: \rho^*(M)-\min_{\pi^2}\rho(M,\pi_k^1, \pi^2) > \varepsilon\}$, and $\text{Reg}_\varepsilon^{\text{(on)}}$ is defined as a summation similar to $\text{Reg}_T^{\text{(on)}}$ except that it is summed only over time steps in phases $k\in K_\varepsilon$.
%We argue that $\text{Reg}_\varepsilon^{\text{(off)}}$'s order does not exceed that of $\text{Reg}_\varepsilon^{\text{(on)}}$, and $\text{Reg}_\varepsilon^{\text{(on)}}$'s upper bound is similar to that of $\text{Reg}_T^{\text{(on)}}$ except that the dependency on $T$ is changed to $L_\varepsilon$. Then we can use the already established bound for  $\text{Reg}_\varepsilon^{\text{(on)}}$ to upper bound $L_\varepsilon$. The first summation is similar to \eqref{eqn:decompose_interval_regret}'s last two terms, or the main term in one-player MDP's regret analysis (e.g., in \cite{jaksch2010near}). For completeness, we provide its bound in the appendix. Formally, we have the following two Theorems. 
With this selection rule, we are able to obtain the following theorems. 
\begin{theorem}
\label{theorem:offline_irreducible}
Under Assumption \ref{assumption:irreducible}, \textsc{UCSG} achieves $L_\varepsilon=\tilde{\mathcal{O}}(D^3S^5A+D^2S^2A/\varepsilon^2)$ \textit{w.h.p.} 
\end{theorem}
\begin{theorem}
\label{theorem:offline_ergodic}
Let Assumption \ref{assumption:ergodic} hold, and further assume that $\displaystyle\max_{s,s'}\max_{\pi^1\in \Pi^{\text{SR}}}\min_{\pi^2\in \Pi^{\text{SR}}}T_{s\rightarrow s'}^{\pi^1, \pi^2}(M) \leq D$. Then \textsc{UCSG} achieves $L_\varepsilon=\tilde{\mathcal{O}}(DS^2A/\varepsilon^3)$ \textit{w.h.p.}
\end{theorem}
The algorithm can output a single stationary policy for Player 1 with the following guarantee: 
if we run the offline version of \textsc{UCSG} for $T>L_\varepsilon$ steps, the algorithm can output a single stationary policy that is $\varepsilon$-optimal. We show how to output this policy in the proofs of Theorem \ref{theorem:offline_irreducible} and \ref{theorem:offline_ergodic}. 

\section{Open Problems}
In this work, we obtain the regret of $\tilde{\mathcal{O}}(D^3S^5A+DS\sqrt{AT})$ and $\tilde{\mathcal{O}}(\sqrt[3]{DS^2AT})$ under different mixing assumptions. A natural open problem is how to improve these bounds on both asymptotic and constant terms. A lower bound of them can be inherited from the one-player MDP setting, which is $\Omega(\sqrt{DSAT})$ \cite{jaksch2010near}. 

Another open problem is that if we further weaken the assumptions to $\max_{s,s^\prime}\min_{\pi_1}\min_{\pi_2}T^{\pi^1,\pi^2}_{s\rightarrow s^\prime}\leq D$, can we still learn the SG? We have argued that if we only have this assumption, in general we cannot get sublinear regret in the online setting. However, it is still possible to obtain polynomial-time offline sample complexity if the two players cooperate to explore the state-action space. 

\section*{Acknowledgments}
We would like to thank all anonymous reviewers who have devoted their time for reviewing this work and giving us valuable feedbacks. We would like to give special thanks to the reviewer who reviewed this work's previous version in ICML; your detailed check of our proofs greatly improved the quality of this paper.

\newpage
\bibliographystyle{plain}
\bibliography{nips_2017}

\begin{thebibliography}{10}

\bibitem{abbasi2013online}
Yasin Abbasi, Peter~L Bartlett, Varun Kanade, Yevgeny Seldin, and Csaba
  Szepesv{\'a}ri.
\newblock Online learning in markov decision processes with adversarially
  chosen transition probability distributions.
\newblock In {\em Advances in Neural Information Processing Systems}, 2013.

\bibitem{auer2007logarithmic}
Peter Auer and Ronald Ortner.
\newblock Logarithmic online regret bounds for undiscounted reinforcement
  learning.
\newblock In {\em Advances in Neural Information Processing Systems}, 2007.

\bibitem{bartlett2009regal}
Peter~L Bartlett and Ambuj Tewari.
\newblock Regal: A regularization based algorithm for reinforcement learning in
  weakly communicating mdps.
\newblock In {\em Proceedings of Conference on Uncertainty in Artificial
  Intelligence}. AUAI Press, 2009.

\bibitem{bowling2001rational}
Michael Bowling and Manuela Veloso.
\newblock Rational and convergent learning in stochastic games.
\newblock In {\em International Joint Conference on Artificial Intelligence},
  2001.

\bibitem{brafman2002r}
Ronen~I Brafman and Moshe Tennenholtz.
\newblock R-max-a general polynomial time algorithm for near-optimal
  reinforcement learning.
\newblock {\em Journal of Machine Learning Research}, 2002.

\bibitem{bubeck2012best}
S{\'e}bastien Bubeck and Aleksandrs Slivkins.
\newblock The best of both worlds: Stochastic and adversarial bandits.
\newblock In {\em Conference on Learning Theory}, 2012.

\bibitem{cho2000Markov}
Grace~E Cho and Carl~D Meyer.
\newblock Markov chain sensitivity measured by mean first passage times.
\newblock {\em Linear Algebra and Its Applications}, 2000.

\bibitem{conitzer2007awesome}
Vincent Conitzer and Tuomas Sandholm.
\newblock Awesome: A general multiagent learning algorithm that converges in
  self-play and learns a best response against stationary opponents.
\newblock {\em Machine Learning}, 2007.

\bibitem{dann2015sample}
Christoph Dann and Emma Brunskill.
\newblock Sample complexity of episodic fixed-horizon reinforcement learning.
\newblock In {\em Advances in Neural Information Processing Systems}, 2015.

\bibitem{dick2014online}
Travis Dick, Andras Gyorgy, and Csaba Szepesvari.
\newblock Online learning in markov decision processes with changing cost
  sequences.
\newblock In {\em Proceedings of International Conference of Machine Learning},
  2014.

\bibitem{even2009online}
Eyal Even-Dar, Sham~M Kakade, and Yishay Mansour.
\newblock Online markov decision processes.
\newblock {\em Mathematics of Operations Research}, 2009.

\bibitem{federgruen1978Nperson}
Awi Federgruen.
\newblock On n-person stochastic games by denumerable state space.
\newblock {\em Advances in Applied Probability}, 1978.

\bibitem{garivier2016maximin}
Aur{\'e}lien Garivier, Emilie Kaufmann, and Wouter~M Koolen.
\newblock Maximin action identification: A new bandit framework for games.
\newblock In {\em Conference on Learning Theory}, pages 1028--1050, 2016.

\bibitem{hordijk1974dynamic}
Arie Hordijk.
\newblock Dynamic programming and markov potential theory.
\newblock {\em MC Tracts}, 1974.

\bibitem{hunter1982generalized}
Jeffrey~J Hunter.
\newblock Generalized inverses and their application to applied probability
  problems.
\newblock {\em Linear Algebra and Its Applications}, 1982.

\bibitem{hunter2005stationary}
Jeffrey~J Hunter.
\newblock Stationary distributions and mean first passage times of perturbed
  markov chains.
\newblock {\em Linear Algebra and Its Applications}, 2005.

\bibitem{Iyengar:2005robustdynamic}
Garud~N. Iyengar.
\newblock Robust dynamic programming.
\newblock {\em Math. Oper. Res.}, 30(2):257--280, 2005.

\bibitem{jaakkola1994convergence}
Tommi Jaakkola, Michael~I Jordan, and Satinder~P Singh.
\newblock On the convergence of stochastic iterative dynamic programming
  algorithms.
\newblock {\em Neural computation}, 1994.

\bibitem{jaksch2010near}
Thomas Jaksch, Ronald Ortner, and Peter Auer.
\newblock Near-optimal regret bounds for reinforcement learning.
\newblock {\em Journal of Machine Learning Research}, 2010.

\bibitem{kakade2003sample}
Sham~Machandranath Kakade et~al.
\newblock {\em On the sample complexity of reinforcement learning}.
\newblock PhD thesis, University of London London, England, 2003.

\bibitem{lagoudakis2002value}
Michail~G Lagoudakis and Ronald Parr.
\newblock Value function approximation in zero-sum markov games.
\newblock In {\em Proceedings of Conference on Uncertainty in Artificial
  Intelligence}. Morgan Kaufmann Publishers Inc., 2002.

\bibitem{lattimore2012pac}
Tor Lattimore and Marcus Hutter.
\newblock Pac bounds for discounted mdps.
\newblock In {\em International Conference on Algorithmic Learning Theory}.
  Springer, 2012.

\bibitem{lim2016robustMDP}
Shiau~Hong Lim, Huan Xu, and Shie Mannor.
\newblock Reinforcement learning in robust markov decision processes.
\newblock {\em Math. Oper. Res.}, 41(4):1325--1353, 2016.

\bibitem{littman1994markov}
Michael~L Littman.
\newblock Markov games as a framework for multi-agent reinforcement learning.
\newblock In {\em Proceedings of International Conference of Machine Learning},
  1994.

\bibitem{maurer2009empirical}
A~Maurer and M~Pontil.
\newblock Empirical bernstein bounds and sample variance penalization.
\newblock In {\em Conference on Learning Theory}, 2009.

\bibitem{mertens1981stochastic}
J-F Mertens and Abraham Neyman.
\newblock Stochastic games.
\newblock {\em International Journal of Game Theory}, 1981.

\bibitem{neu2010online}
Gergely Neu, Andras Antos, Andr{\'a}s Gy{\"o}rgy, and Csaba Szepesv{\'a}ri.
\newblock Online markov decision processes under bandit feedback.
\newblock In {\em Advances in Neural Information Processing Systems}, 2010.

\bibitem{neu2012adversarial}
Gergely Neu, Andr{\'a}s Gy{\"o}rgy, and Csaba Szepesv{\'a}ri.
\newblock The adversarial stochastic shortest path problem with unknown
  transition probabilities.
\newblock In {\em AISTATS}, 2012.

\bibitem{nilim2005robustcontrol}
Arnab Nilim and Laurent~El Ghaoui.
\newblock Robust control of markov decision processes with uncertain transition
  matrices.
\newblock {\em Math. Oper. Res.}, 53(5):780--798, 2005.

\bibitem{perolat2015approximate}
Julien Perolat, Bruno Scherrer, Bilal Piot, and Olivier Pietquin.
\newblock Approximate dynamic programming for two-player zero-sum markov games.
\newblock In {\em Proceedings of International Conference of Machine Learning},
  2015.

\bibitem{prasad2015two}
HL~Prasad, Prashanth LA, and Shalabh Bhatnagar.
\newblock Two-timescale algorithms for learning nash equilibria in general-sum
  stochastic games.
\newblock In {\em Proceedings of the 2015 International Conference on
  Autonomous Agents and Multiagent Systems}. International Foundation for
  Autonomous Agents and Multiagent Systems, 2015.

\bibitem{shapley1953stochastic}
Lloyd~S Shapley.
\newblock Stochastic games.
\newblock {\em Proceedings of the National Academy of Sciences}, 1953.

\bibitem{szepesvari1996generalized}
Csaba Szepesv{\'a}ri and Michael~L Littman.
\newblock Generalized markov decision processes: Dynamic-programming and
  reinforcement-learning algorithms.
\newblock In {\em Proceedings of International Conference of Machine Learning},
  1996.

\bibitem{van1980successive}
J~Van~der Wal.
\newblock Successive approximations for average reward markov games.
\newblock {\em International Journal of Game Theory}, 1980.

\bibitem{yu2009arbitrarily}
Jia~Yuan Yu and Shie Mannor.
\newblock Arbitrarily modulated markov decision processes.
\newblock In {\em Proceedings of Conference on Decision and Control}. IEEE,
  2009.

\end{thebibliography}
\newpage

%##################################################################
%##################################################################

%     Appendix

%##################################################################
%##################################################################

\appendix
\appendixwithtoc
% ##########################################
% Previous bounds
% ##########################################
\section{Previous Bounds for MDPs and SGs}
\label{appendix:previous_bound}
The techniques we use in this paper are most related to the probably approximately correct (PAC) analysis for RL algorithms. Some rather complete reviews of the related works are provided in \cite{jaksch2010near,  dann2015sample}. \cite{jaksch2010near} considers the average-reward MDP that is \textit{communicating} with bounded diameter $D$ (i.e., $\max_{s, s^\prime}\min_{\pi} T^{\pi}_{s\rightarrow s^\prime}(M)\leq D$, where $ T^{\pi}_{s\rightarrow s^\prime}(M)$ is defined as the expected time to reach from state $s$ to state $s^\prime$ under model $M$ and policy $\pi$). Their \textsc{UCRL2} algorithm achieves $\tilde{\mathcal{O}}(DS\sqrt{AT})$ regret upper bound, while still having a gap with the $\Omega(\sqrt{DSAT})$ lower bound. These bounds translate to $\tilde{\mathcal{O}}\left(\frac{D^2S^2A}{\varepsilon^2}\right)$ and $\Omega\Big(\frac{DSA}{\varepsilon^2}\Big)$ sample complexity. The additional $D$ dependency is resolved by \cite{lattimore2012pac, dann2015sample}, though in discounted and episodic settings respectively. These two works leverage the Bellman equation for local variance and obtained sample complexity bounds of order $\tilde{\mathcal{O}}\left(\frac{S^2A}{\varepsilon^2(1-\gamma)^3}\right)$ and $\tilde{\mathcal{O}}\left(\frac{H^2S^2A}{\varepsilon^2}\right)$ ($\gamma$: discount factor, $H$: fixed horizon length), making their gaps with the lower bounds $\Omega\left(\frac{SA}{\varepsilon^2(1-\gamma)^3}\right)$ and $\Omega\left(\frac{H^2SA}{\varepsilon^2}\right)$ remain only an order of $S$.  

The scenario that most resembles ours in the literature is that considered in \cite{brafman2002r}, who proposed the algorithm \textsc{R-max}. \textsc{R-max} is an optimism-based algorithm that can be used to learn stochastic games with arbitrary opponents. However, the algorithm depends on a parameter $\varepsilon$ and the $\varepsilon$-return mixing time $T_{\varepsilon}$ that need to be known in advance. This $\varepsilon$-return mixing time resembles our $\frac{D}{\varepsilon}$ in Assumption \ref{assumption:ergodic}. As a result, their $\tilde{\mathcal{O}}\left(\frac{T_{\varepsilon}^3S^2A}{\varepsilon^3}\right)$ translates to $\tilde{\mathcal{O}}\left(\frac{D^3S^2A}{\varepsilon^6}\right)$, while our bound is $\tilde{\mathcal{O}}\left(\frac{DS^2A}{\varepsilon^3}\right)$. Another difference lies in that the output policy of our algorithm is a stationary one, rather than a $T_\varepsilon$-step non-stationary policy as in \textsc{R-max}.

%##########################################
% Inequalities
%##########################################
\section{Inequalities}
\begin{lemma} (Azuma-Hoeffding's inequality. Theorem 4.2 of \cite{bubeck2012best})
\label{lemma:Azuma-Hoeffding's inequality}
Let $\mathcal{F}_1 \subseteq \cdots \subseteq \mathcal{F}_T$ be a filtration, and $Y_1, \cdots, Y_T$ real random variables such that $Y_t$ is $\mathcal{F}_t$-measurable, $\mathbbm{E}(Y_t\vert \mathcal{F}_{t-1})=0$ and $Y_t\in \left[A_t, A_t+c_t\right]$ where $A_t$ is a random variable $\mathcal{F}_{t-1}$-measurable and $c_t$ is a positive constant.  Then with probability at least $1-\delta$, 
$$\sum_{t=1}^{T}Y_t < \sqrt{\frac{\log(\delta^{-1})}{2}\sum_{t=1}^{T}c_t^2}.$$
\end{lemma}
\begin{lemma}(Bernstein inequality. Lemma 4.4 of \cite{bubeck2012best})
\label{lemma:bernstein_bubeck}
Let $\mathcal{F}_1 \subseteq \cdots \subseteq \mathcal{F}_T$ be a filtration, and $Y_1, \cdots, Y_T$ real random variables such that $Y_t$ is $\mathcal{F}_t$-measurable, $\mathbbm{E}(Y_t\vert \mathcal{F}_{t-1})=0$ and $\lvert Y_t\rvert \leq b$ for some $b>0$. Let $V_T=\sum_{t=1}^{T}\mathbbm{E}(Y_t^2\vert \mathcal{F}_{t-1})$ and $\delta>0$. Then with probability at least $1-\delta$, 
\begin{align*}
\sum_{t=1}^{T}Y_t\leq 2\sqrt{V_T\log(T\delta^{-1})}+\sqrt{5}b\log(T\delta^{-1}).
\end{align*}
\end{lemma}

%#############################################################################
% Perturbation Theorems
%#############################################################################
\section{Perturbation Bounds for Markov Chains}
\label{subsection:perturbation}
Perturbation analysis for Markov chains plays an important role in analyzing reinforcement learning algorithms (e.g., \cite{auer2007logarithmic}). Those analyses mainly center around the question that when the transition probabilities of a Markov chain are perturbed by a little, how much stationary distributions or mean first passage times (as defined in Definition \ref{definition:first_passage_time}) will change. While in \cite{auer2007logarithmic}, the perturbation bound for stationary distributions is used, we further use that of the mean first passage time to get a tighter regret bound.  

In this section, we use $i, j$ to index states, and use $\mu_i$ to denote the stationary distribution of state $i$ in an irreducible Markov chain. 

\begin{definition}[Mean first passage time]
\label{definition:first_passage_time}
In a Markov chain, we define $T_{ij}$ to be the expected time to reach state $j$ starting from state $i$. In the case $i=j$, $T_{ii}$ is the expected time to return to state $i$ when starting from $i$. Thus $T_{ij}\geq 1$ always holds whether $i=j$ or not. 
\end{definition}
\subsection{Perturbation Bounds for Stationary Distribution}

\begin{theorem}[Proposition 2.2 of \cite{cho2000Markov}]
\label{theorem:perturbation_stationary_distribution}
Let $C$ and $\tilde{C}$ be two irreducible Markov chains with the same state space $\mathcal{S}$. Let their transition matrices be $P$, $\tilde{P}$, and stationary distributions be $\mu$, $\tilde{\mu}$. Let $E=\tilde{P}-P$ and use $\norm{\cdot}_\infty$ to represent the largest absolute value in a matrix, then $\forall j$,
\begin{align}
\lvert \tilde{\mu}_j -\mu_j\rvert \leq \mu_j\frac{S\norm{E}_\infty}{2}\max_{i\neq j}T_{ij}.  \label{equation:perturb_element1}
\end{align}
\end{theorem}

With a little modification on the proof of Theorem \ref{theorem:perturbation_stationary_distribution}, we can actually have the following lemma, which only requires that $C$ be an irreducible Markov chain.   

\begin{theorem}
\label{theorem:perturbation_stationary_distribution2}
Let $C$ be an irreducible Markov chain, and $\tilde{C}$ be some Markov chain with the same state space $\mathcal{S}$ as $C$. Let their transition matrices be $P$, $\tilde{P}$, and let $C$'s stationary distributions be $\mu$. Let $E=\tilde{P}-P$. If $\norm{E}_\infty < 2/(S\max_{i\neq j}T_{ij})$, then $\tilde{C}$ is also an irreducible Markov chain; furthermore, the stationary distribution of $\tilde{C}$, $\tilde{\mu}$, satisfies $\forall j$, 
\begin{align}
\lvert \tilde{\mu}_j -\mu_j\rvert \leq \mu_j\frac{S\norm{E}_\infty}{2}\max_{i\neq j}T_{ij}.  \label{equation:perturb_element}
\end{align}
\end{theorem}
\begin{proof}
Let $P^*$ and $\tilde{P}^*$ be the Cesaro limits of $P$ and $\tilde{P}$, which is defined by $P^*=\lim_{T\rightarrow \infty}\frac{1}{T}\sum_{t=1}^{T}P^{t-1}$. Then we have
\begin{align*}
P^*(I-P)=0, \ \ \tilde{P}^*(I-\tilde{P})=\tilde{P}^*(I-P-E)=0, 
\end{align*}
and thus $(\tilde{P}^*-P^*)(I-P)=\tilde{P}^*E$. If $\tilde{P}$ induces an irreducible Markov chain, $\tilde{P}^*$ will have all identical rows and all positive elements. Suppose not, we can still extract its $k$-th row, which corresponds to the stationary distribution when starting from state $k$. Let this $k$-th row's $j$-th element be $\tilde{\mu}^k_j$. We can write $(\tilde{\mu}_j^k-\mu_j)(I-P)=\tilde{\mu}_j^kE.$ Then following the same proof as in \cite{cho2000Markov} or by \cite{hunter2005stationary}'s Theorem 2.1, we still have 
\begin{align*}
\vert \tilde{\mu}_j^k-\mu_j \vert \leq \mu_j\frac{S\norm{E}_\infty}{2}\max_{i\neq j}T_{ij} 
\end{align*}
Now since $\norm{E}_\infty \leq 2/(S\max_{i\neq j}T_{ij})$ and $\mu_j>0 \ \forall j$, we have $\tilde{\mu}_j^k>0 \ \forall j,k$. This means that every state is recurrent and reachable from each other, implying that $\tilde{P}$ induces an irreducible Markov chain.  
\end{proof}

\subsection{Perturbation Bounds for Mean First Passage Time}
The main result of this subsection is stated in Theorem \ref{theorem:perturbation_first_passage_time}. It is developed with the help of Theorem \ref{theorem:first_passage} to Theorem \ref{theorem:bound_G}. 
%#########################################################################
%    perturbation bounds for Markov chains
%#########################################################################
%\section{Perturbation Bounds for Markov Chains}
\begin{definition}[g-inverse, Definition 3.1 of \cite{hunter1982generalized}]
A g-inverse of a matrix $A$ is any matrix $G$ such that $AGA=A$. 
\end{definition}

\begin{theorem}[Theorem 5.3 of \cite{hunter2005stationary}]
\label{theorem:first_passage}
Let $C$ be an irreducible Markov chain with stochastic matrix $P$. Let $T_{ij}$ be the first passage time from state $i$ to state $j$, and let $G$ be any g-inverse of $I-P$. We have
\begin{align*}
\mu_jT_{ij}=G_{jj}-G_{ij}+\delta_{ij}+\mu_j\sum_{k=1}^{n}(G_{ik}-G_{jk}).
\end{align*}

\end{theorem}
\label{theorem:G}
The below theorem introduces a special g-inverse that is convenient for our use. 
\begin{theorem}[Theorem 3.3 of \cite{hunter1982generalized}]
\label{theorem:special_inverse}
Let $P$ be a stochastic matrix of an irreducible Markov chain. Let $p_n^\top$ denote the $n$-th row of $P$, and $e_n$ denote the unit column vector with $n$-th component being 1. Then $I-P+e_np_n^\top$ is non-singular, and $G=(I-P+e_np_n^\top)^{-1}$ is a g-inverse of $I-P$. 
\end{theorem}
\begin{theorem}[Section 5 of  \cite{hunter2005stationary}]
\label{theorem:single_row}
Let $\tilde{P}$ be a stochastic matrix of an irreducible Markov chain perturbed from another stochastic matrix $P$ of an irreducible Markov chain. Suppose that the perturbation only occurs at the $n$-th row of $P$ (i.e. $p_i^\top=\tilde{p}^\top_i \forall i\neq n$). Define $G$ as that in Theorem \ref{theorem:special_inverse}. Then $G=\tilde{G}$. 
\end{theorem}
Suppose that the perturbation only occurs at the $n$-th row, and let $G=(I-P+e_np_n^\top)^{-1}$. Then Theorem \ref{theorem:first_passage} and \ref{theorem:single_row} together imply that for $i\neq j$, 
\begin{align}
\tilde{T}_{ij} = T_{ij}+(G_{ij}-G_{jj})\left(\frac{1}{\mu_j}-\frac{1}{\tilde{\mu}_j}\right),  \label{equation:first_passage_update}
\end{align}
with $T_{jj}=1/\mu_j$ and $\tilde{T}_{jj}=1/\tilde{\mu}_j$ (Corollary 5.3.1 of  \cite{hunter2005stationary}). Here we see that $\tilde{T}_{in}=T_{in}, \forall i$.

\begin{lemma}
\label{theorem:bound_G}
Let $P$ be the stochastic matrix of an irreducible Markov chain, and let $G=(I-P+e_np_n^\top)^{-1}$. If all mean first passage times are bounded by $D^\prime$ (i.e., $T_{ij}\leq D^\prime\ \  \forall i,j$), then $\lvert G_{ij}-G_{jj}\rvert \leq 2\mu_j D^{\prime} \ \ \forall i,j$. 
\end{lemma}
\begin{proof}
We first verify that 
\begin{align}
G= \label{eqn:Gs_form}
\begin{bmatrix}
    (I-P_n)^{-1}  & \textbf{e} \\
    0 & 1 \\
\end{bmatrix},
\end{align}
where $P_n$ is obtained by deleting the $n$-th row and $n$-th column of $P$ (without loss of generality, assume that the $n$-th row is last row of $G$). 

Directly expanding $I-P+e_np_n^\top$, we get
$$ I-P+e_np_n^\top=
\begin{bmatrix}
    (I-P_n)  & d \\
    0 & 1 \\
\end{bmatrix},
$$
where $d=(-p_{1,n}, -p_{2,n}, ..., -p_{n-1,n})^\top$. To verify that $(I-P+e_np_n^\top)^{-1}$ takes the form of \eqref{eqn:Gs_form}, one only needs to verify that $(I-P_n)\textbf{e}+d=0$. This can be seen by $(I-P_n)\textbf{e}=\textbf{e}-P_n\textbf{e}=(1,...,1)-(\sum_{i=1}^{n-1}p_{1,i},...,\sum_{i=1}^{n-1}p_{n-1,i})=-d$. 

For $i\neq n$, from $G$'s expression in \eqref{eqn:Gs_form}, we have  
\begin{align}
\sum_{k=1}^{n}G_{ik}= e_i^\top G \mathbf{e}=e_i^\top (I-P_n)^{-1}\mathbf{e} + 1. \label{eqn:G_expansion}
\end{align}
Note that the dimension of $e_i$ and $\mathbf{e}$ are $n$ in the second expression of  \eqref{eqn:G_expansion}, while are $n-1$ in the third expression. By \cite{cho2000Markov}'s Equation (2.3), $e_i^\top (I-P_n)^{-1}\mathbf{e}=T_{in}$. One can also see this by observing that $(I-P_n)^{-1}=I+P_n+P_n^2+\cdots$, and $e_i^\top P_n^m e_j$ is ``the probability of staying at $j$ after $m$ steps from $i$, while not visiting $n$ in any of the $m$ steps''. Summing $e_i^\top P_n^m e_j$ over $j$ and $m$, the physical meaning becomes the mean first passage time from $i$ to $n$, and the mathematical expression becomes $e_i^\top (I-P_n)^{-1}\mathbf{e}$. Thus, $\lvert \sum_{k=1}^n (G_{ik}-G_{jk})\rvert=\vert T_{in}-T_{jn} \vert \leq \max_{ij} T_{ij}\leq D^\prime$. By Theorem \ref{theorem:first_passage}, whenever $i\neq j$, 
\begin{align*}
\lvert G_{ij}-G_{jj} \rvert &= \Big\vert \mu_jT_{ij}-\mu_j \sum_{k=1}^n (G_{ik}-G_{jk})\Big\vert  \leq \mu_jT_{ij}+\mu_jD^\prime \leq 2\mu_j D^\prime.
\end{align*}
\end{proof}

We now combine \eqref{equation:first_passage_update} with \eqref{equation:perturb_element} and Lemma \ref{theorem:bound_G}. Assuming that $T_{ij}\leq D^\prime$, we have for $i\neq j$,
\begin{align}
\lvert \tilde{T}_{ij}-T_{ij} \lvert = \lvert G_{ij}-G_{jj}\rvert \frac{\lvert \tilde{\mu}_j-\mu_j\rvert}{\mu_j\tilde{\mu}_j} \leq \frac{\mu_j}{\tilde{\mu}_j}\norm{E}_\infty SD^{\prime 2}. \label{eqn:perturbation_for_first_passage}
\end{align}

With  \eqref{equation:perturb_element} and \eqref{eqn:perturbation_for_first_passage} available, we now consider a general perturbation, which can actually be decomposed as $S$ single-row perturbations.  

\begin{theorem}
\label{theorem:perturbation_first_passage_time}
Let $P$, $\tilde{P}$ be the original and the perturbed stochastic matrices, and let $\{T_{ij}\}$, $\{\tilde{T}_{ij}\}$ be their corresponding mean first passage times. If $\max_{ij}T_{ij}\leq D$ and $\norm{E}_\infty = \norm{\tilde{P}-P}_\infty \leq \frac{1}{8DS^2}$, then $\max_{ij}\tilde{T}_{ij}\leq 2D.$
\end{theorem}
\begin{proof}

We do this general perturbation of $P$ by perturbing one row at a time. This procedure will repeat for $S$ times. 

Suppose that the original stationary distribution and first passage times are denoted by $\mu^{(0)}_i$ and $T^{(0)}_{ij}$, and that those after $n$-th perturbation are denoted by $\tilde{\mu}^{(n)}_i$ and $\tilde{T}^{(n)}_{ij}$. 

Suppose that $T_{ij}^{(0)}\leq D\ \forall i,j$ and $\mu_j^{(0)} \geq \frac{1}{D} \ \forall j$. Set $\norm{E}_\infty\leq \frac{1}{8S^2D}$. 
We prove the following facts by induction: 
\begin{align}
\tilde{T}^{(n)}_{ij}&\leq D\left(1+\frac{n}{S}\right), \label{eqn:induction_1} 
\end{align}
\begin{align}
\tilde{\mu}_j^{(n)} &\in \left[ \tilde{\mu}_j^{(0)}\left(1-\frac{1}{8S}\right)^n,  \tilde{\mu}_j^{(0)}\left(1+\frac{1}{8S}\right)^n\right],  \label{eqn:induction_2}
\end{align}
for $n=1, ..., S.$  Since $n\leq S$, these induction hypotheses implicitly imply that 
\begin{align}
\tilde{T}_{ij}^{(n)}&\leq 2D,  \label{eqn:induction_3} 
\end{align}
\begin{align}
\tilde{\mu}_j^{(n)} & \geq \tilde{\mu}_j^{(0)}\left(1-\frac{1}{8S}\right)^S \geq \frac{1}{2D},   \label{eqn:induction_4} 
\end{align}

%$\tilde{T}^{(n)}_{ij}\leq D\left(1+\frac{n}{S}\right)$ and $\tilde{\mu}_j^{(n)} \in \left[ \tilde{\mu}_j^{(0)}\left(1-\frac{1}{8S}\right)^n,  \tilde{\mu}_j^{(0)}\left(1+\frac{1}{8S}\right)^n\right]$, for $n=1, ..., S.$  Since $n\leq S$, these induction hypotheses implicitly imply that $\tilde{T}_{ij}^{(n)}\leq 2D$ and $\tilde{\mu}_j^{(n)} \in \left[ \tilde{\mu}_j^{(0)}\left(1-\frac{1}{8S}\right)^S, \tilde{\mu}_j^{(0)}\left(1+\frac{1}{8S}\right)^S \right] \subset \left[\frac{1}{2}\tilde{\mu}_j^{(0)}, 2\mu_j^{(0)}\right]$.

because $(1-1/(8S))^S \geq 1/2$ for all $S\geq 1$. Now we start the induction. The base case for $n=0$ clearly holds. Suppose that \eqref{eqn:induction_1}-\eqref{eqn:induction_2} hold for all $n\leq k$. Then by \eqref{equation:perturb_element} we have $\Big\vert \tilde{\mu}_j^{(k)}-\tilde{\mu}_j^{(k+1)}\Big\vert  \leq \tilde{\mu}_j^{(k)}\frac{1}{8S} $, so $\tilde{\mu}_j^{(k+1)} \geq \tilde{\mu}_j^{(k)}\left(1-\frac{1}{8S}\right) \geq \tilde{\mu}_j^{(0)}\left(1-\frac{1}{8S}\right)^{(k+1)}$, and $\tilde{\mu}_j^{(k+1)} \leq \tilde{\mu}_j^{(k)}\left(1+\frac{1}{8S}\right) \leq \tilde{\mu}_j^{(0)}\left(1+\frac{1}{8S}\right)^{(k+1)}$. On the other hand, by \eqref{eqn:perturbation_for_first_passage}, for $i\neq j$, we have $\tilde{T}_{ij}^{(k+1)}-\tilde{T}_{ij}^{(k)}\leq \frac{\tilde{\mu}_j^{(k)}}{\tilde{\mu}_j^{(k+1)}}\norm{E}_\infty (2D)^2\leq \frac{1}{1-\frac{1}{8S}}\frac{1}{8S^2D}S(2D)^2 \leq \frac{D}{S}$, so $\tilde{T}_{ij}^{(k+1)}\leq\tilde{T}_{ij}^{(k)}+\frac{D}{S}\leq D\left(1+\frac{k+1}{S}\right)$. In the case $i=j$, we have $\tilde{T}_{jj}^{(k+1)}-\tilde{T}_{jj}^{(k)}=\frac{1}{\tilde{\mu}^{(k+1)}_j}-\frac{1}{\tilde{\mu}^{(k)}_j}\leq\frac{\left(1+\frac{1}{8S}\right)-1}{\tilde{\mu}_j^{(k)}}\leq\frac{2D}{8S}\leq\frac{D}{S}$.
\end{proof}

%#############################################################
%   failing events
%#############################################################

\section{Lemmas for Failing Events}
\label{section:Lemmas for Failing Events}
\begin{lemma}[Proposition 18 of \cite{jaksch2010near}]
\label{lemma:number_phases}
The number of phases is upper bounded by $U_{\max}=SA\log_2{T}$. 
\end{lemma}
\begin{proof}
Since phase changes only occur when the sample count of some $(s, a^1, a^2)$ is doubled, those changes corresponding to a specific $(s, a^1, a^2)$ is upper bounded by $\log_2T$. Considering all states and actions, the total number of phase changes is upper bounded by $SA\log_2 T$.
\end{proof}
\begin{lemma}[Lemma 17 of \cite{jaksch2010near}]
\label{lemma:weissman_bound}
For some specific $k, s$ and $a$, the event $p(\cdot\vert s, a) \in \textsc{conf}_1(\hat{p}_k(\cdot \vert s,a), n_k(s,a))$ holds with probability at least $1-\delta_1$. 
\end{lemma}
\begin{proof}
Please refer to \cite{jaksch2010near}. 
\end{proof}

\begin{lemma}[Lemma 1 of \cite{dann2015sample}, Theorem 10 and 11 of \cite{maurer2009empirical}]
\label{lemma:empirical_bernstein}
For some specific $k, s$ and $a$, the event $p(\cdot\vert s, a) \in \textsc{conf}_2(\hat{p}_k(\cdot \vert s,a), n_k(s,a))$ holds with probability at least $1-S\delta_1$. 
\end{lemma}
\begin{proof}
Please refer to \cite{dann2015sample} or \cite{maurer2009empirical}. 
\end{proof}

\begin{proof}[Proof of Lemma \ref{lemma:bound_fail}]
By Lemma \ref{lemma:number_phases}, there are at most $SA\log_2 T$ confidence set updates to consider. Each update involves only a specific $\hat{p}(\cdot\vert s,a)$ (totally $S$ entries). By Lemma \ref{lemma:weissman_bound}, \ref{lemma:empirical_bernstein} and using the union bound, the event $M\in\mathcal{M}_k \forall k$ holds with probability at least $1-SA\log_2T\times(\delta_1+S\delta_1)\geq 1-\delta$.
\end{proof}

\section{Lemmas for Stationary Optimal Policies}\label{stationary_optimal_policy}
\begin{theorem}
\label{theorem:stationary_optimal}
Given a stochastic game $M=(\mathcal{S},\mathcal{A},r,p)$, where $\mathcal{S}$ is countable, $\mathcal{A}=\mathcal{A}^1\times\mathcal{A}^2$ a compact metric space and both $r(s,\cdot)\in[0,1]$ and $p(s'|s,\cdot)$ are continuous in $a=(a^1,a^2)$. Suppose Assumption \ref{assumption:ergodic} holds for $M$. Then there exist maximin stationary policies $\pi^{*}=(\pi^{1*},\pi^{2*})$ for the two-player zero-sum stochastic game, the maximin stationary policies attain the game value $\rho^*$, which is independent of the initial state, and there is a bounded function $h(\cdot)$ which together with $\rho^*$ satisfies the following Bellman equation. That is, for all state $s$,
\begin{align*}
&\scalebox{1.0}{$ \displaystyle \rho^* + h(s) =\max_{\pi^1\in \Pi^{\text{SR}}} \Big\{ r(s,\pi^1,\pi^{2*}) + \sum_{s'} p(s'|s,\pi^1,\pi^{2*})h(s') \Big\} $} \\
&\scalebox{1.0}{$ \displaystyle \rho^* + h(s) =\min_{\pi^2 \in \Pi^{\text{SR}}} \Big\{ r(s,\pi^{1*},\pi^2) + \sum_{s'} p(s'|s,\pi^{1*},\pi^2)h(s') \Big\}. $}
\end{align*}
\end{theorem}
To prove this, we use the following lemma which connects the boundedness of mean first passage times with the uniform boundedness of $\spa(V_\alpha^*(\cdot))$ for all discount factor $0<\alpha<1$, where $V_\alpha^*(\cdot)$ is the discounted game value defined as $V_\alpha^*(s)=\max_{\pi^1}\min_{\pi^2} \mathbbm{E}^{\pi^1, \pi^2}\left[\sum_{t=1}^{\infty}\alpha^{t-1}r_t \vert s_1=s \right]$. It is known that for any discount factor $0<\alpha<1$, discounted SGs always have maximin stationary policies $\pi_\alpha=(\pi^1_\alpha,\pi^2_\alpha)$ which attain the game value $V_\alpha^*(s)$ for all $s$. We next show that the span of $V_\alpha^*$ is uniformly bounded by $D$ under Assumption \ref{assumption:ergodic}.
\begin{lemma} \cite{hordijk1974dynamic}
Suppose given a stochastic game $M=(\mathcal{S},\mathcal{A},r,p)$, where $0\leq r(s,a^1,a^2)\leq 1$. Suppose $\forall s,s'\in \mathcal{S}$ and for any $\pi^2\in \Pi^{\text{SR}}$ for Player 2, there exists a $\pi^1\in \Pi^{\text{SR}}$ for Player 1 such that the mean first passage time $T^{\pi^1,\pi^2}_{s,s'}\leq D$. Then we have $|V_{\alpha}^*(s)-V_\alpha^*(s')|\leq D, \ \forall s,s'\in \mathcal{S}$, for all $0<\alpha<1$.
\end{lemma}
\begin{proof}
Fix $s,s'\in \mathcal{S}$. Fix a discount factor $0<\alpha<1$. For a fixed pair of maximin stationary policies $\pi_\alpha=(\pi^1_\alpha,\pi^2_\alpha)\in \Pi^{\text{SR}}\times\Pi^{\text{SR}}$, the discounted value function satisfies $V_\alpha^*(s)=r(s,\pi^1_\alpha,\pi^2_\alpha)+\alpha\sum_{s'}p(s'|s,\pi^1_\alpha,\pi^2_\alpha)V^*_\alpha(s')$. Since for any $\pi^1\in \Pi^{\text{SR}}$, $V_\alpha^*(s)\geq r(s,\pi^1,\pi^2_\alpha)+\alpha\sum_{s'} p(s'|s,\pi^1,\pi^2_\alpha)V_\alpha^*(s')$, thus recursively, for any time step $T\geq 1$, we have 
$$V_\alpha^*(s)\geq \mathbb{E}_s^{\pi^1,\pi^2_\alpha}\Big[\sum_{t=1}^{T-1}\alpha^{t-1}r_t+\alpha^{T-1}V_\alpha^*(s_T)\Big], $$
where $\mathbb{E}_s^{\pi^1,\pi^2}\left[\cdot\right]=\mathbb{E}_s\left[\cdot\right\vert \pi^1,\pi^2]$ denote the expectation conditioned on initial state being $s$, and the players executing the policy pair $(\pi^1, \pi^2)$. 
 Hence for any stopping time $\tau$, $$V_\alpha^*(s)\geq \mathbb{E}_s^{\pi^1,\pi^2_\alpha}\Big[\sum_{t=1}^{\tau-1}\alpha^{t-1}r_t+\alpha^{\tau-1} V_\alpha^*(s_\tau)\Big].$$ In particular, by choosing $\tau$ as the hitting time of $s'$ from $s$,
\begin{align*}
V_\alpha^*(s) & \geq \mathbb{E}^{\pi^1,\pi^2_\alpha}_s\Big[\sum_{t=1}^{\tau-1}\alpha^{t-1}r_t\Big]+\mathbb{E}_s\Big[\alpha^{\tau}\Big\vert \pi^1,\pi^2_\alpha\Big]V_\alpha^*(s')\\
&\geq \alpha^{\mathbb{E}_s\big[\tau\big\vert \pi^1,\pi^2_\alpha\big]}V_\alpha^*(s')=\alpha^{T^{\pi^1,\pi^2_\alpha}_{s\rightarrow s'}}V_\alpha^*(s')\\
&\geq V_\alpha^*(s')-(1-\alpha)(T^{\pi^1,\pi^2_\alpha}_{s\rightarrow s'})V_\alpha^*(s')\\
&\geq V_\alpha^*(s')-T^{\pi^1,\pi^2_\alpha}_{s\rightarrow s'}\\
&\geq V_\alpha^*(s')-D.
\end{align*}
For the first inequality we used $V_\alpha(s_\tau)=V_\alpha(s')$ and for the second, the non-negativity of $r(s,a)$ and Jensen's inequality. The equality holds since the expected value of hitting time is the mean first passage time $T^{\pi^1,\pi^2_\alpha}_{s\rightarrow s'}$. The third inequality is essentially $\alpha^x \geq (\alpha-1)x+1$ for $x\geq 1$; the fourth $(1-\alpha)V_\alpha\leq 1$. For the last inequality we used the assumption that there exists some $\pi^1$ for which $T^{\pi^1,\pi^2_\alpha}_{s\rightarrow s'}\leq D$. 
\end{proof}
\begin{lemma}\cite{federgruen1978Nperson}
Suppose $|V^*_\alpha(s)-V^*_\alpha(s')|$ is uniformly bounded for all $0<\alpha<1$ and for any $s,s'\in \mathcal{S}$. Then there exist a pair of maximin stationary policies $\pi=(\pi^{1*},\pi^{2*})$ attaining the game value $\rho^*$ which is independent of the initial state and a bounded function $h(\cdot)$ for which the following equations hold.
For all state $s$,
\begin{align*}
&\scalebox{1.0}{$ \displaystyle \rho^* + h(s) =\max_{\pi^1\in \Pi^{\text{SR}}} \Big\{ r(s,\pi^1,\pi^{2*}) + \sum_{s'} p(s'|s,\pi^1,\pi^{2*})h(s') \Big\}, $} \\
&\scalebox{1.0}{$ \displaystyle \rho^* + h(s) =\min_{\pi^2 \in \Pi^{\text{SR}}} \Big\{ r(s,\pi^{1*},\pi^2) + \sum_{s'} p(s'|s,\pi^{1*},\pi^2)h(s') \Big\}. $}
\end{align*}
\end{lemma}
\begin{proof}
For any discount factor $0<\alpha<1$,
\begin{align*}
&V^*_\alpha(s)=\max_{\pi^1}\{r(s,\pi^1,\pi^2_\alpha)+\alpha\sum_{s'}p(s'|s,\pi^1,\pi^2_\alpha)V^*_\alpha(s')\},\\
&V^*_\alpha(s)=\min_{\pi^2}\{r(s,\pi^1_\alpha,\pi^2)+\alpha\sum_{s'}p(s'|s,\pi^1_\alpha,\pi^2)V^*_\alpha(s')\}.
\end{align*}
Subtracting both sides by $V_\alpha^*(s_1)$ for some fixed state $s_1$, and defining $v_\alpha(s)\coloneqq V^*_\alpha(s)-V_\alpha^*(s_1)$, we get, for all $s$,
\begin{align*}
&\scalebox{1.0}{$\displaystyle v_\alpha(s)=\max_{\pi^1}\{r(s,\pi^1,\pi^2_\alpha)-(1-\alpha)V_\alpha^*(s_1)+\alpha\sum_{s'}p(s'|s,\pi^1,\pi^2_\alpha)v_\alpha(s')\},$}\\
&\scalebox{1.0}{$\displaystyle v_\alpha(s)=\min_{\pi^2}\{r(s,\pi^1_\alpha,\pi^2)-(1-\alpha)V_\alpha^*(s_1)+\alpha\sum_{s'}p(s'|s,\pi^1_\alpha,\pi^2)v_\alpha(s')\}.$}
\end{align*}
Since $-D\leq v_\alpha(s)\leq D$, $0\leq (1-\alpha)V^*_\alpha(s_1)\leq 1$ and $\pi_\alpha^i\in\Pi^{\text{SR}},\ (i=1,2)$, all of which are contained in compact subsets/spaces, by using diagonalization argument and by Lebesgue convergence theorem, we can obtain a sequence $\alpha_k\rightarrow 1$, a bounded function $h$, and a constant $\rho^*$ such that $v_{\alpha_k}(\cdot)\rightarrow h(\cdot)$, $(1-\alpha_k)V^*_{\alpha_k}(s_1)\rightarrow \rho^*$, $\pi_{\alpha_k}^i\rightarrow \pi^{i*},\ (i=1,2)$, and
\begin{align*}
&\scalebox{1.0}{$ \displaystyle \alpha_k\sum_{s'}p(s'|s,\pi^1,\pi^2_{\alpha_k})v_{\alpha_k}(s')\rightarrow \sum_{s'}p(s'|s,\pi^1,\pi^{2*})h(s'),$}\\
&\scalebox{1.0}{$ \displaystyle \alpha_k\sum_{s'}p(s'|s,\pi^1_{\alpha_k},\pi^2)v_{\alpha_k}(s')\rightarrow \sum_{s'}p(s'|s,\pi^{1*},\pi^2)h(s'),$}
\end{align*}
as $k \rightarrow \infty$.
Hence in the limit, for all state $s$,
\begin{align*}
&\scalebox{1.0}{$ \displaystyle \rho^* + h(s) =\max_{\pi^1\in \Pi^{\text{SR}}} \Big\{ r(s,\pi^1,\pi^{2*}) + \sum_{s'} p(s'|s,\pi^1,\pi^{2*})h(s') \Big\}, $} \\
&\scalebox{1.0}{$ \displaystyle \rho^* + h(s) =\min_{\pi^2 \in \Pi^{\text{SR}}} \Big\{ r(s,\pi^{1*},\pi^2) + \sum_{s'} p(s'|s,\pi^{1*},\pi^2)h(s') \Big\}. $}
\end{align*}
\end{proof}

%=========================================================
%  Proof for Value Iteration
%=========================================================
\section{\textsc{Maximin-EVI} and Its Convergence}
\label{appendix:maximinEVI}
\begin{algorithm}[tb]
   \caption{Value Iteration with Schweitzer transform}
   \label{alg:VI}
\begin{algorithmic}
   %\bleftindent
   \STATE {\bfseries Input:} $M=(\mathcal{S}, \mathcal{A}^1\times \mathcal{A}^2, r, p), 0<\gamma<1, 0<\alpha<1$.
   \STATE {\bfseries Initialization:} $v_0\equiv 0$.
   \STATE {\bfseries repeat for } $i=1, 2, ...$
   %\brightindent
   \STATE \scalebox{1.0}{$\displaystyle v_{i}=(1-\alpha)\val\big\{r+Pv_{i-1}\big\}+\alpha v_{i-1}$}.
  % \STATE $\left(\text{or\ }v_i=\val\{r+Pv_{i-1}\}\right)$
   %\erightindent
   \STATE {\bfseries until} $\ \spa\left(v_{i}-v_{i-1}\right)\leq (1-\alpha)\gamma$.
   %\eleftindent
\end{algorithmic}
\end{algorithm}

As noted in Section \ref{subsection:extended SG}, \textsc{Maximin-EVI} proceeds simply by applying value iteration (Algorithm \ref{alg:VI}) on $M^+$. The output of the algorithm is a value vector with tolerable errors. The $\val\{r+Pv_{i-1}\}$ term in Algorithm \ref{alg:VI} becomes
\begin{align}
&\val\Big\{r(s,a^{1+},a^2)+\sum_{s^\prime}p^{+}(s^\prime\vert s, a^{1+}, a^2)v_{i-1}(s^\prime))\Big\}   \nonumber \\  % \Delta_{\mathcal{A}+} \nonumber \\
=&\val\Big\{r(s,a^1,a^2)+\max_{\tilde{p}(\cdot)\in \mathcal{P}_k(s,a^1,a^2)}\sum_{s^\prime}\tilde{p}(s^\prime)v_{i-1}(s^\prime))\Big\}. \label{eqn:linear_programming}
\end{align}
The inner maximization can be efficiently solved with linear programming. The $\textsc{Maximin-EVI}(\mathcal{M}_k, \gamma_k)$ in 
\textsc{UCSG} is then done by running Algorithm \ref{alg:VI} with the evaluation of \eqref{eqn:linear_programming} in every iteration. 

The following three lemmas characterize the convergence of the algorithm, and the properties of its outputs when converged. Lemma 
\ref{lemma:EVI_conv} gaurantees that \textsc{Maximin-EVI} converges. Lemma \ref{lemma:EVI} shows that when the algorithm halts, the output policy's worst-case average reward does not deviate from the maximin reward by more than $\gamma$. Lemma \ref{lemma:bounded_span} shows that the output value vector has a span no more than $D$. 

\begin{lemma}[Theorem 4 in \cite{van1980successive}] \label{lemma:EVI_conv} Suppose that
Assumption \ref{assumption:ergodic} holds for some SG $M$. Then performing Value Iteration with Schweitzer transform on $M$ converges asymptotically. %Furthermore under Assumption \ref{assumption:irreducible} and that $p(s|s,a)\geq \alpha,\ \forall (s,a)$, value iteration converges geometrically.
\end{lemma}
\begin{proof} [Proof of Lemma \ref{lemma:EVI_conv}]
If Assumption \ref{assumption:ergodic} holds, then the Bellman equation holds with an initial-state independent game value by Theorem \ref{theorem:stationary_optimal}. Then by Theorem 4 of \cite{van1980successive}, the value iteration with Schweitzer transform converges. 
\end{proof}

\begin{lemma}
\label{lemma:EVI}
Suppose that Assumption \ref{assumption:ergodic} holds for some stochastic game $M$. Let $\{v_i\}$ be the value sequence in the Value Iteration algorithm. Let $N$ be the index when iteration halts, i.e., $\text{sp}(v_{N+1}-v_N)\leq (1-\alpha)\gamma$. Let $\pi^1\coloneqq $ \scalebox{0.9}{$\displaystyle\solve_1\left\{r+Pv_N\right\}$}. Then $\pi^1$ is $\gamma$-optimal in the sense that $ \min_{\pi^2}\rho(M,\pi^1. \pi^2) \geq \rho^*(M)-\gamma$. 
\end{lemma}

\begin{proof} [Proof of Lemma \ref{lemma:EVI}]
Let $D=\min_s\{v_{N+1}(s)-v_{N}(s)\}$ and $U=\max_s\{v_{N+1}(s)-v_N(s)\}$. Then
\begin{align*}
&D\mathbf{e}+v_N\leq v_{N+1}=(1-\alpha)\val\{r+Pv_N\}\nonumber + \alpha v_N \leq (1-\alpha)(r_{\pi}+P_{\pi}v_N) + \alpha v_N,
\end{align*}
where $\pi=(\pi^1,\pi^2)$ for any $\pi^2\in \Pi^{\text{SR}}$. 
Let $P_{\pi}^*=\lim_{T\rightarrow\infty}\frac{1}{T}\sum_{t=1}^{T} P_{\pi}^{t-1}$ be the Cesaro limit of $P_{\pi}$. Applying it on both sides of the inequality, we get $D\mathbf{e}\leq (1-\alpha) P_{\pi}^*r_{\pi}= (1-\alpha) \mathbf{\rho}(M, \pi^1,\pi^2, \cdot)$, or $D\leq (1-\alpha)\rho(M,\pi^1, \pi^2, s), \ \forall s, \pi^2$. Let $\pi^*=(\pi^{1*}, \pi^{2*})$ be the optimal policy pair and $\rho^*(M)$ be their maximin value, then $D\leq (1-\alpha)\rho(M, \pi^1, \pi^{2*},s)\leq (1-\alpha)\rho^*(M)$. In a similar way, one can prove that $U \geq (1-\alpha)\rho^*(M)$. Since we assume $U-D\leq (1-\alpha) \gamma$, we have $D\geq (1-\alpha)(\rho^*(M)-\gamma)$. Therefore, $\pi^1$ is $\gamma$-optimal in the sense that $\forall \pi^2$, $\rho(M,\pi^1, \pi^2,s) \geq \rho^*(M)-\gamma$.
\end{proof}
\begin{lemma}
\label{lemma:bounded_span}
If Assumption \ref{assumption:ergodic} holds for some model $M$, then value iteration procedure in Algorithm \ref{alg:VI} will always produce value functions with spans bounded by $D$. That is, 
\begin{align*}
\spa\left(v_i\right) \leq D, \ \ \forall i. 
\end{align*}
\end{lemma}
\begin{proof}
Note that value iteration with Schweitzer transform is equivalent to the following procedure. First modify the transition kernel and reward by $p_\alpha(s^\prime\vert s, a^1,a^2)=(1-\alpha)p(s^\prime\vert s, a^1,a^2)+\alpha \delta_{s,s^\prime}$ and $r_\alpha(s,a^1,a^2)=(1-\alpha)r(s,a^1,a^2)+\alpha 0$; then do the normal value iteration by $v_{i}=\val\{r_\alpha + P_\alpha v_{i-1}\}$. By the principle of dynamic programming, $v_i$ is the maximin expected reward in the $i$-step game under the transformed model. 

The transformed model is equivalent to the system where at each time step, the state remains same as the previous one with probability $\alpha$, and within that step there is no reward obtained/paid. 

%There is another way to interpret this transformed game. Imagine that the game performs state transition at a frequency of 1. We are, however, agents that observe the current state and designate the next action with random interval lengths $\sim \exp(1/\alpha)$. Therefore with probability $\alpha$, the observed current state is the same with the previous one, and within the last interval there is no reward obtained/paid. 

Clearly, in this new game, the advantage of starting from state $s$ than starting from state $s^\prime$ (which can be calculated by $v_i(s)-v_i(s^\prime)$) is no more than that in the original game. In the original game, by a similar argument as Remark 8 in \cite{jaksch2010near}, this advantage difference is bounded by $D$. This then implies the argument in the lemma. 
\end{proof}

%================================================================
%    UCSG Technical Proofs
%================================================================
\section{Proof of Lemma \ref{lemma:epsilon_accurate} }
Lemma \ref{lemma:epsilon_accurate} directly follows from Lemma \ref{lemma:sample_complexity1} and \ref{lemma:sample_complexity2}.

In this proof, we borrow the technique used in \cite{lattimore2012pac} and \cite{dann2015sample} to bound the number of steps with inaccurate transition probabilities (while they use this technique to bound the number of steps with inaccurate game value). Note here again that $\pi_t(\cdot)$ can represent any history-dependent policy, and we hide its parameter $H_t=(s_1, a_1, r_1, ..., s_t)$ inside the subscript of $t$. 

Define the \textit{importance} of a joint action $a$ at time $t$ as 
\begin{align*}
\iota_t(a)\coloneqq \max\left\{ z_j : z_j \leq \frac{\pi_t(a)}{w_{\min}} \right\},
\end{align*}
and the its \textit{knownness} as 
\begin{align*}
\kappa_t(a) \coloneqq  \max\left\{ z_j : z_j \leq \frac{n_{k(t)}(s_t,a)}{m \pi_t(a)} \right\},
\end{align*}
with $z_1=0$, $z_j=2^{j-2} \ \forall j=2, 3, ...$, and some pre-defined $w_{\min}>0$, $m>0$. Note that we can always define them in hindsight even though the learner does not know $\pi_t^2$. These two amounts make partitions to the action set available at $s_t$. The partitioning is based on the actions' probability of being selected at time $t$ (i.e., $\pi_t(a)$), and the accuracy it has been estimated (the larger $n_{k(t)}(s_t,a)$, the more accurate). Intuitively, the larger $\kappa_t(a)$, the less likely will action $a$ contribute to inaccurate transition probability estimation. Define the partitions by $X_{t,\kappa,\iota}\coloneqq \{a: \kappa_t(a)=\kappa \text{\ and\ } \iota_t(a)=\iota\},\ \forall \kappa, \iota$. 

If we let $w_{\min}=\frac{\varepsilon}{3\sqrt{2\ln(1/\delta)}A}$ and $m=\frac{5\log_2^2(T/w_{\min})\ln(1/\delta)}{\varepsilon^2}$, with some $0< \varepsilon <1$, we can prove the following lemmas. 

\begin{lemma}
\label{lemma:sample_complexity1}
For any $s$, any $\kappa$ and any $\iota>0$, with probability at least $1-\delta$, 
\begin{align*}
\sum_{t=1}^{T} \mathbbm{1}_{s_t=s}\mathbbm{1}_{\lvert X_{t,\kappa,\iota}\rvert > \kappa} = \mathcal{O}\left(\frac{A\log_2^2(T/w_{\min})\ln(1/\delta)}{\varepsilon^2}\right).  
\end{align*}
\end{lemma}
%\begin{corollary}[Corollary of Lemma \ref{lemma:sample_complexity1}]
%\label{corollary:sample_complexity}
%For any $\kappa$ and any $\iota>0$, with probability at least $1-\delta$, 
%\begin{align*}
%\sum_{t=1}^{T} \mathbbm{1}_{\lvert X_{t,\kappa,\iota}\rvert > \kappa} = \mathcal{O}\left(\frac{SA\log_2^2(T/w_{\min})\ln(1/\delta)}{\varepsilon^2}\right).  
%\end{align*}
%\end{corollary}
%\begin{proof}
%Only needs to sum Lemma \ref{lemma:sample_complexity1}'s result over states. 
%\end{proof}

\begin{lemma}
\label{lemma:sample_complexity2}
If for all $\kappa$ and all $\iota>0$ we have $\lvert X_{t,\kappa, \iota}\rvert\leq \kappa$, then for any plausible $\tilde{p}$ in the confidence set $\mathcal{M}_{k(t)}$, $\lvert \tilde{p}(s^\prime\vert s_t, \pi_t)-p(s^\prime\vert s_t, \pi_t) \rvert \leq \varepsilon$ for all $s^\prime$.
\end{lemma}
%Combining Corollary \ref{corollary:sample_complexity} and Lemma \ref{lemma:sample_complexity2}, we see that there can only be around $\tilde{\mathcal{O}}(SA/\varepsilon^2)$ steps where the transition probabilities on the current state are inaccurately estimated.
 
\subsection{Proof of Lemma \ref{lemma:sample_complexity1}}
We prove Lemma \ref{lemma:sample_complexity1} with the help of Lemma \ref{lemma:sample_complexity3} and \ref{lemma:sample_complexity7}. 
\begin{lemma}
\label{lemma:sample_complexity3}
For any $s, \kappa$, and $\iota>0$, $\sum_{t=1}^{T}\mathbbm{1}_{s_t=s}\mathbbm{1}_{a_t\in X_{t,\kappa,\iota}}\leq 6Am(\kappa+1)\iota w_{\min}.$
\end{lemma}
\begin{proof}
First fix $a$. By the definition of importance, if $a\in X_{t, \kappa, \iota}$, then $\iota w_{\min}\leq \pi_t(a)< 2\iota w_{\min}$. In the case $\kappa>0$, we also have $m\kappa\pi_t(a)\leq n_{k(t)}(s_t,a)< 2m\kappa\pi_t(a)$. They two together imply $m\kappa\iota w_{\min} \leq n_{k(t)}(s_t,a)< 4m\kappa\iota w_{\min}$. This last inequality says that any $(s,a)$ cannot be sampled in the partition $(\kappa, \iota)$ for more than about $3m\kappa\iota w_{\min}$ times. This is because when $(s,a)$ is sampled once (i.e., $s_t=s, a_t=a$), $n_{k(t)}(s,a)$ will be increased by one, and this cannot happen for more than $4m\kappa\iota w_{\min}-m\kappa\iota w_{\min}$ times while $(s,a)\in X_{t,\kappa,\iota}$.  Since \textsc{UCSG} only updates $n_k(s,a)$ when new phases start and doubling the sample count of a state-action triple incurs a phase change, we use a more conservative bound of $6m\kappa\iota w_{\min}$. That is, we have
\begin{align}
\sum_{t=1}^{T} \mathbbm{1}_{s_t=s}\mathbbm{1}_{a_t=a } \mathbbm{1}_{a\in X_{t,\kappa,\iota}}\leq 6m\kappa\iota w_{\min}. \label{eqn:6mki}
\end{align}
In the case $\kappa=0$, we have $n_{k(t)}(s_t,a)<m\pi_t(a)<2m\iota w_{\min}$. Thus similarly, the sample counts of $(s,a)$ in the partition $(\kappa, \iota)$ cannot exceed $4m\iota w_{\min}$. The cases of $\kappa>0$ and $\kappa=0$ can then be combined into a single one: 
\begin{align}
\sum_{t=1}^{T} \mathbbm{1}_{s_t=s}\mathbbm{1}_{a_t=a} \mathbbm{1}_{a\in X_{t,\kappa,\iota}}\leq 6m(\kappa +1)\iota w_{\min}. \label{eqn:6mkplus1i}
\end{align}
Summing \eqref{eqn:6mkplus1i} over all actions leads to the statement in the lemma. 
\end{proof}
Now we sketch the argument of the next lemma. When $\iota>0$, each action in $X_{t,\kappa, \iota}$ are to be sampled with probability no less than $\iota w_{\min}$. If furthermore $\lvert X_{t,\kappa, \iota} \rvert$ is large, the probability that some $a\in X_{t,\kappa, \iota}$ is sampled will be also large.  However by Lemma \ref{lemma:sample_complexity3}, the total times elements in partition $(\kappa, \iota)$ are sampled are upper bounded. Therefore, we can conclude that $\vert X_{t,\kappa,\iota}\vert $ cannot be large for too many steps. Formally, we have
\begin{lemma}
\label{lemma:sample_complexity7}
With probability at least $1-\delta$, 
\begin{align*}
&\sum_{t=1}^{T}\mathbbm{1}_{s_t=s}\mathbbm{1}_{a_t\in X_{t,\kappa,\iota}} \geq \frac{1}{2}(\kappa+1)\iota w_{\min} \sum_{t=1}^{T}\mathbbm{1}_{s_t=s}\mathbbm{1}_{\lvert X_{t, \kappa,\iota}\rvert > \kappa}-\frac{9}{2}\log(T\delta^{-1}).
\end{align*}
\end{lemma}
\begin{proof}
To prove Lemma \ref{lemma:sample_complexity7}, we need the help of Lemma \ref{lemma:bernstein_bubeck}. 

Let $\mathcal{F}_{t-1}=H_{t}=(s_1, a_1, r_1 \cdots, s_{t-1}, a_{t-1},r_{t-1}, s_t)$ and 
\begin{align*}
Y_t=q_t- \mathbbm{1}_{s_t=s}\mathbbm{1}_{\lvert X_{t,\kappa, \iota}\rvert>\kappa}\mathbbm{1}_{a_t\in X_{t,\kappa,\iota}}, 
\end{align*}
where we define 
\begin{align*}
\scalebox{0.92}{$q_t\coloneqq\mathbbm{1}_{s_t=s}\mathbbm{1}_{\lvert X_{t,\kappa, \iota}\rvert>\kappa}\Pr\big\{ a_t\in X_{t,\kappa,\iota} \Big\vert s_t=s, \lvert X_{t,\kappa,\iota}\rvert>\kappa \big\}$}.
\end{align*}
Then Lemma \ref{lemma:bernstein_bubeck}'s conditions are met with $b=1$. Moreover, 
\begin{align*}
V_T=\sum_{t=1}^{T}q_t(1-q_t)\leq \sum_{t=1}^{T}q_t.
\end{align*}
Substituting them into Lemma \ref{lemma:bernstein_bubeck} and rearraging terms, we get that with probability $\geq 1-\delta$, 
\begin{align}
&\sum_{t=1}^{T}\mathbbm{1}_{s_t=s}\mathbbm{1}_{\lvert X_{t,\kappa,\iota}\rvert > \kappa} \mathbbm{1}_{a_t\in X_{t,\kappa,\iota}}\geq \left(\textstyle\sum_{t=1}^{T}q_t\right)- 2\sqrt{\left(\textstyle\sum_{t=1}^{T}q_t\right)\log(T\delta^{-1})} - \sqrt{5}\log(T\delta^{-1}). \nonumber
\end{align}
Solving the above inequality with respect to $\sqrt{\sum_{t=1}^Tq_t}$, we can bound with probability $\geq 1-\delta$ that 
\begin{align}
&\sum_{t=1}^{T}\mathbbm{1}_{s_t=s}\mathbbm{1}_{\lvert X_{t,\kappa,\iota}\rvert > \kappa} \mathbbm{1}_{a_t\in X_{t,\kappa,\iota}}\geq \frac{1}{2} \sum_{t=1}^{T} q_t- \frac{9}{2}\log(T\delta^{-1}). \label{eqn:bernstein_bounded}
\end{align}
Finally we look at $q_t$. Since each action in $X_{t,\kappa,\iota}$ are drawn at time $t$ with probability at least $\iota w_{\min}$, we have 
\begin{align}
q_t&\geq\mathbbm{1}_{s_t=s}\mathbbm{1}_{\lvert X_{t,\kappa, \iota}\rvert>\kappa}\left(\textstyle\sum_{a\in X_{t,\kappa,\iota}}\iota w_{\min}\right) \geq(\kappa+1)\iota w_{\min}\mathbbm{1}_{s_t=s}\mathbbm{1}_{\lvert X_{t,\kappa, \iota}\rvert>\kappa}.  \label{eqn:bound_qt}
\end{align}
Combining \eqref{eqn:bernstein_bounded}, \eqref{eqn:bound_qt}, and noting that $\mathbbm{1}_{s_t=s} \mathbbm{1}_{a_t\in X_{t,\kappa,\iota}}\geq\mathbbm{1}_{s_t=s}\mathbbm{1}_{\lvert X_{t,\kappa,\iota}\rvert > \kappa} \mathbbm{1}_{a_t\in X_{t,\kappa,\iota}}$ concludes the proof. 
\end{proof}
\begin{proof}[Proof of Lemma \ref{lemma:sample_complexity1}]
Combining Lemma \ref{lemma:sample_complexity3} and \ref{lemma:sample_complexity7}, we have
\begin{align}
\sum_{t=1}^{T}\mathbbm{1}_{s_t=s}\mathbbm{1}_{\lvert X_{t,\kappa,\iota}\rvert > \kappa} \leq 12Am + \frac{9}{(\kappa+1)\iota w_{\min}} \label{eqn:12Am_bound}
\end{align}
with probability no less than $1-\delta$. The lemma is then proved by substituting the selection of $m$ and $w_{\min}$ into \eqref{eqn:12Am_bound}, and using $\kappa+1\geq 1$, $\iota\geq 1$. 
\end{proof}

\subsection{Proof of Lemma \ref{lemma:sample_complexity2}}

\begin{proof}[Proof of Lemma \ref{lemma:sample_complexity2}]
\begin{align*}
\lvert \tilde{p}(s^\prime\vert s_t, \pi_t) - p(s^\prime \vert s_t, \pi_t)\rvert &\leq \sum_{a} \pi_t(a) \lvert \tilde{p}(s^\prime\vert s_t, a)-p(s^\prime \vert s_t,a) \rvert \\
&\leq \scalebox{0.83}{$\displaystyle \sqrt{2\ln\frac{1}{\delta_1}} \left(\sum_{a:\iota_t(a)=0}\sqrt{\frac{\pi_t(a)^2}{n_{k(t)}(s_t,a)}} + \sum_{\substack{\kappa, \iota: \\ \iota>0}} \sum_{a\in X_{t,\kappa,\iota}}\sqrt{\frac{\pi_t(a)^2}{n_{k(t)}(s_t,a)}} \right)$} \\
&\leq \scalebox{0.9}{$\displaystyle\sqrt{2\ln\frac{1}{\delta_1}}\left(Aw_{\min} + \sum_{\substack{\kappa, \iota: \\ \iota>0}}  \sqrt{\lvert X_{t,\kappa, \iota}\rvert \sum_{a\in X_{t, \kappa,\iota}} \frac{\pi_t(a)^2}{n_{k(t)}(s_t,a)}}\right)$} \\
&\leq \scalebox{0.9}{$\displaystyle\sqrt{2\ln\frac{1}{\delta_1}}\left(Aw_{\min} + \sum_{\substack{\kappa, \iota: \\ \iota>0, \kappa>0}}  \sqrt{\kappa \sum_{a\in X_{t, \kappa,\iota}} \frac{\pi_t(a)}{m\kappa}}\right)$} \\
&\leq  \scalebox{0.9}{$\displaystyle\sqrt{2\ln\frac{1}{\delta_1}}\left(Aw_{\min} +   \sqrt{\lvert \mathcal{K}\times \mathcal{I}\rvert \sum_{\substack{\kappa, \iota: \\ \iota>0, \kappa>0}}\sum_{a\in X_{t, \kappa,\iota}} \frac{\pi_t(a)}{m}}\right)$} \\
&\leq  \scalebox{0.9}{$\displaystyle\sqrt{2\ln\frac{1}{\delta_1}}\left(Aw_{\min} +   \sqrt{\frac{\lvert \mathcal{K}\times \mathcal{I}\rvert}{m}}\right)$}, 
\end{align*}
where $\mathcal{K}$ and $\mathcal{I}$ are the set of effective $\kappa$'s and $\iota$'s in the above summation (only partitions with $\iota>0$ and $\kappa>0$ are relevant). By definition, there are at most $\log_2\left(\frac{1}{w_{\min}}\right)$ different values of $\iota$ for $\iota>0$, and $\log_2\left(\frac{T}{mw_{\min}}\right)\leq \log_2\left(\frac{T}{w_{\min}}\right)$ different values for $\kappa>0$ when $\iota>0$. The second inequality is by the definition of the confidence set; the third and the fifth are by Cauchy's inequality; the fourth is by the assumption of the lemma. Substituting the values of $w_{\min}$
and $m$ into the last expression, we can get the desired result. 
\end{proof}

\section{Proofs of Lemma  \ref{lemma:undefined} and \ref{lemma:bound_benign} }

To prove Lemma \ref{lemma:undefined} and \ref{lemma:bound_benign}, the following lemma is a useful tool.
In the following texts, we let  $v_k(s)\coloneqq \sum_{t=t_k}^{t_{k+1}-1}\mathbbm{1}_{s_t=s}$, and write the joint policy $(\pi_k^1, \bar{\pi}_k^2)$ as $\bar{\pi}_k$. 
\begin{lemma}
\label{lemma:useful}
Let $v\geq 1$. Then $\forall s$, with high probability, $\sum_k T_k\mathbbm{1}_{v_k(s)\leq v}= \mathcal{\tilde{O}}(vDSA).$
\end{lemma}
\begin{proof}
Under Assumption \ref{assumption:irreducible}, the times a state is visited within an interval of length $D$ is in average no less than $1$ (no matter what policies the players play). Consider any arbitrarily chosen time frame $[\tau, \tau^\prime) \subset [1,T]$. In this time frame, there are $\lfloor \frac{\tau^\prime-\tau}{2D} \rfloor$ intervals each with length $2D$. By Markov's inequality, the probability $s$ is visited at least once within each interval is lower bounded by $\frac{1}{2}$. With Azuma-Hoeffding's inequality, we have with probability at least $1-\frac{\delta}{T^2}$ that
\begin{align}
\sum_{t=\tau}^{\tau^\prime-1} \mathbbm{1}_{s_t=s}&\geq \frac{1}{2} \Big\lfloor \frac{\tau^\prime-\tau}{2D} \Big\rfloor - \sqrt{\Big\lfloor \frac{\tau^\prime-\tau}{2D} \Big\rfloor \log\left(\frac{T^2}{\delta}\right)} \nonumber \\
&\geq \frac{1}{4} \Big\lfloor \frac{\tau^\prime-\tau}{2D} \Big\rfloor - 4\log\left(\frac{T^2}{\delta}\right) \nonumber \\
&\geq \frac{1}{4}\frac{\tau^\prime-\tau}{2D}-\frac{1}{4}-4\log\left(\frac{T^2}{\delta}\right), \label{eqn:useful_1}
\end{align}
where the second inequality is easily verified by substracting RHS from LHS, and the third inequality is by the property of the floor function. Using an union bound over all possible $\tau$ and $\tau^\prime$, we get that  \eqref{eqn:useful_1} holds for all $\tau, \tau^\prime$ with probability at least $1-\delta$. 

Now apply \eqref{eqn:useful_1} to all phases $k$ with $v_k(s)\leq v$, and sum all of them up. Then we get
\begin{align}
\sum_{k: v_k(s) \leq v}v_k(s) \geq \sum_{k: v_k(s) \leq v} \left(\frac{T_k}{8D} -\frac{1}{4}-4\log\left(\frac{T^2}{\delta}\right) \right)  \nonumber 
\end{align}
or 
\begin{align}
\sum_{k: v_k(s) \leq v} T_k \leq 8D\sum_{k: v_k(s) \leq v} \left( v_k(s)+\frac{1}{4}+ 4\log(T^2/\delta)  \right). \label{eqn:useful_2}
\end{align}
Since there are at most $SA\log_2T$ phases, the RHS of \eqref{eqn:useful_2} is further bounded by $\left(8vD + 2D+32D\log(T^2/\delta)\right) SA\log_2T$, which proves this lemma. 
\end{proof}

%\subsection{Proof of Lemma \ref{lemma:undefined} }
\begin{proof}[Proof of Lemma \ref{lemma:undefined}]
$\bar{\pi}_k^{2}$ is not well-defined if and only if there is a $s$ such that $v_k(s)=\sum_{t=t_k}^{t_{k+1}-1} \mathbbm{1}_{s_t=s}=0$. The proof is done by simply applying Lemma \ref{lemma:useful} with $v=1$ together with a union bound over all states $s$.  
\end{proof}

%\subsection{Proof of Lemma \ref{lemma:bound_by_empirical}}
We prove Lemma \ref{lemma:bound_benign} by proving the following Lemma \ref{lemma:bound_by_empirical} and \ref{lemma:bound_h_by_2D}. 
\begin{lemma}
\label{lemma:bound_by_empirical}
\begin{align*}
&\scalebox{0.95}{$\displaystyle \sum_{k} T_k\mathbbm{1}\{\bar{\pi}_k^2 \text{ is well-defined}\}\mathbbm{1}\bigg\{\exists s, \mu(M,  \bar{\pi}_k, s) > \frac{2v_k(s)}{T_k} \bigg\}  $} \leq \tilde{\mathcal{O}}(D^3S^4A) \ \ \text{with high probability.}
\end{align*}
%With high probability, except for at most $\tilde{\mathcal{O}}(D^3S^4A)$ time steps, if $\bar{\pi}_k^{2}$ is well-defined, then $\mu(M,  \bar{\pi}_k, s) \leq \frac{3v_k(s)}{2T_k}, \ \forall s$.  
\end{lemma}
\begin{lemma}
\label{lemma:bound_h_by_2D}
\begin{align*}
&\sum_{k} T_k\mathbbm{1}\{\bar{\pi}_k^2 \text{ is well-defined}\}\mathbbm{1}\Big\{\spa (h(M_k^1, \bar{\pi}_k, \cdot)) > 2D \Big\} \leq \tilde{\mathcal{O}}(D^3S^5A) \ \ \text{with high probability.}
%With high probability, except for at most $\tilde{\mathcal{O}}(D^3S^5A)$ time steps, if $\bar{\pi}_k^{2}$ is well-defined, then $\spa (h(M_k^1, \bar{\pi}_k, \cdot)) \leq 2D$.
\end{align*}  
\end{lemma}

\subsection{Proof of Lemma \ref{lemma:bound_by_empirical}}
\begin{proof}[Proof of Lemma \ref{lemma:bound_by_empirical}]
This lemma says, the stationary distribution of the irreducible Markov chain induced by $\pi_k^1$ and $\bar{\pi}_k^{2}$ won't exceed the empirical distribution too much in most steps. To prove Lemma \ref{lemma:bound_by_empirical}, we will compare three transition probabilities: 
\begin{align*}
 \bar{p}_k(s^\prime\vert s)&\coloneqq p(s^\prime\vert s, \pi_k^1, \bar{\pi}_k^{2})=    \scalebox{0.92}{$\displaystyle \frac{\sum_{t=t_k}^{t_{k+1}-1}\mathbbm{1}_{s_t=s}p(s^\prime\vert s, \pi_k^1, \pi_t^2)}{\sum_{t=t_k}^{t_{k+1}-1}\mathbbm{1}_{s_t=s}},$} \\
\hat{p}_k(s^\prime\vert s)&\coloneqq \frac{\sum_{t=t_k}^{t_{k+1}-1}\mathbbm{1}_{s_t=s}\mathbbm{1}_{s_{t+1}=s^\prime}}{\sum_{t=t_k}^{t_{k+1}-1} \mathbbm{1}_{s_t=s}}, \\
\tilde{p}_k(s^\prime\vert s)&\coloneqq \frac{\sum_{t=t_k}^{t_{k+1}-2}\mathbbm{1}_{s_t=s}\mathbbm{1}_{s_{t+1}=s^\prime} + \mathbbm{1}_{s_{t_{k+1}-1}=s}\mathbbm{1}_{s_{t_k}=s^\prime}}{\sum_{t=t_k}^{t_{k+1}-1} \mathbbm{1}_{s_t=s}}, 
\end{align*} 
and use perturbation analysis to claim that when they are close enough, the stationary distributions they induce will also be close. Here, $\hat{p}_k$ is constructed by counting empirical transitions. $\tilde{p}_k$ is only slightly modified from $\hat{p}_k$: the last term in the numerator changes from $\mathbbm{1}_{s_{t_{k+1}-1}}\mathbbm{1}_{s_{t_{k+1}}}$ to $\mathbbm{1}_{s_{t_{k+1}-1}}\mathbbm{1}_{s_{t_k}}$. Under the condition that $\bar{\pi}_k^{2}$ is well-defined, $\sum_{t=t_k}^{t_{k+1}-1}\mathbbm{1}_{s_t=s} \neq 0\ \forall s$, which means that $\tilde{p}_k$ has non-zero probability to reach any states from any states, hence inducing an irreducible Markov chain. $\bar{p}_k$ also induces an irreducible Markov chain by Assumption \ref{assumption:irreducible}. We denote the stationary distributions corresponding to $\bar{p}_k$ and $\tilde{p}_k$ by $\bar{\mu}_k$ and $\tilde{\mu}_k$.

We will see that $\tilde{\mu}_k$ is exactly the same as the empirical distribution (i.e., $\tilde{\mu}_k(s)=\frac{v_k(s)}{T_k}$). By Theorem \ref{theorem:perturbation_stationary_distribution}, when two transition probabilities are close enough, their stationary distributions will also be close. We will argue that except for a constant amount of steps, $\lvert \bar{p}_k(s^\prime\vert s)-\hat{p}_k(s^\prime\vert s)\rvert \leq \frac{1}{2DS}$ and $\lvert \hat{p}_k(s^\prime\vert s)-\tilde{p}_k(s^\prime\vert s)\rvert \leq \frac{1}{2DS}$ hold for all $s,s^\prime$. When they both hold, we can use Theorem \ref{theorem:perturbation_stationary_distribution} with $\norm{E}_\infty=\max_{s,s^\prime}| \bar{p}_k(s^\prime\vert s)-\tilde{p}_k(s^\prime\vert s)|\leq \frac{1}{DS}$ and bound $\lvert \bar{\mu}_k(s)-\tilde{\mu}_k(s) \rvert\leq \frac{1}{2}\bar{\mu}_k(s)$. This will directly imply $\bar{\mu}_k(s)\leq 2\tilde{\mu}_k(s)=\frac{2v_k}{T_k}$. 

From the discussion above, Lemma \ref{lemma:bound_by_empirical} is proved as long as the three following lemmas (Lemma \ref{lemma:near_probability1}, \ref{lemma:near_probability2}, \ref{lemma:empirical_distribution_as_stationary}) are proved. 
\end{proof}
\begin{lemma}
\label{lemma:near_probability1}

\begin{align*}
&\scalebox{1}{$\displaystyle \sum_{k} T_k\mathbbm{1}\{\bar{\pi}_k^2 \text{ is well-defined}\}\mathbbm{1}\Big\{   
\exists s,s^\prime, \lvert \bar{p}_k(s^\prime \vert s)-\hat{p}_k(s^\prime \vert s)\rvert\ > \frac{1}{2DS}    \Big\}$} \leq \tilde{\mathcal{O}}(D^3S^4A) \ \ \textit{w.h.p.}
\end{align*}  
%\begin{align*}
%\forall s,s^\prime, \lvert \bar{p}_k(s^\prime \vert s)-\hat{p}_k(s^\prime \vert s)\rvert\ \leq \frac{1}{2DS}.
%\end{align*}
\end{lemma}
\begin{lemma}
\label{lemma:near_probability2}\begin{align*}
&\scalebox{1.0}{$\displaystyle \sum_{k} T_k\mathbbm{1}\{\bar{\pi}_k^2 \text{ is well-defined}\}\mathbbm{1}\Big\{   
\exists s,s^\prime, \lvert \hat{p}_k(s^\prime \vert s)-\tilde{p}_k(s^\prime \vert s)\rvert\ > \frac{1}{2DS}    \Big\}$} \leq \tilde{\mathcal{O}}(D^2S^3A) \ \ \textit{w.h.p.}
\end{align*}  
\end{lemma}
\begin{lemma}
\label{lemma:empirical_distribution_as_stationary}
\begin{align*}
\tilde{\mu}_k(s)=\frac{v_k(s)}{T_k}. 
\end{align*}
\end{lemma}

\begin{proof}[Proof of Lemma \ref{lemma:near_probability1}]
Fix $s$, $s^\prime$, and $k$. Consider the martingale difference sequence defined by $Y_t\coloneqq \mathbbm{1}_{s_t=s}\left(p(s^\prime\vert s, \pi_{k(t)}^1, \pi_t^2)-\mathbbm{1}_{s_{t+1}=s^\prime}\right)$, where $k(t)$ denotes the phase to which time step $t$ belongs. By Lemma \ref{lemma:bernstein_bubeck}, for any $\tau\leq T+1$, with probability at least $1-2\delta/T$, 
\begin{align}
\Bigg\lvert \sum_{t=t_k}^{\tau-1}Y_t \Bigg\rvert \leq 2\sqrt{V_{t_k, \tau}\log(T^2\delta^{-1})}+\sqrt{5}\log(T^2\delta^{-1}). \label{eqn:near_probability1_1}
\end{align}
Here $V_{t_k,\tau}=\sum_{t=t_k}^{\tau-1}q_t(1-q_t)\leq \sum_{t=t_k}^{\tau-1}q_t\leq \sum_{t=t_k}^{\tau-1} \mathbbm{1}_{s_t=s}$ where $q_t\coloneqq \mathbbm{1}_{s_t=s}p(s^\prime\vert s, \pi_{k(t)}^1, \pi_t^2)\leq \mathbbm{1}_{s_t=s}$. With an union bound, we have that \eqref{eqn:near_probability1_1} holds for all $\tau$ with probability at least $1-2\delta$. Now pick $\tau$ to be $t_{k+1}$, and thus $V_{t_k, t_{k+1}}\leq \sum_{t=t_k}^{t_{k+1}-1}\mathbbm{1}_{s_t=s}=v_k(s)$. Then we have
\begin{align*}
\lvert \bar{p}_k(s^\prime\vert s) -\hat{p}_k(s^\prime\vert s)\rvert &= \Bigg\vert \frac{\sum_{t=t_k}^{t_{k+1}-1}\mathbbm{1}_{s_t=s}(p(s^\prime\vert s, \pi_k^1, \pi_t^2)-\mathbbm{1}_{s_{t+1}=s^\prime})}{v_k(s)} \Bigg\vert \nonumber \\
&\leq 2\sqrt{\frac{\log(T^2\delta^{-1})}{v_k(s)}} + \frac{\sqrt{5}\log(T^2\delta^{-1})}{v_k(s)} 
\end{align*}
with probability at least $1-2\delta$. Another union bound over $s^\prime$ lets the above inequality holds for all $s^\prime$ with probability at least $1-2S\delta$. 

We need about $v_k(s)\geq 25D^2S^2\log(T^2\delta^{-1})$ to make $\lvert \bar{p}_k(s^\prime\vert s) -\hat{p}_k(s^\prime\vert s)\rvert\leq \frac{1}{2DS} \ \ \forall s^\prime$ in the above inequality. By Lemma \ref{lemma:useful}, we see that the number of steps not satisfying this condition is upper bounded by $\tilde{\mathcal{O}}(D^3S^3A)$. Another union bound over $s$ proves the lemma.  
\end{proof}
\begin{proof}[Proof of Lemma \ref{lemma:near_probability2}]
By the construction of $\tilde{p}_k$,  $\lvert \tilde{p}_k(s^\prime\vert s)-\hat{p}_k(s^\prime\vert s)\rvert \leq \frac{1}{v_k(s)}$ $\forall s^\prime$. Again, we use Lemma \ref{lemma:useful} and set the threshold $v=\tilde{\Theta}(2DS)$ to make $\lvert \tilde{p}_k(s^\prime\vert s)-\hat{p}_k(s^\prime\vert s)\rvert\leq \frac{1}{2DS}$ $\forall s^\prime$. By Lemma \ref{lemma:useful}, this will hold except for $\tilde{\mathcal{O}}(D^2S^2A)$ steps. An union bound over states leads to the  $\tilde{\mathcal{O}}(D^2S^3A)$ bound.
\end{proof}

\begin{proof}[Proof of Lemma \ref{lemma:empirical_distribution_as_stationary}]
We only need to check whether the equation $\tilde{\mu}_k(s^\prime) = \sum_s \tilde{\mu}_k(s)\tilde{p}_k(s^\prime\vert s)$ holds for all $s, s^\prime$. Indeed, 
\begin{align*}
\sum_s \tilde{\mu}_k(s)\tilde{p}_k(s^\prime\vert s)&=\sum_s\frac{v_k(s)}{T_k}\frac{\sum_{t=t_k}^{t_{k+1}-2}\mathbbm{1}_{s_t=s}\mathbbm{1}_{s_{t+1}=s^\prime} + \mathbbm{1}_{s_{t_{k+1}-1}=s}\mathbbm{1}_{s_{t_k}=s^\prime}}{v_k(s)}\\
&=\frac{\sum_{t=t_k}^{t_{k+1}-2}\mathbbm{1}_{s_{t+1}=s^\prime} +\mathbbm{1}_{s_{t_k}=s^\prime}}{T_k}=\tilde{\mu}_k(s^\prime).
\end{align*}
\end{proof}

\subsection{Proof of Lemma \ref{lemma:bound_h_by_2D} }
\begin{proof}[Proof of Lemma \ref{lemma:bound_h_by_2D}]
By Assumption \ref{assumption:irreducible}, the maximum mean first passage time under model $M$ and policy pair $(\pi_k^1, \bar{\pi}_k^2)$ does not exceed $D$, i.e., $T^{\pi_k^1, \bar{\pi}_k^2}(M) \leq D$. Then by Theorem \ref{theorem:perturbation_first_passage_time}, we know that if all transition probabilities in the Markov chain induced by $(M_k^1, \pi_k^1, \bar{\pi}_k^2)$ is perturbed from that induced by $(M, \pi_k^1, \bar{\pi}_k^2)$ within the amount of $\frac{1}{8DS^2}$, the former's maximum mean first passage time can be bounded by two times the latter's, i.e., $T^{\pi_k^1, \bar{\pi}_k^2}(M_k^1) \leq 2T^{\pi_k^1, \bar{\pi}_k^2}(M)$. This also implies that $(M_k^1, \pi_k^1, \bar{\pi}_k^2)$ induces an irreducible Markov chain. Finally, by Remark \ref{remark:mean_first_passage_time}, we have $\spa(h(M_k^1, \pi_k^1, \bar{\pi}_k^2, \cdot))\leq T^{\pi_k^1, \bar{\pi}_k^2}(M_k^1)$. Combining the three inequalities above, we can have $\spa(h(M_k^1, \pi_k^1, \bar{\pi}_k^2, \cdot)) \leq 2D$. As a result, to prove this theorem, we only need to bound the number of steps in phases where there exist $s, s^\prime$ such that the transition probability difference $\lvert p_k^1(s^\prime \vert s, \pi_k^1, \bar{\pi}_k^{2}) - p(s^\prime \vert s, \pi_k^1, \bar{\pi}_k^{2})\rvert $ is larger than $\frac{1}{8DS^2}$ ($p_k^1$ is the transition kernel of $M_k^1$). We define the event $E_k(s)=\big\{\exists s^\prime, \lvert p_k^1(s^\prime \vert s, \pi_k^1, \bar{\pi}_k^{2}) - p(s^\prime \vert s, \pi_k^1, \bar{\pi}_k^{2})\rvert >\frac{1}{8DS^2}  \big\}$, and $E_k=\{\exists s, E_k(s)=1\}$. Our goal is to prove $\sum_k T_k\mathbbm{1}_{E_k} \leq \tilde{\mathcal{O}}(D^3S^5A)$.

Fix $k$. Suppose that $\bar{\pi}_k^{2}$ is well-defined. By the definition of $\bar{\pi}_k^{2}$ and the triangle inequality, we have 
\begin{align}
&\lvert p_k^1(s^\prime \vert s, \pi_k^1, \bar{\pi}_k^{2}) - p(s^\prime \vert s, \pi_k^1, \bar{\pi}_k^{2})\rvert \leq \frac{1}{v_k(s)} \sum_{t=t_k}^{t_{k+1}-1} \mathbbm{1}_{s_t=s} \lvert p_k^1(s^\prime \vert s_t, \pi_t) - p(s^\prime \vert s_t, \pi_t)\rvert. \label{eqn:episode_transition_prob}
\end{align}
Define $\varepsilon_i\coloneqq 2^{-i}$, and define 
\begin{align*}
&G_k(s,\varepsilon)= \{t\in[t_k, t_{k+1}):    s_t=s  \text{\ and\ } \varepsilon<\max_{s^\prime} \lvert  p_k^1(s^\prime \vert s_t, \pi_t) - p(s^\prime \vert s_t, \pi_t)\rvert \leq 2\varepsilon \} , 
\end{align*}
and $g_k(s,\varepsilon)\coloneqq \lvert G_k(s,\varepsilon)\rvert$, i.e., $g_k(s,\varepsilon)$ is the number of steps $t$ in phase $k$ such that $s_t=s$ and the maximum transition probability error at that step is between $\varepsilon$ and $2\varepsilon$. With these definitions, we can continue to upper bound \eqref{eqn:episode_transition_prob} by
\begin{align}
\vert p_k^1(s'\vert s, \pi_k^1, \bar{\pi}_k^2) - p (s'\vert s, \pi_k^1, \bar{\pi}_k^2) \vert & \leq \frac{1}{v_k(s)} \sum_{t=t_k}^{t_{k+1}-1} \mathbbm{1}_{s_t=s} \lvert p_k^1(s^\prime \vert s_t, \pi_t) - p(s^\prime \vert s_t, \pi_t)\rvert \nonumber \\
&\leq\frac{1}{v_k(s)} \left( \sum_{2\varepsilon_i > \frac{1}{24DS^2}} 2\varepsilon_ig_k(s, \varepsilon_i) + \frac{1}{24DS^2}v_k(s)\right) \nonumber \\
&= \frac{1}{24DS^2}+\sum_{i=1}^{\lfloor \log_2(48DS^2) \rfloor} \frac{2\varepsilon_ig_k(s,\varepsilon_i)}{v_k(s)}.  \label{eqn:24DS^2}
\end{align} 
If $\vert p_k^1(s'\vert s, \pi_k^1, \bar{\pi}_k^2) - p (s'\vert s, \pi_k^1, \bar{\pi}_k^2) \vert > \frac{1}{8DS^2}$, then by \eqref {eqn:24DS^2} we have 
\begin{align}
\frac{v_k(s)}{24DS^2} \leq \sum_{i=1}^{\lfloor \log_2(48DS^2) \rfloor} \varepsilon_ig_k(s,\varepsilon_i).\label{eqn:epsilong}
\end{align}
Note that since steps counted in $G_k(s,\varepsilon)$ have maximum transition errors greater that $\varepsilon$, by Lemma \ref{lemma:epsilon_accurate}, with high probability, $\sum_k g_k(s, \varepsilon)$ won't exceed $\frac{c_1A}{\varepsilon^2}$, for some $c_1$ hides logarithmic terms. Now sum the above equation over phases where $E_k(s)$ holds, we get that
\begin{align*}
&\sum_{k:E_k(s)} \frac{v_k(s)}{24DS^2} \leq \sum_{k:E_k(s)}\sum_{i=1}^{\lfloor \log_2(48DS^2) \rfloor} \varepsilon_i g_k(s,\varepsilon_i) \leq  \sum_{i=1}^{\lfloor \log_2(48DS^2) \rfloor} \frac{c_1A}{\varepsilon_i} \leq 48c_1DS^2A \log_2(48DS^2)  
\end{align*} 
or $\sum_{k:E_k(s)}v_k(s)\leq \tilde{\mathcal{O}}(D^2S^4A)$ holds with high probability. Similar to the proof of Lemma \ref{lemma:useful}, we use \eqref{eqn:useful_1} and lower bound $\sum_{k:E_k(s)}v_k(s) \geq \sum_{k:E_k(s)} \left(\frac{T_k}{8D}-\tilde{\mathcal{O}}(1)\right)$. Combining the lower bound and the upper bound, we get $\sum_{k:E_k(s)}T_k \leq \tilde{\mathcal{O}}(D^3S^4A)$. Finally, summing over $s$, we get the desired bound. 
\end{proof}

\section{Proofs of Lemma \ref{lemma:bound_fourth_term} and \ref{lemma:typical}}

%\subsection{Proofs of Lemma \ref{lemma:bound_fourth_term}}
\begin{proof}[Proof of Lemma \ref{lemma:bound_fourth_term}]
Define notations: $\bar{\pi}_k=(\pi_k^1, \bar{\pi}_k^2)$, $\bar{p}_k(s^\prime\vert s)\coloneqq p(s^\prime\vert s, \bar{\pi}_k)$, $\bar{h}_k(s)\coloneqq h(M, \bar{\pi}_k, s)$, $\bar{\rho}_k\coloneqq \rho(M, \bar{\pi}_k)$, $\bar{r}_k(s)\coloneqq r(s,\bar{\pi}_k)$, $r_t\coloneqq r(s_t,a_t)$. 

By the construction of $\bar{\pi}_k^{2}$, we have
\begin{align}
\bar{p}_k(s^\prime\vert s)&=\sum_{a^2} \frac{\sum_{t=t_k}^{t_{k+1}-1} \mathbbm{1}_{s_t=s}\pi_t(a^2)}{v_k(s)}p(s^\prime\vert s, \pi_k^1, a^2) =\frac{1}{v_k(s)}\sum_{t=t_k}^{t_{k+1}-1}\mathbbm{1}_{s_t=s}p(s^\prime\vert s,\pi_k^1, \pi_t^2) \label{eqn:p*(s'|s)}
\end{align}
and 
\begin{align}
\bar{r}_k(s)&=\sum_{a^2} \frac{\sum_{t=t_k}^{t_{k+1}-1}\mathbbm{1}_{s_t=s}\pi_t^2(a^2)}{v_k(s)}r(s, \pi_k^1, a^2)=\frac{1}{v_k(s)}\sum_{t=t_k}^{t_{k+1}-1}\mathbbm{1}_{s_t=s}r(s,\pi_k^1, \pi_t^2).   \label{eqn:r*(s'|s)}
\end{align}
Our target in phase $k$ can be decomposed as: 
\begin{align}
&T_k\bar{\rho}_k-\sum_{t=t_k}^{t_{k+1}-1}r_t  = \sum_{t=t_k}^{t_{k+1}-1} \left(\bar{\rho}_k- \bar{r}_k(s_t)\right) + \sum_{t=t_k}^{t_{k+1}-1} \left(\bar{r}_k(s_t)-r_t\right), \label{eqn:decompose2}
\end{align} 
Now manipulate individual terms. 
\begin{align}
\sum_{t=t_k}^{t_{k+1}-1} \left(\bar{\rho}_k- \bar{r}_k(s_t)\right) &=\sum_{t=t_k}^{t_{k+1}-1}  \left(\sum_{s^\prime}\bar{p}_k(s^\prime \vert s_t)\bar{h}_k(s^\prime) -\bar{h}_k(s_t)\right)\nonumber \\
&=\sum_{s,s^\prime}v_k(s)\bar{p}_k(s^\prime\vert s)\bar{h}_k(s^\prime)-\sum_{t=t_k}^{t_{k+1}-1} \bar{h}_k(s_t) \nonumber \\
&=\sum_{s,s^\prime}\sum_{t=t_k}^{t_{k+1}-1}\mathbbm{1}_{s_t=s}p(s^\prime\vert s, \pi_k^1, \pi_t^2)\bar{h}_k(s^\prime)- \sum_{t=t_k}^{t_{k+1}-1} \bar{h}_k(s_t) \nonumber \\
&=\sum_{t=t_k}^{t_{k+1}-1}\sum_{s^\prime} p(s^\prime\vert s_t,\pi_k^1, \pi_t^2)\bar{h}_k(s^\prime)-\sum_{t=t_k}^{t_{k+1}-1}\bar{h}_k(s_t), \label{eqn:fourth_term1}
\end{align}
where the third equality follows from \eqref{eqn:p*(s'|s)}; 
\begin{align}
\sum_{t=t_k}^{t_{k+1}-1}\left(\bar{r}_k(s_t)-r_t\right)&=\sum_s v_k(s)\bar{r}_k(s)-\sum_{t=t_k}^{t_{k+1}-1}r_t \nonumber \\
&=\sum_s\sum_{t=t_k}^{t_{k+1}-1}\mathbbm{1}_{s_t=s}r(s,\pi_k^1,\pi_t^2)-\sum_{t=t_k}^{t_{k+1}-1}r_t \nonumber \\
&=\sum_{t=t_k}^{t_{k+1}-1}r(s_t, \pi_k^1, \pi_t^2)-\sum_{t=t_k}^{t_{k+1}-1}r_t, \label{eqn:fourth_term2}
\end{align}
where the second equality follows from \eqref{eqn:r*(s'|s)}. Substituting \eqref{eqn:fourth_term1} and \eqref{eqn:fourth_term2} into \eqref{eqn:decompose2}, we get
\begin{align}
&\scalebox{0.9}{$\displaystyle T_k\bar{\rho}_k - \sum_{t=t_k}^{t_{k+1}-1}r_t =\bar{h}_k(s_{t_{k+1}})-\bar{h}_k(s_{t_k})+\sum_{t=t_k}^{t_{k+1}-1}\left(Y_t^1+Y_t^2\right),$}   \label{eqn:contribution_from_third_term} 
\end{align}
where $Y_t^{1}\coloneqq \left(\sum_{s^\prime}p(s^\prime\vert s_t, \pi_k^1, \pi_t^2)\bar{h}_k(s^\prime)-\bar{h}_k(s_{t+1})\right)$, and $Y_t^2\coloneqq \left(r(s_t, \pi_k^1, \pi_t^2)-r_t\right)$. It seems that $Y_t^{1}$ and $Y_t^2$ have expectations of zero and should be able to be bounded with Bernstein's inequality. Nevertheless, one needs to be careful about that $\bar{h}_k$ depends on $\bar{\pi}_k^{2}$, which is only known after phase $k$ ends. In other words, $\bar{h}_k$ is not $\mathcal{F}_{t}$-measurable for $t\in \text{ph}(k)$, where $\mathcal{F}_{t-1}\coloneqq \{s_1,a_1, \cdots, s_t\}$. The solution is as follows. Let $\mathcal{D}$ be the set where $\bar{h}_k$ possibly lies. We discretize $\mathcal{D}$ and use the Bernstein bound on all discretization points. Finally, we use the fact that $\bar{h}_k$ is not far from the nearest discretization point to bound the sum of $Y_t^{1}$.  

Let $\mathcal{D}\coloneqq [-D,D]^S$, and thus $\bar{h}_k\in \mathcal{D}$. Clearly, there is a discretization $\mathcal{D}_d$ with $\lvert \mathcal{D}_d\rvert \leq (2DST)^S$ such that any $h \in \mathcal{D}$ can find some $h_d\in \mathcal{D}_d$ with $\vert h(s)-h_d(s)\rvert \leq \frac{1}{ST}\ \forall s$. Now let $Y_t^{1(j)}\coloneqq\left( \sum_{s^\prime}p(s^\prime\vert s_t, \pi_k^1, \pi_t^2)h^{(j)}(s^\prime)-h^{(j)}(s_{t+1})\right)$ for every $h^{(j)}\in \mathcal{D}_d, j=1,..., (2DST)^S$. Now $Y_t^{1(j)}$'s are martingale difference sequences with respect to $\mathcal{F}_{t-1}$, so we can apply Azuma-Hoeffding's inequality and bound
\begin{align}
\sum_k\sum_{t=t_k}^{t_{k+1}-1}Y_t^{1(j)} \leq \sqrt{\frac{\log((2DST)^S\delta^{-1})}{2}T(2D)^2} \label{eqn:bernstein_2DS}
\end{align}
with probability at least $1-\frac{\delta}{(2DST)^S}$. Using the union bound, \eqref{eqn:bernstein_2DS} holds for all $j$ with probability at least $1-\delta$. Also, there exists a $j$ such that $\sum_{t=t_k}^{t_{k+1}-1}\left(Y_t^1-Y_t^{1(j)}\right)\leq T_k\times\frac{2S}{ST}=\frac{2T_k}{T}$. Thus we have 
\begin{align}
\sum_{k}\sum_{t=t_k}^{t_{k+1}-1}Y_t^1 \leq \tilde{\mathcal{O}}(D\sqrt{ST})\label{martingale_term}
\end{align}
with high probability. We also have $\sum_k\sum_{t=t_k}^{t_{k+1}-1}Y_t^2\leq \tilde{\mathcal{O}}(\sqrt{T})$ by Azuma-Hoeffding's inequality. Also, $\bar{h}_k(s_{t_{k+1}})-\bar{h}_k(s_{t_{k}})\leq 2D$. Collecting terms, we get the desired bound. 
\end{proof}

%\subsection{Proof of Lemma \ref{lemma:typical}}
\begin{proof}[Proof of Lemma \ref{lemma:typical}]
First fix $k$. Denote the transition probabilities of the optimistically selected model $M_k^1$ by $p_k^1(\cdot\vert \cdot,\cdot,\cdot)$.  In this proof, we define $\tilde{h}(\cdot)\coloneqq h(M_k^1, \bar{\pi}_k, \cdot)$, $h(\cdot)\coloneqq h(M, \bar{\pi}_k, \cdot)$, $\tilde{\mu}(\cdot)\coloneqq \mu(M_k^1, \bar{\pi}_k, \cdot)$, $\mu(\cdot)\coloneqq \mu(M, \bar{\pi}_k, \cdot)$, $\tilde{\rho}\coloneqq\rho(M_k^1, \bar{\pi}_k)$, $\rho\coloneqq \rho(M,\bar{\pi}_k)$, $r(\cdot)\coloneqq r(s, \bar{\pi}_k)$, $\tilde{p}(s^\prime\vert s)\coloneqq{p}_k^1(s^\prime \vert s, \bar{\pi}_k)$, $p(s^\prime\vert s)\coloneqq p(s^\prime\vert s, \bar{\pi}_k)$. 

By Bellman equation and the properties of irreducible Markov chains, we have
\begin{align*}
\rho=r(s)+\sum_{s^\prime}p(s^\prime\vert s)h(s^\prime) - h(s) \nonumber\\
\tilde{\rho}=r(s)+\sum_{s^\prime}\tilde{p}(s^\prime\vert s)\tilde{h}(s^\prime) - \tilde{h}(s)
\end{align*}
for all $s$. Therefore, we can write (for any $s$)
\begin{align}
\tilde{\rho}-\rho&=\sum_{s^\prime}\left(\tilde{p}(s^\prime\vert s)\tilde{h}(s^\prime)-p(s^\prime\vert s)h(s^\prime)\right)-\tilde{h}(s)+h(s)\nonumber \\
&=\sum_{s^\prime}\left(\tilde{p}(s^\prime\vert s)-p(s^\prime\vert s)\right)\tilde{h}(s^\prime)  + \sum_{s^\prime} \left( p(s^\prime\vert s)-\delta_{s,s^\prime} \right)\left(\tilde{h}(s^\prime)-h(s^\prime)\right). \label{eqn:rho difference}
\end{align}
Thus, 
\begin{align}
T_k(\tilde{\rho}-\rho)&=\sum_{s}T_k\mu(s)(\tilde{\rho}-\rho) \nonumber \\
&=\sum_{s} T_k \mu(s)\sum_{s^\prime}\left(\tilde{p}(s^\prime\vert s)-p(s^\prime\vert s)\right)\tilde{h}(s^\prime) \nonumber \\
&\leq\sum_s T_k \mu(s) \norm{\tilde{p}(\cdot\vert s)-p(\cdot\vert s)}_1\spa(\tilde{h}) \label{eqn:T_k_rho}, 
\end{align}
where the second equality is by using \eqref{eqn:rho difference} and the property of stationary distribution: $\sum_s \mu(s) \left(p(s^\prime\vert s)-\delta_{s,s^\prime}\right)=0$. 
By the definition of $\tilde{p}$ and $p$, we have
\begin{align}
\norm{\tilde{p}(\cdot\vert s)-p(\cdot\vert s)}_1 & \leq\sum_{a}\pi_t(a) \frac{\sum_{t=t_k}^{t_{k+1}-1}\mathbbm{1}_{s_t=s}\norm{\tilde{p}(\cdot\vert s,a)-p(\cdot\vert s,a)}_1}{v_k(s)} \nonumber \\
&\scalebox{0.9}{$\displaystyle= \frac{\sum_{t=t_k}^{t_{k+1}-1}\mathbbm{1}_{s_t=s}\norm{\tilde{p}(\cdot\vert s,a_t)-p(\cdot\vert s,a_t)}_1+\sum_{t=t_k}^{t_{k+1}-1}Y_t}{v_k(s)}$} \label{eqn:to_be_continue}
\end{align}
where $Y_t\coloneqq \mathbbm{E}\left[q_t\right]-q_t$, and $q_t\coloneqq \mathbbm{1}_{s_t=s} \norm{\tilde{p}(\cdot\vert s,a_t)-p(\cdot\vert s,a_t)}_1$. To apply Lemma \ref{lemma:bernstein_bubeck}, we note that $\lvert q_t \rvert\leq 2$ and $V_T\coloneqq \sum_{t=t_k}^{t_{k+1}-1} \mathbbm{E}[Y_t^2 \vert \mathcal{F}_{t-1}] \leq 2\sum_{t=t_k}^{t_{k+1}-1}\mathbbm{E}[q_t]$. Then we can bound
\begin{align}
&\sum_{t=t_k}^{t_{k+1}-1} \left(\mathbbm{E}[q_t]-q_t\right) \leq 2\sqrt{\left(2\sum_{t=t_k}^{t_{k+1}-1} \mathbbm{E}[q_t]\right)\log(T^2\delta^{-1})} +2\sqrt{5}\log(T^2\delta^{-1})    \label{eqn:to_be_solve}
\end{align}
with probability at least $1-\delta$. \eqref{eqn:to_be_solve} implies
\begin{align}
\sum_{t=t_k}^{t_{k+1}-1} (\mathbbm{E}[q_t]-q_t) \leq \sum_{t=t_k}^{t_{k+1}-1}q_t + 17\log(T^2\delta^{-1}). \label{eqn:e_q_q_bound}
\end{align}
Continuing \eqref{eqn:T_k_rho} with the help of \eqref{eqn:to_be_continue} and \eqref{eqn:e_q_q_bound}, we get
\begin{align*}
T_k(\tilde{\rho}-\rho) &\leq \scalebox{0.95}{$\displaystyle 2D\sum_s T_k\mu(s)\frac{2\sum_{t=t_k}^{t_{k+1}-1}q_t + 17\log(T^2\delta^{-1})}{v_k(s)}$} \nonumber \\
&\leq 3D\sum_s \left(2\sum_{t=t_k}^{t_{k+1}-1}q_t + 17\log(T^2\delta^{-1})\right) \nonumber \\
&\leq 6D\sum_{t=t_k}^{t_{k+1}-1} \norm{\tilde{p}(s_t,a_t)-p(s_t,a_t)}_1 + \tilde{\mathcal{O}}(DS)\nonumber\\
&\leq12D\sqrt{2S\ln{\frac{1}{\delta_1}}}\sum_{s,a} \frac{v_k(s,a)}{\sqrt{n_k(s,a)}} + \tilde{\mathcal{O}}(DS), 
\end{align*}
where we have used the assumptions in this lemma. Now sum over benign phases, we get 
\begin{align}
\sum_{k: \text{benign}}T_k\left(\rho(M_k^1, \bar{\pi}_k)-\rho(M, \bar{\pi}_k)\right) &\leq \sum_k \sum_{s,a} \frac{v_k(s,a)}{\sqrt{n_k(s,a)}} \tilde{\mathcal{O}}(D\sqrt{S}) + \sum_k \tilde{\mathcal{O}}(DS) \label{sample_bound_term}\\
&\leq 2.5\sqrt{SAT}\tilde{\mathcal{O}}(D\sqrt{S})+\tilde{\mathcal{O}}(DS^2A) \nonumber \\
&=\tilde{\mathcal{O}}(DS\sqrt{AT}+DS^2A). \nonumber 
\end{align}
with high probability. The last inequality is by the following Lemma together with Cauchy's inequality. 
\end{proof}
\begin{lemma}[cf. Lemma 19 of \cite{jaksch2010near}]
\label{lemma:any_sequences} For any sequence $\{z_i\}, i=1,...,N$ with $0\leq z_i\leq Z_{i-1}\coloneqq \max\{1,\sum_{\ell=1}^{i-1} z_\ell\}$. Let $K$ be a subset of $\{1,...,N\}$. Then we have
\begin{align*}
\sum_{i\in K}\frac{z_i}{\sqrt{Z_{i-1}}}\leq (\sqrt{2}+1)\sqrt{L},
\end{align*}
where $L\coloneqq \sum_{i\in K}z_i$.
\end{lemma}

\begin{proof}
\begin{align*}
\sum_{i\in K}\frac{z_i}{\sqrt{Z_{i-1}}}&\leq (\sqrt{2}+1)\sum_{i\in K}\frac{z_i}{\sqrt{Z_i}+\sqrt{Z_{i-1}}}\\
&=(\sqrt{2}+1)\sum_{i\in K}\frac{z_i\cdot (\sqrt{Z_i}-\sqrt{Z_{i-1}})}{(\sqrt{Z_i}+\sqrt{Z_{i-1}})\cdot(\sqrt{Z_i}-\sqrt{Z_{i-1}})}\\
&=(\sqrt{2}+1)\sum_{i\in K}(\sqrt{Z_i}-\sqrt{Z_{i-1}})\\
&\leq (\sqrt{2}+1)\sum_{i\in K}(\sqrt{L_i}-\sqrt{L_{i-1}})\leq (\sqrt{2}+1)\sqrt{L},
\end{align*}
where $L_i\coloneqq \sum_{\ell\in K: \ell\leq i}z_i$. We used the inequality 
\begin{align*}
\scalebox{1.0}{$ \displaystyle \sqrt{Z_i}-\sqrt{Z_{i-1}}\leq \sqrt{L_i}-\sqrt{L_{i-1}}\Leftrightarrow \frac{z_i}{ \sqrt{L_i}+\sqrt{L_{i-1}}}\leq  \frac{z_i}{ \sqrt{Z_i}+\sqrt{Z_{i-1}}}.$}
\end{align*}
\end{proof}
\if 0
\begin{lemma}
For any sequences of numbers $z_1$, $z_2$, ..., $z_n$ with $0\leq z_k\leq Z_{k-1}\coloneqq\max\left\{1,\sum_{i=1}^{k-1}z_i\right\},$
\begin{align}
\sum_{k=1}^{n}\frac{z_k}{\sqrt{Z_{k-1}}}\leq \left(\sqrt{2}+1\right)\sqrt{Z_n}< 2.5\sqrt{Z_n}. 
\end{align}
\end{lemma}
\fi
\section{Proofs of Lemma \ref{theorem:FH_fifth_term} and \ref{lemma:FH_second_term}}
\begin{proof}[Proof of lemma \ref{theorem:FH_fifth_term}]
Note that for any phase $k$ and any episode $i$ that fully lies in phase $k$, we have $\mathbbm{E}\left[\sum_{t=\tau_i}^{\tau_{i+1}-1}r(s_t, a_t)\right]=V_H(M, \pi_k^1, \pi_i^2, s_{\tau_i})$. Therefore, the terms in $\sum_k\Delta_k^{(5)}$ form a martingale difference sequence with no more than $T/H$ terms. Furthermore, $0\leq \sum_{t=\tau_i}^{\tau_{i+1}-1}r(s_t, a_t) \leq H$. By Lemma \ref{lemma:Azuma-Hoeffding's inequality}, with probability $1-\delta$, we have $\sum_k \Delta_k^{(5)}\leq \sqrt{\frac{\log(\delta^{-1})}{2} \frac{T}{H}H^2}=\tilde{\mathcal{O}}(\sqrt{HT})$.
\end{proof}
\begin{proof}[Proof of Lemma \ref{lemma:FH_second_term}]
Suppose that the value iteration halts at iteration $N$, then under Assumption \ref{assumption:ergodic} and by the proof of Lemma \ref{lemma:EVI}, we have 
\begin{align}
(1-\alpha) (\rho^*(M^+)-\gamma)\textbf{e}&\leq D\textbf{e} \leq  v_{N+1}-v_{N}=(1-\alpha)\left(\val\{r+Pv_N\}-v_N\right). \label{eq:near_Bellman}
\end{align}
Since $(M_k^1, p_k^1)$ is selected based on the $v_N$ when the value iteration halts, \eqref{eq:near_Bellman} is equivalent to 
\begin{align}
\rho^*(M^+) - \gamma \leq \min_{\pi^2} \Big\{r(s, \pi_k^1, \pi^2)+\sum_{s^\prime}p_k^1(s^\prime\vert s, \pi_k^1, \pi^2)v_N(s^\prime)\Big\}-v_N(s). \label{eq:near_Bellman2}
\end{align}

Besides, the span of the vector $v_N$ is bounded by $D$ by Lemma \ref{lemma:bounded_span}. Now we fix Player 1's policy as $\pi^1_k$ in the extended game, and let Player 2 run an $H$-step SG. The least amount Player 2 has to pay Player 1 in this SG is $\min_{\pi^2}V_H(M_k^1, \pi_k^1, \pi^2, s)$ (assuming that the game starts from $s$), which can be calculated by dynamic programming. The dynamic programming goes as follows: for $i=0,...,H-1$, for all $s$,
\begin{align*}
&u_0(s)=0,   \nonumber \\
&u_{i+1}(s)=\min_{a}\{r(s,\pi_k^1,a)+\sum_{s'}p_k^1(s'| s,\pi_k^1,  a)u_i(s')\},
\end{align*}
which, in its vector form, can be written as $u_{i+1}=\min_{a}\{r_{a}+P_{a}u_i\}$ by denoting $r_a(\cdot)\coloneqq r(\cdot,\pi_k^1,a)$ and $(P_a)_{ij}\coloneqq p_k^1(j\vert i, \pi_k^1, a)$. We can re-write the induction procedure as 
\begin{align*}
u_{i+1}-v_N=\min_a\{r_a+P_av_N+P_a(u_i-v_N)\}-v_N
\end{align*}
without affecting the solution. By the property $\min\{u+v\}\geq\min\{u\}+\min\{v\}$, we have
\begin{align}
u_{i+1}-v_N \geq \min_a\{r_a+P_av_N\}+\min_a\{P_a(u_i-v_N)\}-v_N  \label{eq:to_induction}
\end{align}
By \eqref{eq:near_Bellman2}, $\min_a\{r_a+P_av_N\}-v_N\geq \rho^*(M^+)-\gamma$, and since $P_a$ is stochastic, $P_a(u_i-v_N)\geq \min_{s^\prime}\{u_i(s^\prime)-v_N(s^\prime) \}$. Combining them with \eqref{eq:to_induction}, we have $u_{i+1}(s)-v_N(s)\geq \rho^*(M^+)-\gamma + \min_{s^\prime}\{u_i(s^\prime)-v_N(s^\prime)\}$ for all $s$. Then by induction, we can easily prove $u_{i}(s)-v_N(s) \geq i\left(\rho^*(M^+)-\gamma\right)+\min_{s'}\{u_0(s')-v_N(s')\}$, and therefore, $u_i(s)\geq i\left(\rho^*(M^+)-\gamma\right)+v_N(s)-\max_{s'}v_N(s')\geq i\left(\rho^*(M^+)-\gamma\right)-D$. 

Let $i=H$ and note that $\rho^*(M^+)=\max_{\tilde{M}}\max_{\pi^1}\min_{\pi^2} \rho(\tilde{M}, \pi^1, \pi^2)\geq \min_{\pi^2}\rho(M_k^1, \pi_k^1, \pi^2, s)$. The above result translates to $\min_{\pi^2}V_H(M_k^1, \pi_k^1, \pi^2, s) \geq H\min_{\pi^2}\rho(M_k^1, \pi_k^1, \pi^2,s)-D-H\gamma$, which bounds $\Delta_k^{(2)}$  by $\sum_{i\in \text{ph}(k)}(D+H\gamma)$. 
\end{proof}

%========================================================
% Fixed Horizon's proofs
%========================================================
\section{Proof of Theorem \ref{theorem:PAC} and Lemma \ref{theorem:FH_fourth_term} }
\begin{theorem} \label{theorem:PAC}(Sample Complexity Bound of \textsc{UCSG}. cf. Theorem 1 \cite{dann2015sample})
Given $\delta>0$, with probability at least $1-\delta$, for any $0<\varepsilon<1$, \textsc{UCSG} produces a sequence of policies $\pi_k^1$, that yield at most $\tilde{\mathcal{O}}\big(\frac{H^2S^2A}{\varepsilon^2}\big)$ episodes $i$ such that $|V_H(M,\pi^1_k,\pi^2_i,s_{\tau_i})-V_H(M_k^1,\pi_k^1,\pi^2_i,s_{\tau_i})|>\varepsilon$.
\end{theorem}
Theorem \ref{theorem:PAC} mainly follows from the following Lemma \ref{lemma:bad_event} and \ref{lemma:good_event}. In \cite{dann2015sample} the analysis of sample complexity is facilitated by partitioning the state-action space. The state-action pairs are grouped into different categories according to two indices. The first index, \textit{importance}, measures in log-scale the relative occurrence frequency of $(s,a)$ with respect to a fixed constant under the policy. The second index, \textit{knownness}, measures also in log-scale the ratio of the total number of observations to the occurrence frequency.
Here we modify the the definition of weight, importance, and knownness for a state-joint action $(s,a)=(s,a^1,a^2)$ defined below to have a partition of the state-joint-action space  $\mathcal{S}\times\mathcal{A}=\mathcal{S}\times\mathcal{A}^1\times\mathcal{A}^2$ for each episode.
\begin{definition}\label{def:weight.FH}
Define the weight of a state-joint-action pair $(s,a)$ under joint policy $\pi_i$ in episode $i$ as the expected occurrence frequency of $(s,a)$ in episode $i$,
\begin{equation}
w_{i}(s,a)\coloneqq \sum_{t=\tau_i}^{\tau_{i+1}-1}\mathbb{P}(s_t=s,a_t=a | a_t\sim \pi_{i}, s_{\tau_i}).\nonumber
\end{equation}
\end{definition}
The setting in \cite{dann2015sample} is somewhat different from two-player zero-sum SGs. In the episodic RL setting after an episode is over, a new episode starts afresh with the same initial distribution $p_0$, while in the non-episodic setting, initial state $s_{\tau_i}$ in each episode is sampled from a different distribution. Initial state distributions do not matter that much in our setting except we need the initial state $s_{\tau_i}$ to compute the expected frequency $w_i(s,a)$.
\begin{definition}\label{def:importance.FH}
Define the importance of a state-joint-action pair $(s,a)$ in episode $i$ as
$$\iota_{i}(s,a)\coloneqq \max\bigg\{z_j: z_j\leq \frac{w_i(s,a)}{w_\text{min}}\bigg\},$$
where $z_1=0$ and $z_j=2^{j-2}\  \forall j=2,3,...$
\end{definition}
\begin{definition}\label{def:knownness.FH}
Define the knownness of a  a state-joint-action pair $(s,a)$ in episode $i$ as
$$\scalebox{1}{$ \displaystyle \kappa_{i}(s,a)\coloneqq \max\bigg\{z_j: z_j\leq \frac{n_{k(i)}(s,a)}{mw_{i}(s,a)}\bigg\},$}$$ 
\end{definition}
where $z_1=0$ and $z_j=2^{j-2}\  \forall j=2,3,...$
\begin{definition}\label{def.category_X.FH} We can now categorize state-joint-action pairs $(s,a)$ into subsets
$$\scalebox{1}{$\displaystyle X_{i,\kappa,\iota}\coloneqq \{(s,a)\in X_i: \kappa_{i}(s,a)=\kappa, \iota_{i}(s,a)=\iota\}$,}$$
$$\scalebox{0.9}{$\displaystyle \text{ and } \bar{X}_i=\mathcal{S}\times\mathcal{A} \backslash X_i, \text{ where } X_i=\{(s,a)\in\mathcal{S}\times\mathcal{A}: \iota_i(s,a)>0 \}.$}   $$
\end{definition}
In contrast to the original definitions \cite{dann2015sample} which are designated for each phase $k$ in the episodic RL setting, in our setting, weight $w_i(s,a)$, importance $\iota_i(s,a)$, knownness $\kappa_i(s,a)$ are now indexed for each episode $i$ because Player 2 may have arbitrary policies in different episodes.

%############################################################
%Proofs of section under assumption 2
%############################################################
Theorem \ref{theorem:PAC} mainly follows from the following Lemma \ref{lemma:bad_event} and \ref{lemma:good_event}.
Select $m=\frac{512SH^2(\log\log H)^2 \log_2\left(8T^2SH\right)\ln(6/\delta_1)}{\varepsilon^2}$, $\delta_1\coloneqq \frac{\delta}{2U_{max}S}$, $U_{max}\coloneqq SA\log_2T$ and $w_{\min}\coloneqq\frac{\varepsilon}{4HSA}$ for any $0< \varepsilon < H$, and any $0<\delta<1$ and then we have the following two lemmas. 
\begin{lemma} \label{lemma:bad_event} (cf. Lemma 2 in \cite{dann2015sample}) Let $E$ be the number of episodes $i$ for which there are $\kappa$ and $\iota$ with $|X_{i,\kappa,\iota}|>\kappa$, i.e. $E=\sum_{i=1}^{\infty} \mathbbm{1}\{ \exists (\kappa,\iota):|X_{i,\kappa,\iota}|>\kappa\}$ and assume that $m\geq \frac{6H^2}{\varepsilon}\ln(2E_{\max}/\delta)$, where $E_{\max}=\log_2 \Big(\frac{H}{w_{\min}}\Big)\log_2 (SA)$. Then $\mathbb{P}(E\leq 6SAE_{\max}m)\geq 1-\delta/2$.
\end{lemma}
\begin{proof}
The proof mainly follows as Lemma 2 \cite{dann2015sample}. Here we point out the differences between the original \textsc{UCFH} algorithm \cite{dann2015sample} and our \textsc{UCSG}, when we remove the input $\varepsilon$. 
\begin{enumerate}
\item Their stopping rule for phase $k$ is dependent on the specification of $\varepsilon$.
\item They set an upper bound for the maximum number of executions for each state-action pair $(s,a)$, which is determined beforehand and hardcoded in their algorithm.
\item Our algorithm only needs input $\delta$ to specify the failure probability  and has $(\varepsilon,\delta)$-PAC bounds for arbitrarily selected $\varepsilon$.
\end{enumerate}
The original \textsc{UCFH} nearly doesn't need the parameter $\varepsilon$ except at one place: their phases stops when ``$\exists (s,a)$, $v_k(s,a)\geq \max\{mw_{\min}, n_k(s,a)\}$ and $n_k(s,a)<SmH$.'' Since $w_{\min}$ and $m$ are defined through $\varepsilon$, this stopping rule requires $\varepsilon$ to be known by the algorithm. They need this because they would like to control $U_{\max}$, the total number of phases run by the algorithm. In their case, having this stopping rule, $U_{\max}\leq SA\log_2\frac{SmH}{mw_{\min}}=SA\log_2\frac{SH}{w_{\min}}$ because phase change won't be triggered when $n_k(s,a)<mw_{\min}$ or $n_k(s,a)>SmH$. However, since we assume that the time horizon $T$ is known, we can simply use $U_{\max}\leq SA\log_2T$, and this can simplify our stopping rule to only ``$\exists (s,a), v_k(s,a)\geq n_k(s,a)$.''  

Therefore, we can totally abandon the use of $\varepsilon$ in our algorithm, but enjoy their analysis results. The results automatically hold for arbitrarily selected $\varepsilon$. However, since we bound the number of $\kappa$ by $\log_2(4HSAT/\varepsilon)$ in Lemma \ref{lemma:good_event}, we cannot let $\varepsilon$ tends to $0$ too fast. (The minimum $\varepsilon$ we will select is $\varepsilon_0=\min\{H,\sqrt{(H^3S^2A)/T}\}$ as in the proof of Lemma \ref{theorem:FH_fourth_term}, where we select $H=\max\{D,\sqrt[3]{D^2T/(S^2A)}\}$ for Theorem \ref{regret_finite_horizon_variant} ). 
\end{proof}
%\begin{lemma} (cf. Lemma 2 \cite{dann2015sample})
%Let $E$ be the number of episodes $i$ for which there are $\kappa$ and $\iota>0$ with $\lvert X_{i,\kappa, \iota}\rvert > \kappa$, i.e., $E=\sum_{i=1}^{\infty} \mathbbm{1}\{ \exists (\kappa,\iota):|X_{i,\kappa,\iota}|>\kappa\}$. Then $E\leq 6mSAE_{max}$ with high probability. 
%\end{lemma}
%\begin{proof}
%The proof mainly follows from Lemma 2 \cite{dann2015sample}. 
%\end{proof}
%\begin{lemma}\label{lemma:good_event}(cf. Lemma 3 in \cite{dann2015sample}) Assume $M\in\mathcal{M}_k$. If $|X_{i,\kappa,\iota}|\leq \kappa$ for all $\iota>0$ and $\kappa$. Then with high probability, $|V_H(M^1_k,\pi^1_k,\pi^2_i, s_{\tau_i})-V_H(M,\pi^1_k,\pi^2_i, s_{\tau_i})|\leq \varepsilon.$
%\end{lemma}
\begin{lemma} \label{lemma:good_event}(cf. Lemma 3 in \cite{dann2015sample}) Assume $M\in\mathcal{M}_k$. If $|X_{i,\kappa,\iota}|\leq \kappa$ for all $(\kappa,\iota)$ and for all $0<\varepsilon\leq 1$ and $m\geq 512\frac{CH^2}{\varepsilon^2}(\log_2\log_2 H)^2\log_2\big(\frac{4HSAT}{\varepsilon}\big)\log_2(SA)\ln(6/\delta_1)$. Then $|V_H(M^1_k,\pi^1_k,\pi^2_i)-V_H(M,\pi^1_k,\pi^2_i)|\leq \varepsilon.$
\end{lemma}
\begin{proof}
It mainly follows the same proof as Lemma 3 in \cite{dann2015sample}. It was shown sufficient to let $m\geq 512C(\log_2\log_2 H)^2|\mathcal{K}\times\mathcal{I}|\frac{H^2}{\varepsilon^2}\ln(6/\delta_1)$. The only differences are in the upper bounds for $|\mathcal{K}\times\mathcal{I}|$.  In  $\textsc{UCFH}$, the maximum number of executions of each state-action pair is set equal to $mSH$. Thus their knownness $\kappa(s,a)$ is no more than $\frac{n(s,a)}{mw_{min}}\leq \frac{4S^2AH^2}{\varepsilon}$, whereas in our setting, since $n(s,a)\leq T$, $\kappa(s,a)\leq \frac{n(s,a)}{mw_{min}}\leq \frac{4HSAT}{\varepsilon}$. Thus in our setting $|\mathcal{K}\times\mathcal{I}| \leq \log_2(\frac{4HSAT}{\varepsilon})\log_2(SA)$.
\end{proof}
%#########################################################################################################################
% Proof that PAC bounds imply Regret bounds
%#########################################################################################################################
\begin{proof}[Proof of Lemma \ref{theorem:FH_fourth_term}] Let $\varepsilon_0\coloneqq$$ \min\big\{H,\sqrt{(H^3 S^2 A)/T}\big\}$, and $\delta_0\coloneqq \delta/\lceil\log_2(H/\varepsilon_0)\rceil$. We invoke logarithmically many times the bound in Theorem \ref{theorem:PAC} and use the union bound to obtain the regret. By assumption, for $j=1,...,\lceil\log_2(H/\varepsilon_0)\rceil$, with probability no less than $1-\delta_0$, there are at most $\tilde{\mathcal{O}}(4^jS^2A)$ episodes that are not $(2^{-j}H)$-optimal.  Then the total error is bounded by
\begin{align*}
&\sum_{k}\sum_{i\in\text{ph}(k)}\Big\lvert V_H(M^1_k,\pi^1_k,\pi_i^2,s_{\tau_i})-V_H(M,\pi^1_k,\pi_i^2,s_{\tau_i})\Big\rvert\\
&\coloneqq \sum_k\sum_{i\in \text{ph}(k)} r_i=\sum_{i:r_i\leq\varepsilon_0} r_i+\sum_{i:r_i>\varepsilon_0} r_i\\
&\leq\varepsilon_0 \frac{T}{H}+\sum_{j=1}^{\lceil\log_2(H/\varepsilon_0)\rceil} \tilde{\mathcal{O}}(4^jS^2 A) (2^{-j+1}H)\\
&=\varepsilon_0\frac{T}{H}+4(2^{\lceil\log_2(H/\varepsilon_0)\rceil}-1)\tilde{\mathcal{O}}(HS^2A)\\
&\leq\varepsilon_0\frac{T}{H}+8\frac{\tilde{\mathcal{O}}(H^2S^2A)}{\varepsilon_0}=\tilde{\mathcal{O}}(S\sqrt{HAT}+HS^2A). \ \ \ \ \ \qedhere
\end{align*}
\end{proof}

\section{Proofs for Offline Training Complexity }
\label{appendix:Proofs for Offline Training Complexity}
\begin{proof}[Proof of Theorem \ref{theorem:offline_irreducible}]
Define 
\begin{align*}
K_\varepsilon&\coloneqq \{k: \rho^*(M)-\min_{\pi^2}\rho(M,\pi^1_k,\pi^2)>\varepsilon\}, \nonumber \\
K_\varepsilon^\prime &\coloneqq K_\varepsilon \cap\{k: \text{phase\ } k \text{\ is benign}\}. 
\end{align*}
Also, define
\begin{align}
&\text{Reg}_\varepsilon^{\text{(off)}\prime}\coloneqq \sum_{k\in K_\varepsilon^\prime} T_k(\rho^*(M)-\min_{\pi^2}\rho(M,\pi_k^1, \pi^2)) \nonumber \\
&= \text{Reg}_\varepsilon^{\text{(on)}\prime} + \sum_{k\in K_\varepsilon^\prime} \sum_{t=t_k}^{t_{k+1}-1} \left(r_t - \min_{\pi^2}\rho(M ,\pi_k^1, \pi^2)\right). \nonumber \\
\end{align}
where $\text{Reg}_\varepsilon^{\text{(on)}\prime}$ is defined as a summation similar to $\text{Reg}_T^{\text{(on)}}$ except that it is summed only over time steps in phases $k\in K_\varepsilon^\prime$. Besides, analogous to the definition of $L_\varepsilon$, we define $L_\varepsilon^\prime\coloneqq \sum_{k: \text{benign}}T_k\mathbbm{1}\{\rho^*(M)-\min_{\pi^2}\rho(M,\pi_k^1, \pi^2)>\varepsilon\}$. 

We will argue (a) the order of $\text{Reg}_\varepsilon^{\text{(off)}\prime}$ does not exceed that of $\text{Reg}_\varepsilon^{\text{(on)}\prime}$, and (b) the upper bound of $\text{Reg}_\varepsilon^{\text{(on)}\prime}$ is similar to that of $\text{Reg}_T^{\text{(on)}}$ except that the dependency on $T$ is replaced by $L_\varepsilon^\prime$. %Then we can use the already established bound for  $\text{Reg}_\varepsilon^{\text{(on)}}$ to upper bound $L_\varepsilon$. The first summation is similar to \eqref{eqn:decompose_interval_regret}'s last two terms, or the main term in one-player MDP's regret analysis (e.g., in \cite{jaksch2010near}). For completeness, we provide its bound in the appendix. Formally, we have the following two Theorems. 

To show (a), we note that the extra terms in $\text{Reg}_\varepsilon^{\text{(off)}\prime}$ compared to $\text{Reg}_\varepsilon^{\text{(on)}\prime}$ are the sum of 
\begin{align*}
&\sum_{t=t_k}^{t_{k+1}-1} \left(r_t-\min_{\pi^2}\rho(M, \pi_k^1, \pi^2)\right)=\sum_{n=5}^{7} \Lambda_k^{(n)}, \nonumber \\
&\Lambda_k^{(5)}\coloneqq \sum_{t=t_k}^{t_{k+1}-1} \left(r_t- \rho(M, \pi_k^1, \pi_k^2)\right), \\
&\Lambda_k^{(6)}\coloneqq  \sum_{t=t_k}^{t_{k+1}-1} \left(\rho(M, \pi_k^1, \pi_k^2)- \rho(M_k^2, \pi_k^1, \pi_k^2, s_{t_k})\right), \\
&\Lambda_k^{(7)}\coloneqq \sum_{t=t_k}^{t_{k+1}-1} \left(\rho(M_k^2, \pi_k^1, \pi_k^2, s_{t_k})-\min_{\pi^2}\rho(M, \pi_k^1, \pi^2)\right),
\end{align*}
over $k\in K_\varepsilon^\prime$. $\Lambda_k^{(7)}$ is bounded by $T_k \gamma_k$ by \eqref{eqn:pessimistic_selection}; the bound of this term is the same as that of $\Lambda_k^{(1)}$. $\Lambda_k^{(5)}$ and $\Lambda_k^{(6)}$ are symmetric to $\Lambda_k^{(4)}$ and $\Lambda_k^{(3)}$ respectively (note that the $\bar{\pi}_k^2$ we constructed in Section \ref{Regret decomposition and the introduction of pi2} will be identical to $\pi_k^2$ in the offline setting). Therefore, we can use the same bounds for the corresponding terms. 

Now we proceed to argue (b) and bound $\text{Reg}_\varepsilon^{\text{(on)}\prime}$. We will largely reuse the regret analysis we already done for $\text{Reg}_T^{\text{(on)}}$, but only sum up the contribution from phases in $K_\varepsilon^\prime$. 

The contribution to $\text{Reg}_\varepsilon^{\text{(on)}\prime}$ from $\Lambda_k^{(1)}$  is 
\begin{align}
\sum_{k\in K_\varepsilon^\prime} T_k \gamma_k=\sum_{k\in K_\varepsilon^\prime} T_k/\sqrt{t_k};  \label{eqn:offline_contribution_1}
\end{align}
the contribution from $\Lambda_k^{(3)}$ is as shown in \eqref{sample_bound_term}: 
\begin{align}
\sum_{k\in K_\varepsilon^\prime} \sum_{s,a} \frac{v_k(s,a)}{\sqrt{n_k(s,a)}}\tilde{\mathcal{O}}(D\sqrt{S})+\sum_{k\in K_\varepsilon^\prime}\tilde{\mathcal{O}}(DS);  \label{eqn:offline_contribution_2}
\end{align} 
finally, the contribution from $\Lambda_k^{(4)}$ is as shown in \eqref{eqn:contribution_from_third_term}: 
\begin{align}
\sum_{k\in K_\varepsilon^\prime} \left(\bar{h}_k(s_{t_{k+1}-1})-\bar{h}(s_{t_k}) +\sum_{t=t_k}^{t_{k+1}-1}(Y_t^1+Y_t^2)  \right). \label{eqn:offline_contribution_3}
\end{align}
$\text{Reg}_\varepsilon^{\text{(on)}\prime}$ is then bounded by the sum of \eqref{eqn:offline_contribution_1}-\eqref{eqn:offline_contribution_3}. By lemma \ref{lemma:any_sequences}, \eqref{eqn:offline_contribution_1} is bounded by $(\sqrt{2}+1)\sqrt{L_\varepsilon^\prime}$, and the first term in \eqref{eqn:offline_contribution_2} is bounded by $\tilde{\mathcal{O}}(\sqrt{SAL_\varepsilon^\prime})\tilde{\mathcal{O}}(D\sqrt{S})=\tilde{\mathcal{O}}(DS\sqrt{AL_\varepsilon^\prime})$ by Cauchy inequality. The second term in \eqref{eqn:offline_contribution_2} can be still bounded by $\tilde{\mathcal{O}}(DS^2A)$. Since the martingale difference sequences in \eqref{eqn:offline_contribution_3} are now summing over a total of $L_\varepsilon^\prime$ steps, \eqref{eqn:offline_contribution_3} is now bounded by $DSA+D\sqrt{SL_\varepsilon^\prime}$ (cf. \eqref{martingale_term}). 

As a whole, we conclude that $\text{Reg}_\varepsilon^{\text{(on)}\prime} \leq \tilde{\mathcal{O}}(DS\sqrt{AL_\varepsilon^\prime}+DS^2A)$, and hence $\text{Reg}_\varepsilon^{\text{(off)}\prime} \leq \tilde{\mathcal{O}}(DS\sqrt{AL_\varepsilon^\prime}+DS^2A)$ by the argument in (a). 

Note that by the definition of $K_\varepsilon^\prime$, we have
\begin{align}
\text{Reg}_\varepsilon^{\text{(off)}\prime}&=\sum_{k\in K_\varepsilon^\prime}T_k(\rho^*(M)-\min_{\pi^2}\rho(M,\pi_k^1, \pi^2))\nonumber \\
&\geq \sum_{k\in K_\varepsilon^\prime} T_k \varepsilon = \varepsilon L_\varepsilon^\prime. \label{eqn:Reg_epsilon_lower_bound}
\end{align}
Combining \eqref{eqn:Reg_epsilon_lower_bound} with the upper bound of $\text{Reg}_\varepsilon^{\text{(off)}\prime}$ just established, we have
\begin{align*}
\varepsilon L_\varepsilon^\prime \leq \tilde{\mathcal{O}}(DS\sqrt{AL_\varepsilon^\prime}+DS^2A), 
\end{align*}
which has the solution
\begin{align*}
L_\varepsilon^\prime\leq \tilde{\mathcal{O}}\left(\frac{D^2S^2A}{\varepsilon^2}\right). 
\end{align*}
Comparing the definitions of $L_\varepsilon$ and $L_\varepsilon^\prime$, and by Lemma \ref{lemma:bound_benign}, we get 
\begin{align*}
L_\varepsilon\leq L_\varepsilon^\prime + \tilde{\mathcal{O}}(D^3S^5A)=\tilde{\mathcal{O}}\left(D^3S^5A+\frac{D^2S^2A}{\varepsilon^2}\right).
\end{align*}

Finally, we remark on how to select a single stationary policy after we have run the algorithm for $T$ steps. Note that in our proofs, we actually bound the single step regret in phase $k$ through 
\begin{align}
&\rho^*(M)-\min_{\pi^2}\rho(M, \pi_k^1, \pi^2) \leq \min_{\pi^2}\rho(M_k^1, \pi_k^1, s_{t_k})-\rho(M_k^2, \pi_k^1, \pi_k^2, s_{t_k})+2\gamma_k 
\end{align}
because LHS is $\frac{1}{T_k}\sum_{n=1}^{7}\Lambda_k^{(n)}$ while RHS is $\frac{1}{T_k}\sum_{n=2}^{6}\Lambda_k^{(n)}+2\gamma_k$. Note that the terms on RHS can all be obtained by the algorithm, so they form an available upper bound for the LHS. Let $u_k$ denotes the RHS. Then the previous proofs actually proved that 
\begin{align*}
\sum_k T_k\mathbbm{1}\{u_k>\varepsilon\} \leq \tilde{\mathcal{O}}\left(D^3S^5A + \frac{D^2S^2A}{\varepsilon^2}\right)
\end{align*}
holds with high probability. Therefore, if $T>\tilde{\Omega}\left(D^3S^5A+\frac{D^2S^2A}{\varepsilon^2}\right)$, there will be some $k$ such that $u_k<\varepsilon$. Since the algorithm knows $u_k$, it can just select the minimum of all $u_k$'s among all phases. That will output a policy $\pi_k^1$ such that $\rho^*(M)-\min_{\pi^2}\rho(M,\pi_k^1, \pi^2) \leq \varepsilon$.

\end{proof}

\if 0
\begin{proof}[Proof of Theorem \ref{theorem:offline_ergodic}]
The regret in each phase $k$ is decomposed into six terms as in \eqref{eqn:fh_decompose}.
Of the six terms, the fourth term dominates over the first, the third (nonpositive), the fifth and the sixth. So we only look at the second and the fourth term. The second term is bounded by $\frac{T_k}{H}D+T_k\gamma_k$. Summing over $k\in K_\varepsilon$ by Lemma \ref{lemma:any_sequences} gives $\frac{L_\varepsilon}{H}D+\tilde{\mathcal{O}}(\sqrt{L_\varepsilon})$. Thus its average error is bounded by $\tilde{\mathcal{O}}(D/H+1/\sqrt{L_\varepsilon})$. By taking $H=D/(2\varepsilon)$, the sample complexity for $\Delta_k^2$ is $\tilde{\mathcal{O}}(1/\varepsilon^2)$. On the other hand, by Theorem \ref{theorem:PAC}, the fourth term has sample complexity bound $\tilde{\mathcal{O}}(HS^2A/\varepsilon^2)$. By substituting $H=D/(2\varepsilon)$ gives the dominating sample complexity bound $\tilde{\mathcal{O}}(DS^2A/\varepsilon^3)$. 
\end{proof}
\fi

\begin{proof}[Proof of Theorem \ref{theorem:offline_ergodic}]
\begin{align}
\text{Reg}_\varepsilon^{\text{(off)}}&\coloneqq \sum_{k\in K_\varepsilon} T_k(\rho^*(M)-\min_{\pi^2}\rho(M,\pi_k^1, \pi^2)) \nonumber \\
&= \text{Reg}_\varepsilon^{\text{(on)}} + \sum_{k\in K_\varepsilon} \sum_{t=t_k}^{t_{k+1}-1} \left(r_t - \min_{\pi^2}\rho(M ,\pi_k^1, \pi^2)\right), \nonumber
\end{align}
where $\text{Reg}_\varepsilon^{\text{(on)}}$ is the sum of $\Delta_k$ over $k\in K_\varepsilon$.
Of the six regret terms \eqref{eqn:fh_decompose}, $\Delta_k^{(4)}$ dominates over $\Delta_k^{(1)}$, $\Delta_k^{(3)}$, $\Delta_k^{(5)}$, and $\Delta_k^{(6)}$. So we only look at the $\Delta_k^{(2)}$ and $\Delta_k^{(4)}$. $\Delta_k^{(2)}$ is bounded by $\frac{T_k}{H}D+T_k\gamma_k$. Summing over $k\in K_\varepsilon$ by Lemma \ref{lemma:any_sequences} gives $\frac{L_\varepsilon}{H}D+\tilde{\mathcal{O}}(\sqrt{L_\varepsilon})$. Thus its average error is bounded by $\tilde{\mathcal{O}}(D/H+1/\sqrt{L_\varepsilon})$. By taking $H=D/(2\varepsilon)$ we have the sample complexity for the second term is $\tilde{\mathcal{O}}(1/\varepsilon^2)$. On the other hand, by Theorem \ref{theorem:PAC}, $\Delta_k^{(4)}$ has sample complexity bound $\tilde{\mathcal{O}}(HS^2A/\varepsilon^2)$. By substituting $H=D/(2\varepsilon)$ gives the dominating sample complexity bound $\tilde{\mathcal{O}}(DS^2A/\varepsilon^3)$. 
We argue again the order of $\text{Reg}_\varepsilon^{\text{(off)}}$ does not exceed that of $\text{Reg}_\varepsilon^{\text{(on)}}$. 
To show this, we note that the extra terms in $\text{Reg}_\varepsilon^{\text{(off)}}$ compared to $\text{Reg}_\varepsilon^{\text{(on)}}$ are the sum of
\begin{align*}
&\sum_{t=t_k}^{t_{k+1}-1} \left(r_t-\min_{\pi^2}\rho(M, \pi_k^1, \pi^2)\right)=\sum_{n=7}^{11}\Delta_k^{(n)},  \nonumber \\
&\scalebox{0.85}{$ \displaystyle \Delta_k^{(7)}\coloneqq \sum_{i \in \text{ph}(k)}\left(\sum_{t=\tau_i}^{\tau_{i+1}-1}r(s_t,a_t)- V_H(M, \pi_k, s_{\tau_i})\right),$} \\
&\scalebox{0.85}{$ \displaystyle \Delta_k^{(8)}\coloneqq\sum_{i\in \text{ph}(k)} \left(V_H(M, \pi_k,s_{\tau_i})- V_H(M_k^2, \pi_k, s_{\tau_i})\right),$} \\
&\scalebox{0.85}{$ \displaystyle \Delta_k^{(9)}\coloneqq\sum_{i\in \text{ph}(k)} \left(V_H(M_k^2, \pi_k,s_{\tau_i})- H\rho(M_k^2, \pi_k, s_{\tau_i})\right),$} \\
&\scalebox{0.85}{$ \displaystyle \Delta_k^{(10)}\coloneqq \sum_{i\in \text{ph}(k)} \left(H\rho(M_k^2, \pi_k,s_{\tau_i})- H\min_{\pi^2}\rho(M, \pi_k^1,\pi^2, s_{\tau_i})\right),$}\\
&\scalebox{0.85}{$\displaystyle \Delta_k^{(11)}\coloneqq 2H,$}
\end{align*}
over $k\in K_\varepsilon$. This decomposition mirrors that in \eqref{eqn:fh_decompose} where  $\Delta_k^{(7)}$, $ \Delta_k^{(8)}$, $ \Delta_k^{(9)}$, $ \Delta_k^{(10)}$ and $ \Delta_k^{(11)}$ are symmetric to the $\Delta_k^{(5)}$, $\Delta_k^{(4)}$, $\Delta_k^{(2)}$, $\Delta_k^{(1)}$, and $\Delta_k^{(6)}$ in \eqref{eqn:fh_decompose}, respectively, and we can use the same bounds for the corresponding terms. 

Finally, we can pick an $\varepsilon$-optimal policy $\pi_k^1$ after the algorithm has run for $T>\tilde{\mathcal{O}}\left(\frac{DS^2A}{\varepsilon^3}\right)$ steps. The way is similar to that described in the proof of Theorem \ref{theorem:offline_irreducible}.
\end{proof}

%========================================================================================
% Other Technical Lemmas
%========================================================================================
\section{Other Technical Lemmas}\label{section:other_technical_lemmas}
\begin{remark}\label{remark:mean_first_passage_time}
Under Assumption \ref{assumption:irreducible}, note that for any stationary policy $\pi$, we have $\spa(h(M, \pi, \cdot))\leq T^{\pi}(M)$. Indeed, 
\begin{align*}
h(M,\pi,s)&=\mathbb{E}^{\pi}_s\Big[\sum_{t=1}^\infty r_t-\rho(M,\pi)\Big]\\
&\leq T^{\pi}_{s\rightarrow s'}(M)+\mathbb{E}^{\pi}_{s'}\Big[\sum_{t=1}^\infty r_t-\rho(M,\pi)\Big]\\
&=T^{\pi}_{s\rightarrow s'}(M)+h(M,\pi,s').
\end{align*}
\end{remark}

\begin{remark}\label{remark:average_time_<D}
Imagine an MDP where all transitions from $s\neq s^\prime$ remain the same while $s^\prime$ becomes an absorbing state; rewards on $s \neq s^\prime$ are all $1$ and $0$ on $s^\prime$. Now $\max_{\pi^1}\max_{\pi^2}T^{\pi^1, \pi^2}_{s\rightarrow s^\prime}(M)$ is equivalent to the maximum reward on this MDP, which can be achieved by stationary joint policy by both players. 
\end{remark}

%=========================================================
% Regularization approach
% =========================================================

\section{Regularization/Constraint-based Approach for Assumption \ref{assumption:irreducible}}
It is possible to improve the $\tilde{\mathcal{O}}(D^3S^5A)$ term in the regret bound under Assumption \ref{assumption:irreducible}. Note that this term mainly comes from Lemma \ref{lemma:bound_h_by_2D}, which says that to wait until $\spa(h(M_k^1, \pi_k^1, \bar{\pi}_k^2, \cdot)) < 2D$, we need to pay $\tilde{\mathcal{O}}(D^3S^5A)$ regret. However, if we can know the value of $D$ in advance, the optimistic model $M_k^1$ can be selected based on the following constrained optimization problem: 
\begin{align*}
M_k^1&=\argmax_{\tilde{M}\in \mathcal{M}_k} \max_{\pi^1}\min_{\pi^2} \rho(\tilde{M}, \pi^1,\pi^2, s_{t_k}),  \nonumber \\
\text{subject to \ \   }  &\forall \pi^1, \pi^2 \in \Pi^{\text{SR}},  \spa(h(\tilde{M},\pi^1, \pi^2, \cdot)) \leq D.  \nonumber
\end{align*}
Clearly, the true model $M$ still lies in this feasible set, so this is a valid way to select $M_k^1$. It is also possible to convert this into a regularized optimization problem as demonstrated by \cite{bartlett2009regal}. Nevertheless, we are not aware of any practical algorithm that can solve either optimization problem. We just demonstrated in this paper that the benefit of this regularization/constraint-based approach is only on the additive constant but not on the asymptotic performance. 

% In the unusual situation where you want a paper to appear in the
% references without citing it in the main text, use \nocite
%\nocite{langley00}

\end{document}